\documentclass[10pt]{article} 

\usepackage{microtype}
\usepackage{graphicx}
\usepackage{subcaption}
\usepackage{booktabs} 
\usepackage{physics}
\usepackage[accepted]{tmlr}

\usepackage{amsmath,amsfonts,bm}

\def\eqref#1{equation~\ref{#1}}

\def\1{\bm{1}}

\DeclareMathAlphabet{\mathsfit}{\encodingdefault}{\sfdefault}{m}{sl}
\SetMathAlphabet{\mathsfit}{bold}{\encodingdefault}{\sfdefault}{bx}{n}

\newcommand{\E}{\mathbb{E}}

\newcommand{\R}{\mathbb{R}}

\usepackage{hyperref}
\usepackage{url}

\usepackage{algorithm}
\usepackage{algpseudocode}
\usepackage{amsmath}
\usepackage{amssymb}
\usepackage{mathtools}
\usepackage{amsthm}

\newcommand{\inner}[2]{\left\langle #1,\, #2 \right\rangle}
\DeclareMathOperator*{\argmax}{arg\,max}
\DeclareMathOperator*{\argmin}{arg\,min}

\newcommand{\prox}[0]{\text{prox}}

\newcommand{\Norm}[1]{\left\|#1\right\|_2}

\usepackage[capitalize,noabbrev]{cleveref}

\theoremstyle{plain}
\newtheorem{theorem}{Theorem}[section]

\newtheorem{lem}[theorem]{Lemma}

\theoremstyle{definition}
\newtheorem{definition}[theorem]{Definition}
\newtheorem{asm}[theorem]{Assumption}
\theoremstyle{remark}

\usepackage{makecell}

\def \S {\mathbf{S}}

\def \X {\mathcal{X}}
\def \Y {\mathcal{Y}}

\def \R {\mathbb{R}}

\def \w {\mathbf{w}}
\def \v {\mathbf{v}}

\def \x {\mathbf{x}}

\def \x {\mathbf{x}}

\def \1 {\mathbf{1}}

\def \s {\mathbf{s}}

\def \y {\mathbf{y}}
\def \u {\mathbf{u}}

\def \F {\mathcal{F}}

\def \B {\mathcalB}
\def \C {\mathbf C}

\def \U {\mathcal{U}}

\def \y {\mathbf{y}}
\def \E {\mathbf{E}}
\def \bT {\mathbb{T}}
\def \x {\mathbf{x}}
\def \w {\mathbf{w}}

\def \D {\mathcal{D}}

\def \u {\mathbf{u}}

\def \w {\mathbf{w}}

\def \R {\mathbb{R}}
\def \S {\mathcal{S}}

\def \W {\mathcal{W}}

\def \v {\mathbf{v}}

\def \B {\mathcal{B}}

\def \s {\mathbf{s}}

\def \C {\mathcal{C}}

\def \X {\mathcal{X}}

\def \F {\mathcal{F}}

\def \by {\bar{\y}}

\def \bs {\bar{\s}}
\def \ty {\tilde{\y}}
\def \hy {\hat{\y}}
\def \C {\mathcal{C}}
\def \D {\mathcal{D}}
\def \dist {\text{dist}}
\def \prox {\text{prox}}

\usepackage{multirow}

\usepackage[customcolors]{hf-tikz}
\usepackage{MnSymbol,wasysym}

\newcommand{\prt}[1]{\left(#1\right)}
\newcommand{\brk}[1]{\left[#1\right]}
\newcommand{\crk}[1]{\left\{#1\right\}}

\title{Stochastic Primal-Dual Double Block-Coordinate for Two-way Partial AUC Maximization}

\author{\name Linli Zhou \email lynn94874@tamu.edu \\
      \addr Department of Computer Science and Engineering\\
      Texas A\&M University  
      \AND
      \name Bokun Wang \email bokun.wang@utexas.edu \\
      \addr Department of Electrical and Computer Engineering\\
      University of Texas, Austin
      \AND
      \name My T. Thai \email mythai@cise.ufl.edu\\
      \addr Department of Computer and Information Science and Engineering\\
      University of Florida
      \AND
      \name Tianbao Yang \email tianbao-yang@tamu.edu \\
      \addr Department of Computer Science and Engineering\\
      Texas A\&M University \\
    }

\def\month{09}  
\def\year{2025} 
\def\openreview{\url{https://openreview.net/forum?id=M3kibBFP4q}} 

\begin{document}
\maketitle

\begin{abstract}
Two-way partial AUC (TPAUC) is a critical performance metric for binary classification with imbalanced data, as it focuses on specific ranges of the true positive rate (TPR) and false positive rate (FPR). However, stochastic algorithms for TPAUC optimization remain under-explored, with existing methods either limited to approximated TPAUC loss functions or burdened by sub-optimal complexities. To overcome these limitations, we introduce two innovative stochastic primal-dual double block-coordinate algorithms for TPAUC maximization. These algorithms utilize stochastic block-coordinate updates for both the primal and dual variables, catering to both convex and non-convex settings. We provide theoretical convergence rate analyses, demonstrating significant improvements over prior approaches. 
Our experimental results, based on multiple benchmark datasets, validate the superior performance of our algorithms, showcasing faster convergence and better generalization. This work advances the state of the art in TPAUC optimization and offers practical tools for real-world machine learning applications.
\end{abstract}

\section{Introduction}
    
    The area under the ROC curve, commonly referred to as AUC, is frequently utilized as a measure of the model’s classification ability, without the explicit setting of a threshold. With a long history dating back to the late 90s \citep{Herbrich:1999gl}, AUC is acknowledged as a more informative metric than accuracy for assessing the performance of binary classifiers in the context of imbalanced data and widely used in machine learning.
    
    In many applications, there are large monetary costs due to high false positive rates (FPR) and low true positive rates (TPR), e.g., in medical diagnosis. Hence, a measure of primary interest is the region of the ROC curve corresponding to low FPR and high TPR,  i.e., TPR $\geq 1-\theta_0$, FPR $\leq \theta_1$, for some $\theta_0,\theta_1 \in (0, 1)$, which is referred to as two-way partial AUC (TPAUC).   Nevertheless, research on efficient optimization algorithms to optimize TPAUC for learning a classifier remains underdeveloped. 
    
    {Compared with standard AUC maximization, optimizing TPAUC presents several unique technical challenges. First, the estimator of TPAUC requires selecting subsets of positives and negatives in the top and bottom ranks. Some earlier works have proposed heuristic approaches for TPAUC maximization, including selecting examples based on their ranks in the mini-batch or converting data selection into ad-hoc data weighting~\citep{yang2021all,kar2014online}, which do not provide a guarantee of optimizing TPAUC losses. 
    } 

    \begin{table*}[t]
    \caption{Comparison with prior works for optimizing the TPAUC loss, where $n_+$ is the number of positive examples, 
    $S$ is the mini-batch size of positive examples, $B$ is the mini-batch size of negative examples, and $d$ is the dimension of the model parameter  
    }
    \centering
    \begin{tabular}{ccclc}
    \hline
    Method                                  & Convexity  & Loop   & \multicolumn{1}{c}{Iteration Complexity} & Total Complexity                                     \\ \hline
    SONX \citep{hu2023non}  & non-convex & Single & $\order{(B+S)d}$               & $\order{\frac{n_+}{B^{1/2}S\epsilon^6}}$ \\
    SOTA \citep{zhu2022auc} & non-convex & Double & $\order{(B+S)d+n_+}$             & $\order{\frac{n_+}{\epsilon^6}}$               \\ \hline
    STACO1 (Ours)                           & convex     & Single & $\order{(B+S)d}$               & $\order{\frac{n_+}{S\epsilon^2}}$              \\
    STACO2 (Ours)                           & non-convex & Double & $\order{(B+S)d}$               & $\order{\frac{n_+}{BS\epsilon^6}}$             \\ \hline
    \end{tabular}
    \label{tab:comp}
    \vspace*{-0.1in}
    \end{table*}
    Recently, \citet{zhu2022auc} have initiated rigorous optimization of TPAUC losses. They converted data selection in top/bottom ranks into pairwise loss selection and reformulated it using the tool of distributionally robust optimization. They have proposed two algorithms for two different formulations: SOTAs for solving a smooth coupled compositional objective that corresponds to a soft TPAUC loss and SOTA for solving a non-smooth min-max objective that corresponds to an exact TPAUC loss. Nevertheless, SOTAs is not for optimizing the exact TPAUC loss, and SOTA is inefficient for large datasets as it requires updating all coordinates of an auxiliary variable corresponding to all positive examples at every iteration. Additionally, its convergence rate analysis fails to demonstrate any mini-batch speedup.
    
\citet{hu2023non} has developed an algorithm for solving non-convex non-smooth coupled compositional objective of the exact TPAUC loss as formulated in \citep{zhu2022auc}. However, their method cannot achieve linear speedup in terms of the mini-batch size of negative examples. 
{In addition, since their method does not exploit convexity, its convergence guarantee still exhibits a complexity of  $O(1/\epsilon^6)$ even in the convex setting.}

{To overcome these difficulties, this paper proposes improved algorithms and analysis over SOTA for solving the non-smooth min-max objective of the exact TPAUC loss. Our key idea is to design stochastic {double block-coordinate updates} that simultaneously act on both primal and dual variables. We propose two methods: STACO1 for convex objectives and STACO2 for non-convex objectives. Our convergence analysis introduces novel techniques for handling non-bilinear min-max objectives with stochastic block-coordinate updates, establishing state-of-the-art complexity bounds. Our algorithms enable scalable updates and provable mini-batch parallel speedup.  We compare our results with prior works in Table \ref{tab:comp}.}


   We summarize the main contributions of our work below:    
    \begin{itemize}
        \item We propose novel primal-dual double block-coordinate algorithms STACO (\textbf{S}tochastic \textbf{T}wo-way partial \textbf{A}UC block-\textbf{C}ordinate \textbf{O}ptimizer) designed for convex functions (STACO1) and non-convex functions (STACO2). These algorithms leverage double block-coordinate updates for both the primal and dual variables.

        \item We provide a novel convergence analysis of STACO1 for convex functions. To the best of our knowledge, this is the first work to analyze double block-coordinate updates for both primal and dual variables for min-max optimization without a bilinear structure. We extend this analysis to STACO2 for non-convex cases, demonstrating its ability to find (nearly) stationary solutions. We demonstrate our algorithm enjoys better convergence rate than existing results \cite{hu2023non, zhu2022auc} by improving the block-size dependency, achieving full mini-batch speedup and time efficiency. 

        \item We conduct comprehensive experiments on both linear and deep models for image classification and graph classification tasks involving imbalanced data. Our algorithms consistently demonstrate better performance compared to existing TPAUC maximization methods and various baselines. Additionally, we perform ablation studies to verify the improved convergence rates of our methods.

    \end{itemize}

\section{Related Work}

    \textbf{Two-way Partial AUC (TPAUC).} AUC has been studied for more than two decades~\citep{hanley1982meaning}, and a huge amount of work has been devoted to AUC maximization~\citep{yang2022auc}. Compared to AUC maximization, two-way partial AUC (TPAUC) maximization is much more challenging due to that it involves the selection of examples whose prediction scores are in a certain range. Recently, studies on TPAUC have emerged, as researchers have argued that for certain tasks, only the TPR or FPR within a specific range is of interest \citep{narasimhan2013svmpauctight,yang2019two,yuan2021compositional, zhu2022auc,xie2024weakly}. In particular, by replacing TPR and FPR with surrogate losses, TPAUC maximization problem can be further transformed into coupled compositional optimization and min-max optimization~\citep{zhu2022auc}. 
    Some other works are also focusing on TPAUC \citep{zhang2023federated, shao2023weighted, yang2023auc, yang2022optimizing, shao2022asymptotically}. \citet{zhang2023federated} focuses on optimizing a compositional formulation for AUC maximization, \citet{shao2023weighted} considers a weighted AUC formulation for cost-sensitive learning, and \citet{yang2023auc} considers AUC maximization with certified robustness. \citet{yang2022optimizing, shao2022asymptotically} focus on TPAUC maximization with the following differences: \citet{yang2022optimizing} tackles the data selection challenge by a weighting scheme, which does not yield the exact TPAUC surrogate objective; \citet{shao2022asymptotically} considers TPAUC maximization with a special square loss. In contrast, we directly tackle solving the exact TPAUC surrogate objective without further approximation and our result applies to any non-decreasing loss function.

    \textbf{Compositional Optimization.} Compositional optimization has gained substantial attention in recent years. This area of optimization deals with objective functions that are composed of multiple nested functions, leading to challenges in efficient evaluation and optimization. Several papers \citep{wang2017stochastic, wang2017accelerating, zhang2020optimal, zhang2022stochastic} have considered standard compositional optimization, where the inner function does not depend on the random variable of the outer level. However, simply applying these algorithms to TPAUC maximization would suffer a high cost~\citep{qi2021stochastic}.  
    To address this issue, \citet{zhu2022auc} have formulated TPAUC maximization as FCCO (Finite-Sum Coupled Compositional Optimization) as introduced in \citep{qi2021stochastic}. \citet{hu2023non} have proposed an algorithm termed SONX for solving a non-smooth FCCO optimization where the outer function is non-smooth and applied it to TPAUC maximization.  

    \textbf{Min-Max Optimization.} 
    Many stochastic primal-dual algorithms have been proposed to solve non-convex min-max optimization since the seminal work~\citep{rafique2022weakly}. Built on their proximal-guided algorithmic framework, \citet{zhu2022auc}  developed SOTA for solving the min-max formulation of TPAUC loss.  However, their algorithm suffers from the limitations mentioned before. 
    To address its limitations, we have to consider double block-coordinate updates for both primal and dual variables and develop advanced techniques to derive a complexity that has a parallel speed-up, {which means complexity is linearly dependent on both positive and negative mini-batch size }. Several works \citep{zhang2015stochastic,alacaoglu2022complexity} have considered stochastic primal-dual block-coordinate algorithms for solving finite-sum min-max problems with a bilinear structure, where the block-coordinate update is only applied to the dual variable.  \citet{hamedani2023randomized,jalilzadeh2019doubly} have considered more general min-max problems using block-coordinate updates for the primal variable only or for both primal and dual variables. However, their algorithm and analysis require the coupled function to be smooth in terms of both the primal and dual variables, which is not applicable to TPAUC maximization.  
    {In addition, Li et al. (2025) propose a Smoothed Proximal Linear Descent-Ascent (Smoothed PLDA) algorithm for deterministic nonsmooth nonconvex-nonconcave minimax problems with convergence guarantees under the KL property. However, PLDA is not directly applicable to large-scale stochastic problems with composite structure, where full dual updates and deterministic computations are infeasible.} Recently, \citet{wangnear} proposed a novel stochastic primal-dual block-coordinate algorithm to solve convex finite-sum compositional optimization problems, which only employs the block-coordinate update on the dual variable. 

\subsection{Notations and Definitions}

    We present notations in this section. For any $\w \in \mathcal W$, the subdifferential $\partial_{\w} f(\w)$ is the set of subgradients of $f$ at point $\w$. For a vector $\y \in \R^n$, $\y^{(i)} \in \R$ represents the $i$-th coordinate (block) of $\y$, i.e., $\y = (\y^{(1)},\cdots, \y^{(n)})^{\bT}$. We use $f_i^*$ to denote the convex conjugate of $f_i$. For a function $g(\x) = \E_{\xi\sim\mathbb{P}}\brk{g(\x; \xi)}$, we define the stochastic estimator based on the mini-batch $\B$ as $g(\x;\B) \coloneqq \frac{1}{|\B|} \sum_{\xi\in\B}g(\x;\xi)$. 

\section{Primal-dual Double Block-Coordinate Algorithms for TPAUC Maximization}

   Let $\x$ denote an input example and $h_\w(\x)$ denote a prediction of a parameterized model such as a deep neural network or a linear model on data $\x$. Denote by $\S_+$ the set of $n_+$ positive examples and by $\S_-$ the set of $n_-$ negative examples. TPAUC measures the area under the ROC curve where the TPR is higher than $1-\theta_0$ and the FPR is lower than an upper bound $\theta_1$. A surrogate loss for optimizing TPAUC with TPR $\geq 1-\theta_0$, FPR$\leq \theta_1$ is given by: 
    \begin{align}\label{eqn:tpauc-o}
        \min_{\w\in\R^d}\frac{1}{n_+n_-}\sum_{\x_i \in \S_+^{\uparrow}[1,k_1]}\sum_{\x_j \in \S_-^{\downarrow}[1,k_2]} \ell(h_{\w}(\x_j)-h_{\w}(\x_i)),
    \end{align}
    where $\ell(\cdot)$ is a convex, monotonically non-decreasing surrogate loss of the indicator function $\mathbb{I}(h_{\w}(\x_j)\geq h_{\w}(\x_i))$, $\S_+^{\uparrow}[1,k_1]$ is the set of positive examples with $k_1 = \left \lfloor n_+ \theta_0 \right \rfloor$ smallest scores, and $\S_-^{\downarrow}[1,k_2]$ is the set of negative examples with $k_2 = \left \lfloor n_- \theta_1 \right \rfloor$ largest scores. To tackle the challenge of selecting examples for $\S_+^{\uparrow}[1,k_1]$ and $\S_-^{\downarrow}[1,k_2]$, we use the following lemma to reformulate (\ref{eqn:tpauc-o})~\citep{zhu2022auc}.
    
    \begin{lem} If $\ell(\cdot)$ is  non-decreasing, then the TPAUC loss minimization problem~(\ref{eqn:tpauc-o}) is equivalent to the following: 
    \begin{align}
    \label{eq:tpauc_ori}
        \min_{\w,s',\s} \frac{1}{n_+} \sum_{\x_i\in\S_+} f_i(g_i(\w, \s^{(i)}), s'), 
    \end{align} 
    where $\s = (\s^{(1)}, \cdots , \s^{(n_+)})^\top$, $f_i(g,s')=s'+\frac{1}{\theta_0}[g-s']_+$, and $g_i(\w,\s^{(i)})= \frac{1}{n_-}\sum_{\x_j \in \S_-} \s^{(i)}+\frac{[\ell(h_{\w}(\x_j)-h_{\w}(\x_i))-\s^{(i)}]_+}{\theta_1}$. 
    \end{lem} 
    \vspace*{-0.1in}
    {The reformulation above uses an equivalent form of the conditional-value-at-risk (CVaR) loss,  $\frac{1}{n\gamma}\sum_{i=1}^{n\gamma} \ell_{[i]}(\cdot) = \min_{s} s + \frac{1}{n\gamma}\sum^n_{i=1}[\ell_i(\cdot)-s]_{+}$, where $\gamma = k/n$ for some integer $k \in [n]$, $\ell_{[i]}(\cdot)$ denotes the $i$-th largest value in $\{\ell_1, \cdots, \ell_n\}$. ~\citep[Lemma 1]{Ogryczak:2003dl}.}
    {Since $[t]_+=\max_{y\in[0,1]}ty$}, we cast~(\ref{eq:tpauc_ori}) into an equivalent min-max problem:
    \begin{align}
    \label{eq:tpauc_minmax}
    \min_{\w\in\R^d,s'\in\R\atop\s\in\R^{n_+}}\max_{\y\in\brk{0,1}^{n_+}} \frac{1}{n_+} \sum_{\x_i\in\S_+} \y^{(i)} \cdot \frac{g_i(\w,\s^{(i)})-s'}{\theta_0} + s'.
    \end{align} 
 This problem presents unique challenges that make existing algorithms unsuitable for direct application: (i) the objective function is non-smooth with respect to $\w$ and $\s$ due to the hinge function in $g_i$; (ii) both the primal variable $\s$ and the dual variable $\y$ are high-dimensional and depend on all positive examples, preventing their full coordinate updates in each iteration; and (iii) the coupled term is not bilinear with respect to the primal and dual variables.

\subsection{Algorithms}
    Now we present our efficient algorithms designed to solve the min-max problem~(\ref{eq:tpauc_minmax}) in convex and non-convex settings. 
    
    \textbf{STACO1 for convex functions.} We first consider the convex case when $\ell(h_{\w}(\x_j)-h_{\w}(\x_i))$ is a convex function of $\w$. This is true when we learn a linear model such that $h_\w(\x) = \w^{\top}\x$. Hence, $g_i(\w, s)$ is convex w.r.t. $(\w, s)$ for any $i\in[n]$, and~(\ref{eq:tpauc_minmax}) is a convex-concave min-max problem. 

\begin{algorithm}[t]
        \caption{STACO1}
        \label{alg:single_tpauc}
        \begin{algorithmic}[1] 
            \State Initialize $\w_0\in\W$, $\y_0 = \mathbf{1}^{n_+}$, $\s_0=\mathbf{1}^{n_+}$, $s'_0=1$,
            \For{$t=0,1,\dotsc,T-1$}
            \State Sample a batch $\S_t\subset\S_+$ with $|\S_t| = S$ 
            \State Sample independent mini-batches $\B_t, \tilde{\B}_t\subset\S_-$
            \For{each $i\in\S_t$} 
            \State Update $\y_{t+1}^{(i)}$ according to~(\ref{eqn:y}) 
            \State Update $\s_{t+1}^{(i)}$ according to~(\ref{eqn:s}) 
            \EndFor
            \State For each $i\notin \S_t$, $\y_{t+1}^{(i)} = \y_t^{(i)}$ and $\s_{t+1}^{(i)} = \s_t^{(i)}$
            \State Update $\w_{t+1}$ according to~(\ref{eqn:w})  
            \State Update $s'_{t+1}$ according to~(\ref{eqn:sp}) 
            \EndFor
            \State $\bar{\w}=\frac{1}{T}\sum_{t=0}^{T-1}\w_{t+1}, \bar{\s}=\frac{1}{T}\sum_{t=0}^{T-1}\s_{t+1}, \bar{s}'=\frac{1}{T}\sum_{t=0}^{T-1}s'_{t+1}$
            \State Return $\bar{\w},\bar{\s},\bar{s}'$
        \end{algorithmic}
    \end{algorithm}

    A challenge of solving~(\ref{eq:tpauc_minmax}) is that updating all coordinates for $\s, \y$ would require computing $g_i(\w, \s^{(i)})$ and its gradient for all positive examples $\x_i\in\S_+$, which is prohibited when the number of positive examples is large. Hence, we have to use block-coordinate updates for both $\s$ and $\y$. Let us consider how to update $\y^{(i)}$ and $\s^{(i)}$ for a sampled coordinate $i$. A simple method is to use gradient ascent to update $\y^{(i)}$ and use gradient descent to update $\s^{(i)}$, which require computing $g_i(\w,  \s^{(i)})$ and $\partial_{\s^{(i)}}g_i(\w,  \s^{(i)})$. However, this would require processing all negative examples $\S_-$ as $g_i(\w, \s^{{(i)}})$ depends on all negative examples. To reduce this cost, we need to use stochastic estimators of their gradients. For a random mini-batch of negative samples $\B\subset\S_-$, we let 
   \begin{align*}
    g_i(\w,\s^{(i)};\B)= \frac{1}{|\B|}\sum_{\x_j \in \B} \s^{(i)}+\frac{[\ell(h_{\w}(\x_j)-h_{\w}(\x_i))-\s^{(i)}]_+}{\theta_1}. 
   \end{align*}
   At the $t$-th iteration, we sample a mini-batch of $S$ positive examples $\S_t\subset\S_+$ and a mini-batch of $B$ negative examples $\B_t\subset\S_-$. We update $\y_{t+1}^{(i)}$ according to
   \begin{align}
   \label{eqn:y}
       \y_{t+1}^{(i)} = \argmax_{\y^{(i)}\in[0,1]}&\left\{\y^{(i)} \cdot\frac{g_i(\w_t,\s_t^{(i)};\B_t)-s'_t}{\theta_0} - \frac{1}{2\alpha} \left(\y^{(i)} - \y_t^{(i)}\right)^2\right\}, \forall \x_i\in\S_t
   \end{align}
    where $\alpha$ is a step size parameter. Then we update $\s_{t+1}^{(i)}, i\in\S_t$ and $\w_{t+1}$ using stochastic gradient descent: 
    \begin{align}
       \s_{t+1}^{(i)} & = \s_t^{(i)} - \frac{\beta}{\theta_0} \y_{t+1}^{(i)}\partial_{\s^{(i)}} g_i (\w_t,\s_t^{(i)};\tilde{\B}_t), \forall \x_i\in\S_t\label{eqn:s}\\
       \w_{t+1} & = \w_t - \frac{\eta}{\theta_0}\frac{1}{S}\sum_{i\in\S_t} \y_{t+1}^{(i)} \partial_{\w} g_i (\w_t,\s_t^{(i)};\tilde{\B}_t)\label{eqn:w}\\
         s'_{t+1} & = s'_t - \beta'(1-\frac{1}{\theta_0 S}\sum_{i\in\S_t} \y_{t+1}^{(i)})\label{eqn:sp}
   \end{align}
   where $\beta, \eta, \beta'$ are step size parameters, and we use another mini-batch of negative samples $\tilde\B_t$ independent of $\B_t$ to decouple the dependence between  $\y_{t+1}^{(i)}$ and $\tilde\B_t$.   The detailed steps of \textbf{STACO1} are presented in Algorithm~\ref{alg:single_tpauc}.



     \begin{algorithm}[t]
        \caption{STACO2}
        \label{alg:double_tapuc}
        \begin{algorithmic}[1] 
            \State Initialize $\w_0\in\W$, $\s_0=\mathbf{1}^{n_+}$, $s'_0=1$
            \For{$t=0,1,\dotsc,T-1$}
                \State Initialize $\y_{t,0} = \mathbf{1}^{n_+}$
                \State Set $\w_{t,0} = \w_t, \s_{t,0} = \s_t, s'_{t,0}=s'_t$
                \For{$k=0,1,\dotsc,K_t-1$}
                    \State Sample a batch $\S_{t,k}\subset \S_+$, where $|\S_{t,k}| = S$ 
                    \State Sample independent mini-batches $\B_{t,k}$, $\tilde{\B}_{t,k} \subset \S_-$
                    \For{each $i\in\S_{t,k}$} 
                        \State Update $\y_{t,k+1}^{(i)}$ according to~(\ref{eqn:yy})
                        \State Update $\s_{t,k+1}^{(i)}$ according to~(\ref{eqn:ss})
                    \EndFor
                    \State For each $i\notin \S_{t,k}$, $\y_{t,k+1}^{(i)} = \y_{t,k}^{(i)}$ and $\s_{t,k+1}^{(i)} = \s_{t,k}^{(i)}$
                    \State Update $\w_{t+1}$ according to~(\ref{eqn:ww})
                    \State Update $s'_{t+1}$ according to~(\ref{eqn:ssp})
                \EndFor
                \State $(\bar{\w}_{t}, \bar{\s}_{t}, \bar{s}'_{t}) = \frac{1}{K_t}\sum_{k=0}^{K_t - 1} (\w_{t,k+1}, \s_{t,k+1}, s'_{t,k+1})$
                \State Set $\w_{t+1}=\bar{\w}_{t}, \s_{t+1}=\bar{\s}_{t}, s'_{t+1}=\bar{s}'_{t}$
            \EndFor
            \State Return $\w_{T},\s_{T},s'_{T}$
    \end{algorithmic}
    \end{algorithm}
    \textbf{STACO2 for non-convex functions.}
    Next we consider the non-convex case. We assume $\ell(h_{\w}(\x_j)-h_{\w}(\x_i))$ is weakly-convex with respect to $\w$, which holds true when $\ell$ is a convex non-smooth function and $h_\w(\x)$ is a smooth function of $\w$~\citep{hu2023non}. Hence, $g_i(\w, s)$ is weakly-convex with respect to $(\w, s)$, and~(\ref{eq:tpauc_minmax}) is a weakly-convex concave min-max problem.  Inspired by the proximal-guided algorithm \citep{rafique2022weakly} for non-smooth weakly-convex concave problems, we propose a double-loop algorithm \textbf{STACO2} for solving problem (\ref{eq:tpauc_minmax}). The inner loop updates apply STACO1 to solve the following  problem approximately at  the $t$-th outer iteration:
    \begin{align}
    \label{eq:sub}
    &\min_{\w\in\R^d,s'\in\R \atop\s\in\R^{n_+}}\max_{\y\in\brk{0,1}^{n_+}} 
    \frac{1}{n_+} \sum_{\x_i\in\S_+} \y^{(i)} \cdot \frac{g_i(\w,\s^{(i)})-s'}{\theta_0} + s' + \frac{1}{2\gamma}\Norm{\w-\w_{t,0}}^2 + \frac{1}{2n_+\gamma}\Norm{\s-\s_{t,0}}^2,
    \end{align} 
    where $\w_{t,0},\s_{t,0}$ are initial value of $\w,\s$ at $t$-th stage, $\gamma>0$ is a proper parameter. The addition of quadratic functions is to ensure the function becomes convex in terms of $\w, \s$.  At $k$-th iteration in $t$-th stage, we utilize following updates:
     \begin{align}
       \y_{t,k+1}^{(i)} &= \argmax_{\y^{(i)}\in[0,1]}\left\{\y^{(i)} \cdot\frac{g_i(\w_{t,k},\s_{t,k}^{(i)};\B_{t,k})-s'_{t,k}}{\theta_0} - \frac{1}{2\alpha_t} \left(\y^{(i)} - \y_{t,k}^{(i)}\right)^2\right\}, \forall \x_i\in\S_{t,k}\label{eqn:yy}\\
       \s_{t,k+1}^{(i)} & = \s_{t,k}^{(i)} - \frac{\beta_t}{\theta_0}\left(\y_{t,k+1}^{(i)}\partial_{\s^{(i)}} g_i (\w_{t,k},\s_{t,k}^{(i)};\tilde{\B}_{t,k}) + \frac{1}{\gamma}(\s^{(i)}_{t,k}-\s^{(i)}_{t,0})\right), \forall \x_i\in\S_{t,k}\label{eqn:ss}\\
       \w_{t,k+1} & = \w_{t,k} - \frac{\eta_t}{\theta_0}\left(\frac{1}{S}\sum_{i\in\S_{t,k}} \y_{t,k+1}^{(i)} \partial_{\w} g_i (\w_{t,k},\s_{t,k}^{(i)};\tilde{\B}_{t,k}) + \frac{1}{\gamma}(\w_{t,k}-\w_{t,0})\right)
       \label{eqn:ww}\\
         s'_{t,k+1} & = s'_{t,k} - \beta'_t(1-\frac{1}{\theta_0 S}\sum_{i\in\S_{t,k}} \y_{t,k+1}^{(i)})\label{eqn:ssp},
   \end{align}
where $\alpha_t, \beta_t, \eta_t, \beta'_t$ are step size parameters.

   

We would like to highlight the difference between STACO2 and SOTA~\citep{zhu2022auc}, where we use block-coordinate update for $\s\in\R^+$. In contrast, SOTA needs to update all coordinates of $\s$. This difference is caused by different techniques for handling all coordinates: they compute an unbiased sparse stochastic gradient for $\s$ by sampling and then update $\s$ using a stochastic proximal gradient method. The unbiased sparse stochastic gradient used in SOTA cannot enjoy a variance bound that scales with the mini-batch size.    In contrast, we just compute an unbiased stochastic gradient for the sampled coordinate of $\s$, and perform a stochastic gradient descent on sampled coordinates and leave other coordinates unchanged. It is this difference that makes our analysis more involved and leads to a parallel speed-up. 

\section{Analysis}

    In this section, we present the convergence results for our algorithms. We emphasize the contributions of our convergence analysis for both convex and non-convex settings compared to~\cite{zhu2022auc}: (i) our convergence analysis for the convex case is more refined, leading to an optimal convergence rate which implies a parallel speed-up in terms of mini-batch size; (ii) our analysis for the non-convex case is also improved, which not only enjoys a parallel speed-up but also removes strong boundedness assumptions of $\s_{t, k}, s'_{t,k}$ and the pairwise loss values at all iterations.

    For analysis, we consider the following optimization problem: 
    \begin{align}
    \label{eq:primal}
        \min_{\u\in \U, \s\in \S} F(\u,\s) \coloneqq \frac{1}{n}\sum_{i=1}^n f_i(g_i(\u,\s^{(i)})),
    \end{align}
    where $f_i:\R\rightarrow\R$ is closed proper convex and lower-semicontinuous, $g_i:(\U,\in\S_i)\rightarrow \R$ is possibly non-convex, and $\U,\S$ are convex closed sets, $g_i(\u,\s^{(i)})\coloneqq\E_{\zeta_i\sim \mathbb{P}_i}\left[g_i(\u,\s^{(i)};\zeta_i)\right]$. It is equivalent to the following min-max problem:
    \begin{align}
    \label{eq:pd}
        \min_{\u\in \U,\s\in\S}\max_{\y\in\Y} L(\u,\s,\y) \coloneqq \frac{1}{n}\sum_{i=1}^n \y^{(i)}g_i(\u,\s^{(i)}) - f_i^*(\y^{(i)}).
    \end{align}
    Compared to problem (\ref{eq:tpauc_ori}), (\ref{eq:primal}) excludes parameter $s'$. Since the update of $s'$ is almost the same as $\w$, our analysis for solving~(\ref{eq:primal}) can be easily extended to STACO1 and STACO2.

\subsection{Assumptions}

    We first outline assumptions underlying our analysis. Notably, these assumptions are easily satisfied for TPAUC maximization when the loss function $\ell$ is Lipchitz continuous. 
    
    \begin{asm}
    \label{asm:lip}
         For any $i\in[n]$, we suppose $f_i,g_i$ is Lipschitz continuous, i.e.,  there exists $C_f,C_g>0$ such that 
        \begin{align*}
            &|f_i(u) - f_i(\bar{u})|\leq C_f|u-\bar{u}| \nonumber\\
            &\left|g_i(\u,\s^{(i)})- g_i(\bar{\u},\bs^{(i)})\right| \leq C_g \left(\Norm{\u - \bar{\u}}+\abs{\s^{(i)} - \bar{\s}^{(i)}}\right),
        \end{align*}
        for any $u,\bar{u}\in\R$, $\u,\bar{\u}\in\U$ and $\s^{(i)},\bar{\s}^{(i)}\in\S_i$.
    \end{asm}
    
    
    \begin{asm}
    \label{asm:var}
        For any $i\in[n]$, there exists finite $\sigma_0^2,\sigma_1^2,\sigma_2^2$ such that 
        \begin{align*}
            & \E_{\zeta_i}\left|g_i(\u,\s^{(i)}) - g_i(\u,\s^{(i)};\zeta_i)\right|^2 \leq \sigma_0^2, \nonumber\\
            &\E_{\zeta_i}\Norm{\hat{G}^{(i)}_1(\zeta_i)  - G^{(i)}_1}^2 \leq \sigma_1^2, ~~\E_{\zeta_i}\Norm{\hat{G}^{(i)}_2(\zeta_i)  - G^{(i)}_2}^2 \leq \sigma_2^2,
        \end{align*}
        for stochastic subgradients $\hat{G}^{(i)}_1(\zeta_i) \in \partial_{\u} g_i(\u,\s^{(i)};\zeta_i)$, $\hat{G}^{(i)}_2(\zeta_i) \in \partial_{\s^{(i)}} g_i(\u,\s^{(i)};\zeta_i)$ at any $ \u\in\U$, and $\s^{(i)}\in\S_i$. Besides, there exists $\delta^2$ such that
        \begin{align*}
            \E_{j} \Norm{y^{(j)} G^{(j)}_1 - \frac{1}{n}\sum_{i=1}^n y^{(i)} G^{(i)}_1 }^2 \leq \delta^2,
        \end{align*}
        for any $G^{(i)}_1 \in \partial_1 g_i(\u,\s^{(i)})$, $\u\in\U,\s^{(i)}\in\S_i$, and $\y\in \Y$. Note that under Assumption~\ref{asm:lip}, we have $\delta^2\leq C_f^2 C_g^2$.
    \end{asm}

\subsection{Convex Case}

    We first analyze the Algorithm \ref{alg:single}, which aims to solve the problem (\ref{eq:pd}) when both $f_i$ and $g_i$ are convex for any $i\in[n]$. The analysis is motivated by techniques proposed in \citet{wangnear}. However, the  problem they considered is $\frac{1}{n}\sum_{i=1}^n f_i(g_i(\u))$, which excludes the primal parameter $\s$. Notably, the analysis of convergence of primal parameter $\u$ is more tricky than $\w$ since its updating only lies in selected coordinates each iteration. 

    \begin{theorem}
    \label{thm:iteration_complexity_cvx}
        Under Assumptions \ref{asm:lip} and \ref{asm:var}, when $g_i(\u, \s^{(i)})$ is convex w.r.t $\u, \s^{(i)}$, let $\eta = \order{\epsilon}$, $\beta = \order{\epsilon}$, and $\alpha = \order{B\epsilon}$, 
        STACO1 can make $\E\brk{F(\bar{\u},\bar{\s})-F(\u^*,\s^*)} \leq \epsilon$ after $T =  \order{\frac{nC_g^2 C_f^2}{S\epsilon^2} + \frac{C_f^2 \sigma_1^2}{B\epsilon^2} + \frac{nC_f^2 \sigma_2^2}{BS\epsilon^2} + \frac{\delta^2}{S\epsilon^2} + \frac{n\sigma_0^2}{BS\epsilon^2}}$ iterations, where $\bar{\u}=\frac{1}{T}\sum_{t=0}^{T-1}\u_{t+1},\bar{\s}=\frac{1}{T}\sum_{t=0}^{T-1}\s_{t+1}$.
    \end{theorem}


     \textbf{Remark.} The proof is included in Appendix \ref{app:iteration_complexity_cvx}. The above convergence rate implies a parallel speed-up in terms of the positive batch size $S$ and negative batch size $B$. When we use full information at each iteration, which means $\sigma_0=0, \sigma_1=0, \sigma_2=0, \delta=0, S=n_+$, the above complexity reduces to $O(1/\epsilon^2)$, which is a standard complexity for non-smooth convex optimization~\citep{nesterov2018lectures}. In addition, the dominating term $O(n/(S\epsilon^2))$ matches the lower bound proved in  \citet{wangnear}.  

\subsection{Non-convex Case}

    Now we consider the non-convex case when $g_i$ is weakly convex as stated in the following assumption.    \begin{asm}[weakly convexity of $g_i$]
    \label{asm:lip_ncvx}
        For any $i\in[n]$, we suppose that $g_i(\u,\s^{(i)})$ is $\rho$-weakly convex to $\u$ and $\s^{(i)}$ for any $\u\in\U$ and $\s^{(i)}\in\S_i$, i.e., $g_i(\cdot) + \frac{\rho}{2}\Norm{\cdot}^2$ is convex, where $\rho$ is a positive number.
    \end{asm} 

It is sometimes  difficult to find an $\epsilon$-stationary point $(\u,\s)$ of the non-smooth function $F$, i.e., $\dist(0, \partial F (\u,\s)) \leq \epsilon$. For example, an $\epsilon$-stationary point of function $f(\x) = |\x|$ does not exist for $0 \leq \epsilon < 1$ unless it is the optimal solution. To address this problem, \citep{davis2018stochastic} proposed using the stationarity of the Moreau envelope of the problem as the convergence metric, which has become a standard metric for solving weakly convex problems.
    
    Given a $\rho$-weakly convex function $f: \R^d \rightarrow \R$, its Moreau envelope is constructed as
    \begin{align}
    \label{eq:moreau_envelope}
        f_{\gamma}(\x)\coloneqq \min_{\w\in\R^d}\crk{f(\w)+ \frac{1}{2\gamma}\Norm{\w-\x}^2},
    \end{align}
    where $\gamma$ is a positive constant. For a $\rho$-weakly convex function $f$, it can be shown that $f_{\gamma}$ is smooth when $\frac{1}{\gamma} > \rho$ \citep{davis2019proximally} and its gradient is
    \begin{align}
    \label{eq:proximal_map}
        \nabla f_{\gamma}(\x) = \frac{1}{\gamma}(\x-\prox_{\gamma}f(\x)), 
    \end{align}
    where
    \begin{align}
        \prox_{\gamma}f(\x) \coloneqq \argmin_{\w}\{f(\w)+ \frac{1}{2\gamma}\Norm{\w-\x}^2\}.
    \end{align}
    Notice that when $\frac{1}{\gamma} > \rho$, the minimization in problem (\ref{eq:moreau_envelope}) is strongly convex, which ensures $\prox_{\gamma}f(\x)$ is uniquely defined. Moreover, for any point $\x \in \R^d$, the proximal point $\x^{\dagger} \coloneqq \prox_{\gamma}f(\x)$ satisfies \citep{hu2023non}
    \begin{align}
        &\Norm{\x^{\dagger}-\x} = \gamma\Norm{\nabla f_{\gamma}(\x)}, \quad f_{\gamma}(\x^{\dagger}) \leq f_{\gamma}(\u), \quad \dist(0, \partial f(\x^{\dagger})) \leq \Norm{\nabla f_{\gamma}(\x)}.
    \end{align}
    Thus if $\Norm{\nabla f_{\gamma}(\x)} \leq \epsilon$, we can say $\x$ is close to a point $\x^{\dagger}$ that is $\epsilon$-stationary, which is called nearly $\epsilon$-stationary solution of $f(\x)$. Given an iterate $\x_t$, a common idea is using the stochastic subgradient method (SSG) to approximately solve (\ref{eq:moreau_envelope}) with $\x = \x_t$, namely, to compute a solution $\x_{t+1}$ such that 
    \begin{align}
    \label{eq:subgrad}
        \x_{t+1} \approx \prox_{\gamma}(\x_t) = \argmin_{\x} \crk{f(\x) + \frac{1}{2\gamma}\Norm{\x-\x_t}^2}.
    \end{align}
    Then $\x_{t+1}$ returned by the SSG method will then be used in the next iterate.  Inspired by \citet{rafique2022weakly}, we consider the following update according to equation (\ref{eq:subgrad})
    \begin{align}
        &(\u_{t+1},\s_{t+1},\y_{t+1}) \approx \argmin_{\u\in\U,\s\in\S} \argmax_{\y\in\Y} \left\{L_{\gamma}(\u,\s,\y;\u_t,\s_t)\right\}, \nonumber\\
        &\text{where }L_{\gamma}(\u,\s,\y;\u',\s') \coloneqq \frac{1}{n}\sum_{i=1}^n \left(\y^{(i)}g_i(\u,\s^{(i)}) - f_i^*(\y^{(i)})\right) + \frac{1}{2\gamma}\Norm{\u-\u'}^2 + \frac{1}{2n\gamma}\Norm{\s-\s'}^2.   
    \end{align}

    \begin{theorem}
    \label{thm:iteration_complexity_nvx}
        Under Assumptions \ref{asm:lip}, \ref{asm:var} and \ref{asm:lip_ncvx},
    STACO2  with $\gamma \leq \frac{1}{2C_f\rho}$, $\eta_t = \order{\epsilon^2}$, $\beta_t = \order{\epsilon^2}$, $\alpha_t = \order{B\epsilon^2}$, and $K_t = \order{\frac{n}{BS\epsilon^4}\lor \frac{1}{\eta_t}\lor \frac{n}{S\beta_t}}$ can converge to an $\epsilon$-stationary point of $\Phi_\gamma(\u, \s)$ in $ \order{\frac{ C_f^2\sigma_1^2}{B\epsilon^4} + \frac{\delta^2}{S\epsilon^4} + \frac{nC_g^2C_f^2}{S\epsilon^4} + \frac{nC_f^2\sigma_2^2}{BS\epsilon^4} + \frac{n\sigma_0^2}{BS\epsilon^6}}$ iterations, where  $\Phi_\gamma(\u, \s)=\min_{\tilde{\u},\tilde{\s}}F(\tilde{\u},\tilde{\s})+\frac{1}{2\gamma}\Norm{\tilde{\u}-\u}^2+\frac{1}{2\gamma n}\Norm{\tilde{\s}-\s}^2$ is a Moreau envelope of $F(\u, \s)$.
     \end{theorem}


    \textbf{Remark} The proof is included in Appendix \ref{app:iteration_complexity_nvx}. We compare the above result with the complexity of SOTA and  SONX. In particular, SOTA has a complexity of $\order{\frac{n}{\epsilon^6}}$ result, which cannot show any mini-batch speedup. 
     SONX has an iteration complexity of $\order{\frac{n}{B^{1/2}S\epsilon^6}}$ in Theorem C.4 \citep{hu2023non}. 
     In comparison, our complexity $\order{\frac{n}{BS\epsilon^6}}$  has a better dependence on $B$. 

\section{Experiments}
We evaluate the empirical performance of our proposed algorithm against baselines for Two-way Partial AUC Maximization (TPAUC) in a convex setting for learning linear models and a non-convex setting for learning deep models.

\subsection{Settings}
    {\bf Datasets.} For linear model experiments, we use three datasets in \citep{chang2011libsvm}, namely HIGGS, SUSY, and ijcnn1. 
    For SUSY and HIGGS, we use the first 80\% of the data as the training dataset and the remaining 20\% as the testing dataset. For ijcnn1, we follow the existing split in \citep{chang2011libsvm}. To create imbalanced datasets for HIGGS and SUSY (ijcnn1 itself is imbalanced), we randomly remove 99.5\% positive data. For deep learning model experiments, we use two molecule datasets from the Stanford Open Graph Benchmark (OGB) website \citep{hu2020open} and two biomedical image datasets from MedMNIST \citep{medmnistv2}, namely moltox21 (the No.0 target), molmuv (the No.1 target), nodulemnist3d, and adrenalmnist3d. Those four datasets are naturally imbalanced. The task in molecular datasets is to predict certain properties of molecules, and the task in biomedical image datasets is binary classification. The statistics of datasets are presented in Table \ref{tb:dataset_sta} in Appendix \ref{sec:add}.

    {\bf Models.} For linear model experiments, we let $h_{\w}(\x)= \w^{\top}\x$. In deep model experiments, for molecule datasets moltox21 and molmuv, we use Graph Isomorphism Network (GIN) \citep{xu2018powerful} as the backbone model, which has 5 mean-pooling layers with 64 hidden units and 0.5 dropout rate. For image datasets nodulemnist3d and adrenalmnist3d, we learn a convolutional neural network (CNN) and use ResNet18 \citep{he2016deep}. We utilize the sigmoid function for the final output layer to generate the prediction score and set the surrogate loss $\ell(\cdot)$ as the squared hinge loss with a margin parameter~\citep{zhu2022auc}. 

    {\bf Baselines.} We evaluate our algorithms, STACO1 and STACO2, by comparing their training and testing performance against various baselines, while STACO1 is for linear model and STACO2 is for deep model. Specifically, we benchmark our methods against other approaches that optimizes different objectives, including CE for optimizing the cross-entropy loss, AUCM for optimizing an AUC min-max margin loss \citep{yuan2021federated}, SOTAs for optimizing a soft TPAUC loss \citep{zhu2022auc}, SOTA~\citep{zhu2022auc}, and SONX for optimizing the same TPAUC loss as ours \citep{hu2023non}, and PAUCI for optimizing an instance-wise TPAUC loss \citep{shao2022asymptotically}. 

    {\bf Evaluation Metrics.} For linear and deep learning model experiments, we evaluate TPAUC with two settings, i.e., TPR $\geq 0.5$ and FPR $\leq 0.5$, and TPR $\geq 0.25$ and FPR $\leq 0.75$.

    {\bf Hyperparameter Tuning.} In linear model experiments, the model is trained by 3000 iterations, and the learning rate is decreased by 10-fold on the 500th, 1500th, and 2500th iterations for all methods. For deep learning experiments, the model is trained by 60 epochs and the learning rate is decreased by 10-fold after every 20 epochs for all methods. In addition, we pre-train the model for deep learning experiments following previous studies~\citep{yuan2021federated,zhu2022auc}. The pre-trained model is trained for 60 epochs using CE loss with an Adam optimizer on the training datasets, and the initial learning rate is 1e-3 which is decreased by 10-fold on the 30th and 45th epochs. We tune the step sizes of STACO1, STACO2, SOTA, PAUCI, and AUCM in the range \{1e-2, 1e-1, 5e-1\}, and tune the step sizes of SONX, SOTAs, and CE in the range \{1e-3, 1e-2, 1e-1\}. For STACO1, STACO2, SOTA, and SONX,  we fix the margin parameter of the surrogate loss $\ell$ as 0.5, and tune the rate parameter $\theta_0,\theta_1$ in $\{0.4, 0.5, 0.75\}$ for reporting testing performance. For SONX, we fix the moving average parameter  as 0.9 and tune the momentum parameter in the range \{0, 1e-3, 1e-2, 1e-1\}. For AUCM, we choose the momentum parameter as 0.9, the margin parameter of the surrogate loss as 0.5, and tune the hyperparameter $\gamma$ that controls consecutive epoch-regularization in \{100, 500, 1000\}. For SOTAs, we fix $\gamma_0 = \gamma_1 = 0.9$ and tune $\lambda,\lambda'$ in \{0.1, 1.0, 10\}. For PAUCI, we tune $k$ in [1, 10], $c_1, c_2, \mu, \lambda$ in [0, 1], $m$ in [10, 100] and $\kappa$ in [2, 6]. For all algorithms, we choose the weight decay parameter as 2e-4. Without specific statements, each algorithm samples 64 data points in each iteration. We execute all experiments using 5-fold-cross-validation to evaluate testing performance based on the best validation performance and report the average and standard deviation over multiple runs.

    \subsection{Results}

    \begin{figure*}[hbtp]
        \centering
        \subfloat[HIGGS]
        {\includegraphics[width=0.245\textwidth]{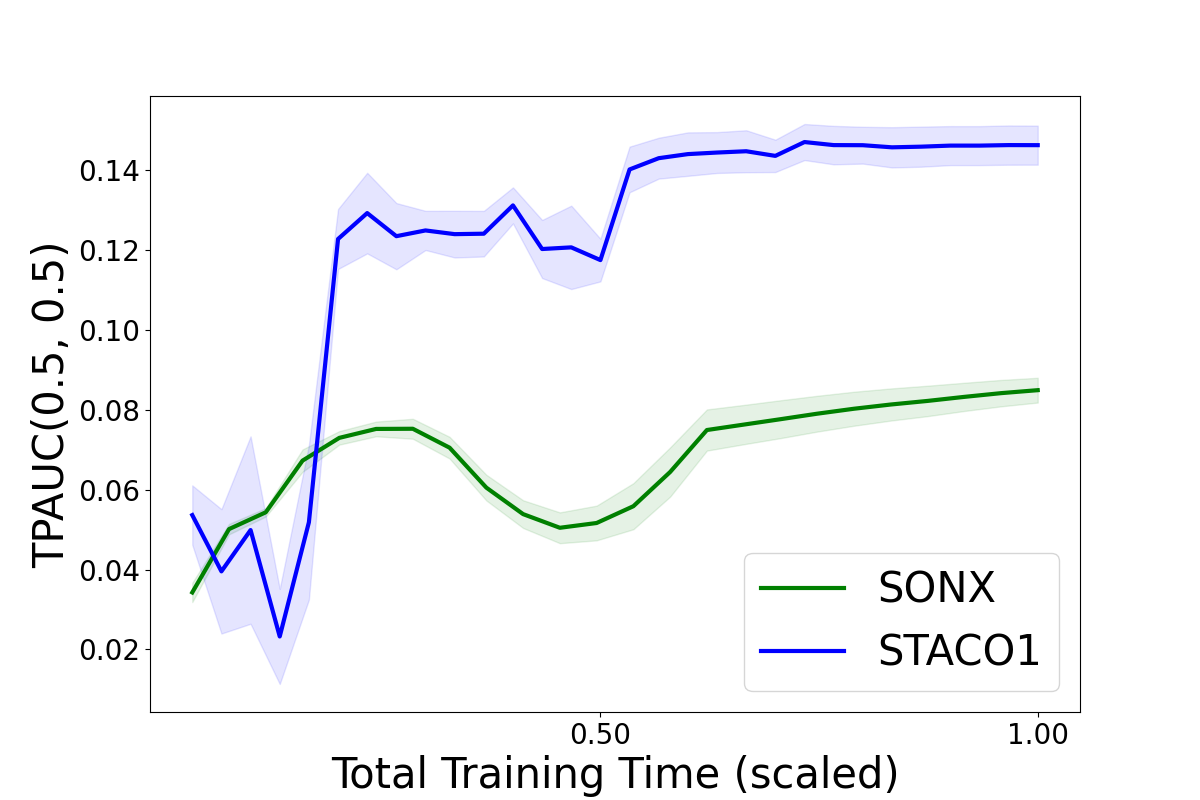}}
        \subfloat[SUSY]
        {\includegraphics[width=0.245\textwidth]{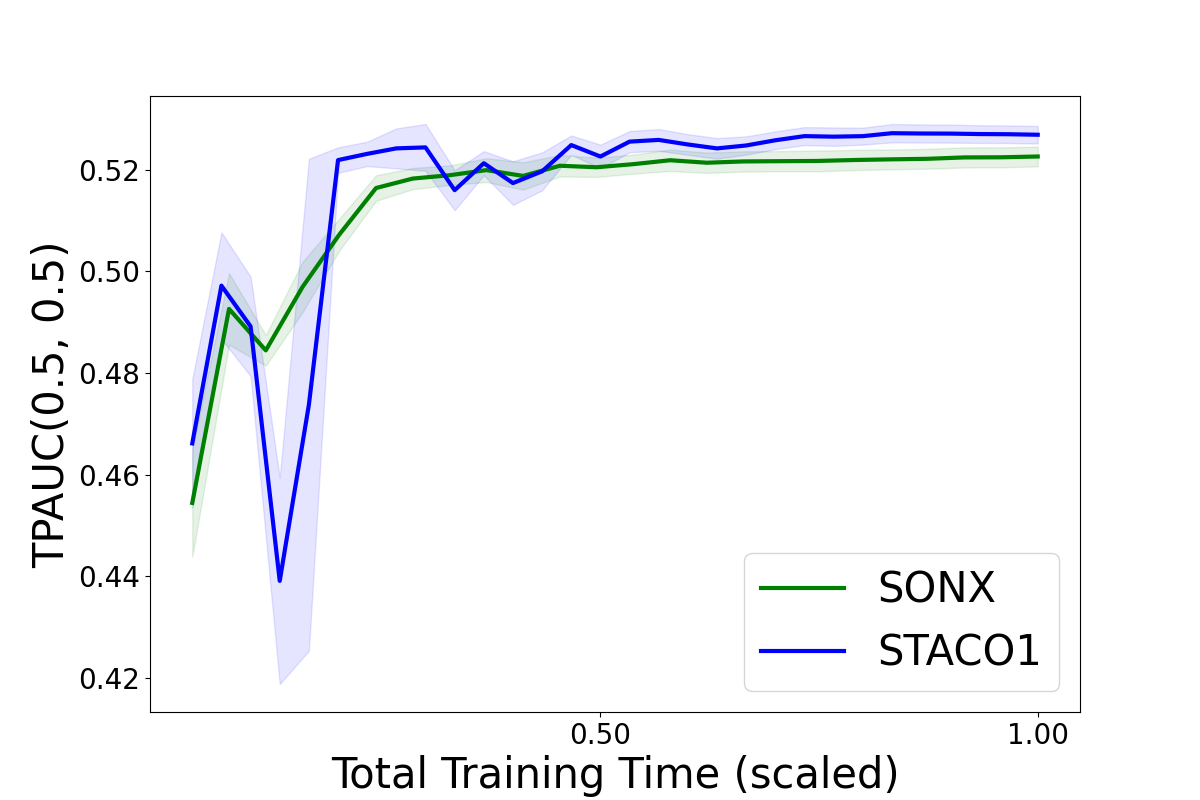}}
        \subfloat[HIGGS]
        {\includegraphics[width=0.245\textwidth]{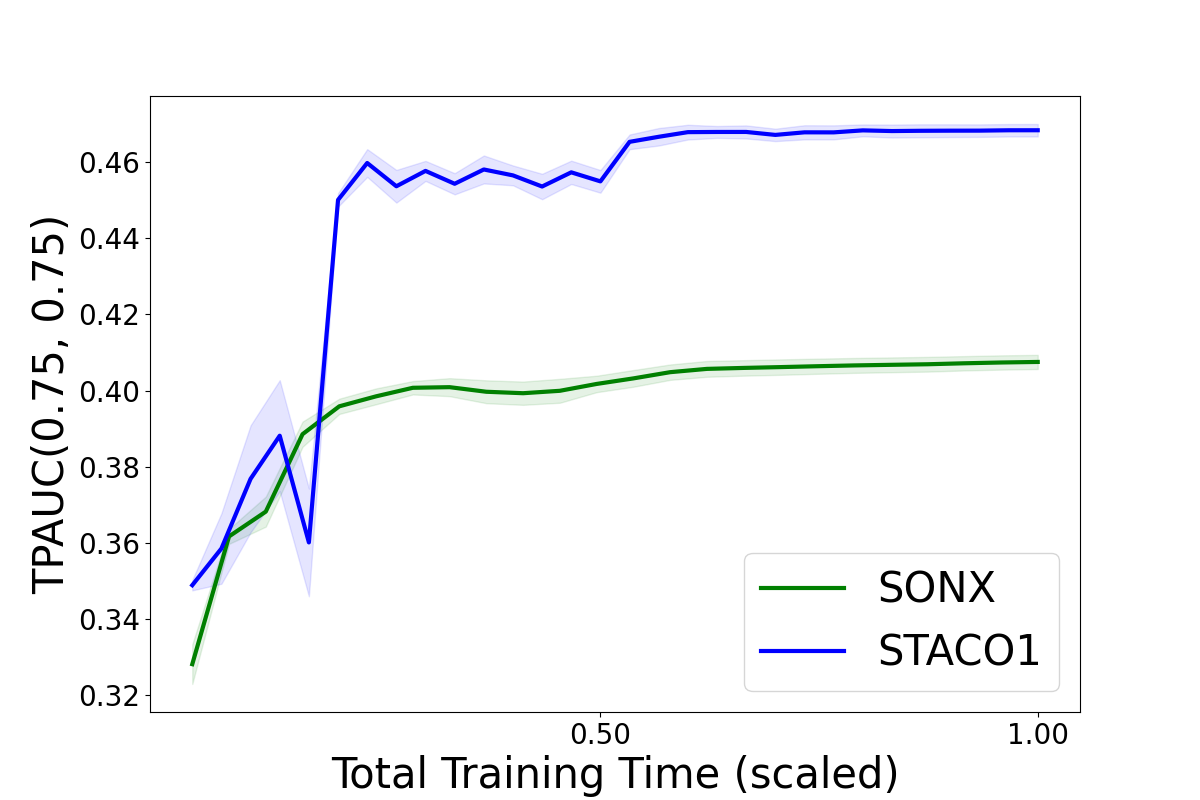}}
        \subfloat[SUSY]
        {\includegraphics[width=0.245\textwidth]{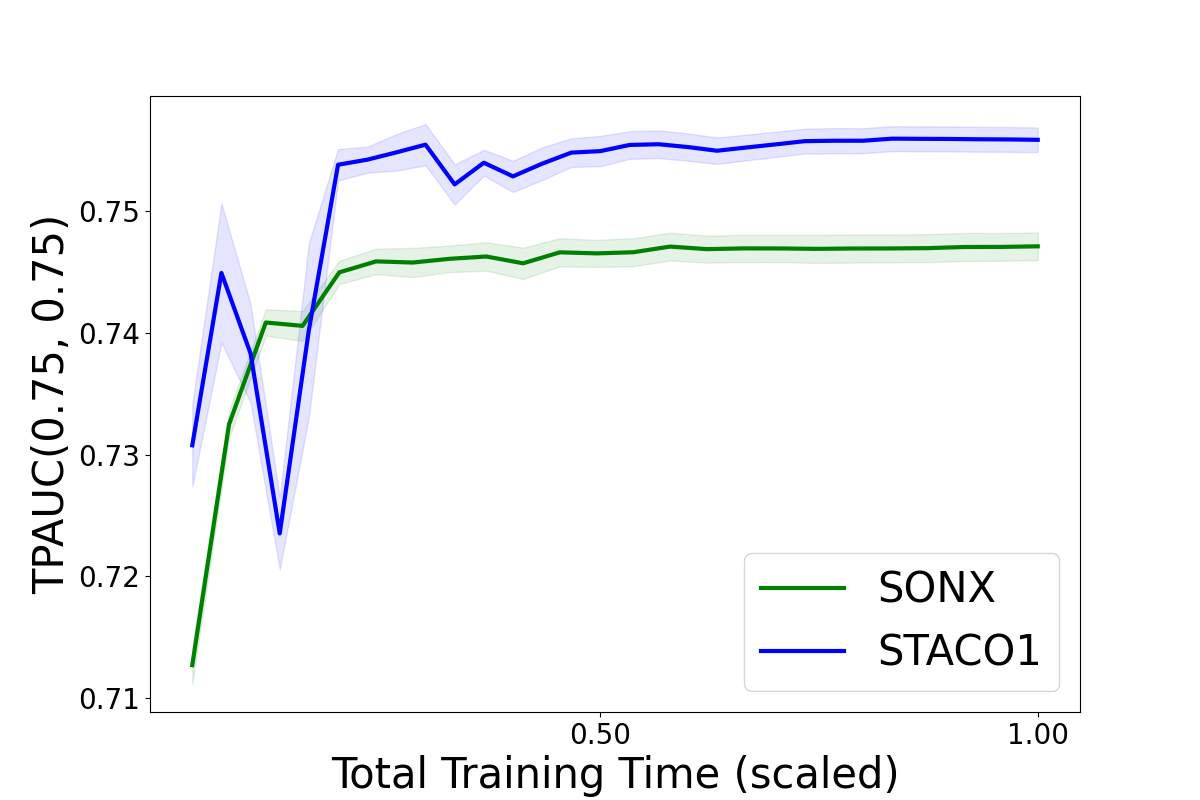}}
        \caption{Training TPAUC Curves of STACO1 and SONX on two different datasets. The first two shows the TPAUC (0.5, 0.5) results, and the last two shows the TPAUC (0.75, 0.75) results.}
        \label{fig:linear_tpauc}

        \centering
        \subfloat[molmuv(t1)]
        {\includegraphics[width=0.245\textwidth]{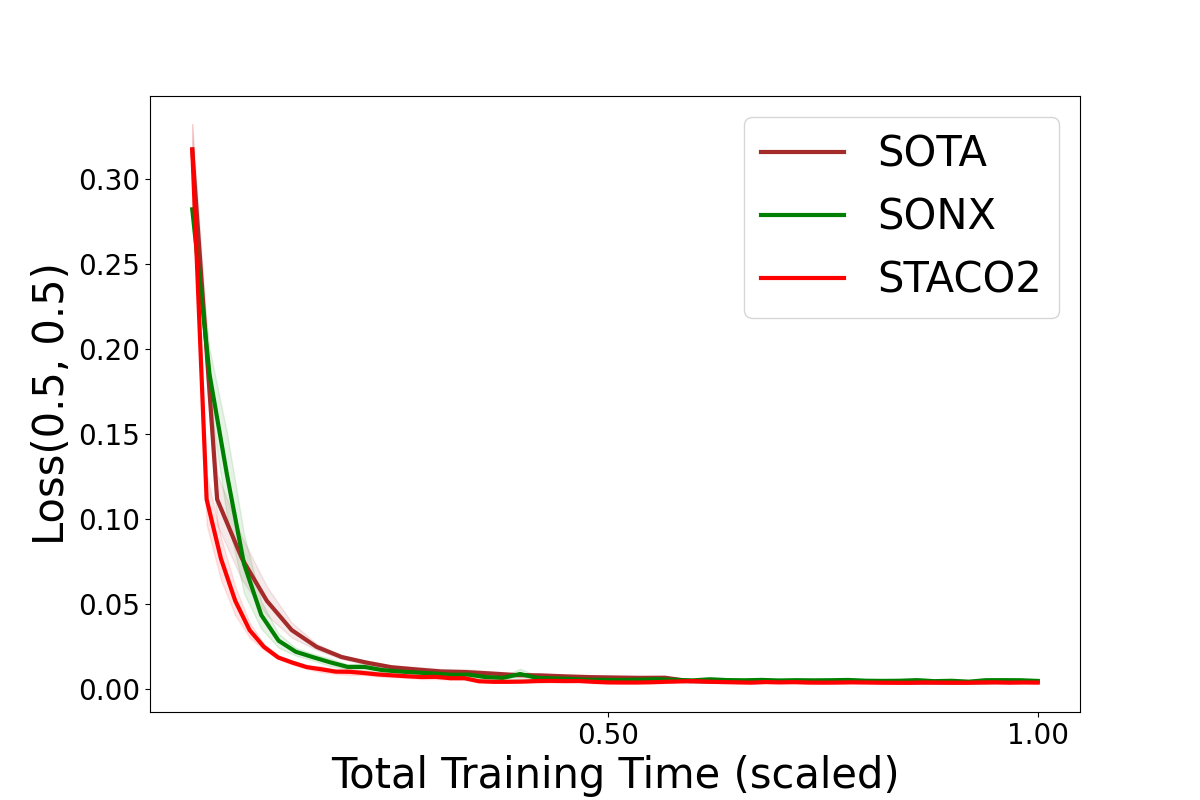}}
        \subfloat[moltox21(t0)]
        {\includegraphics[width=0.245\textwidth]{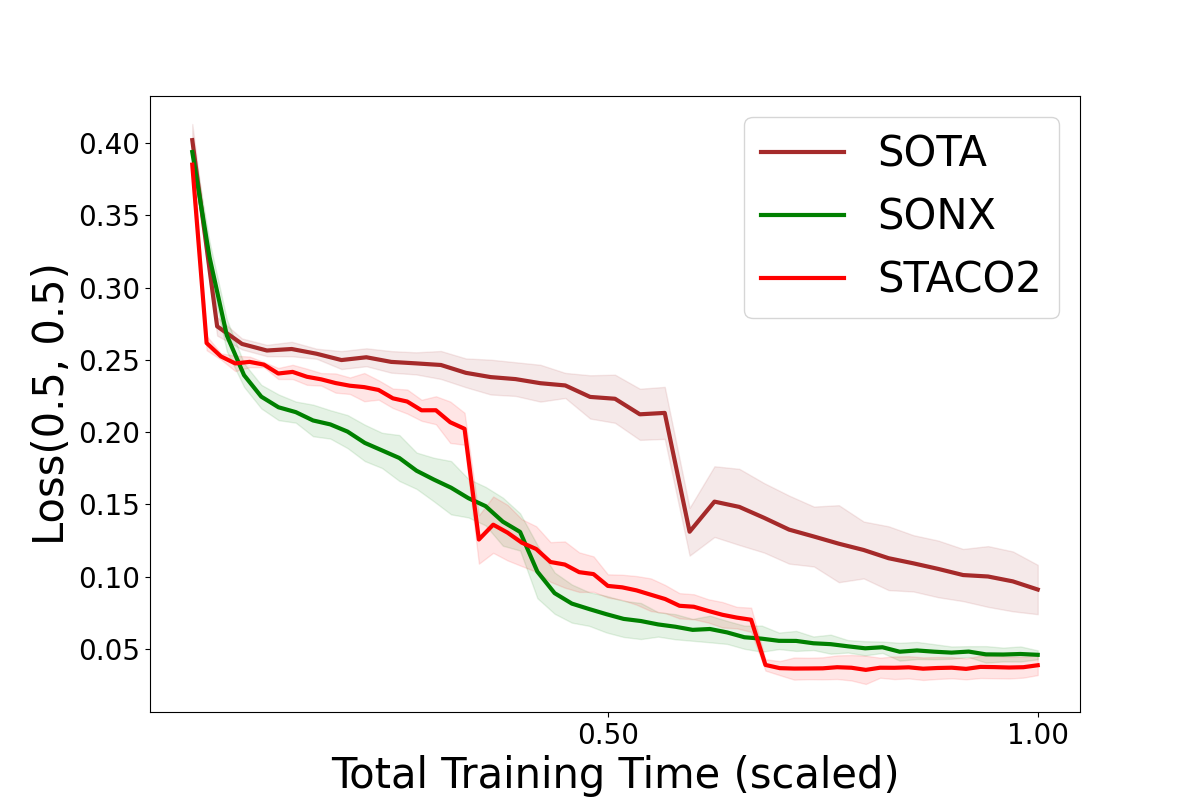}}
        \subfloat[adrenalmnist3d]
        {\includegraphics[width=0.245\textwidth]{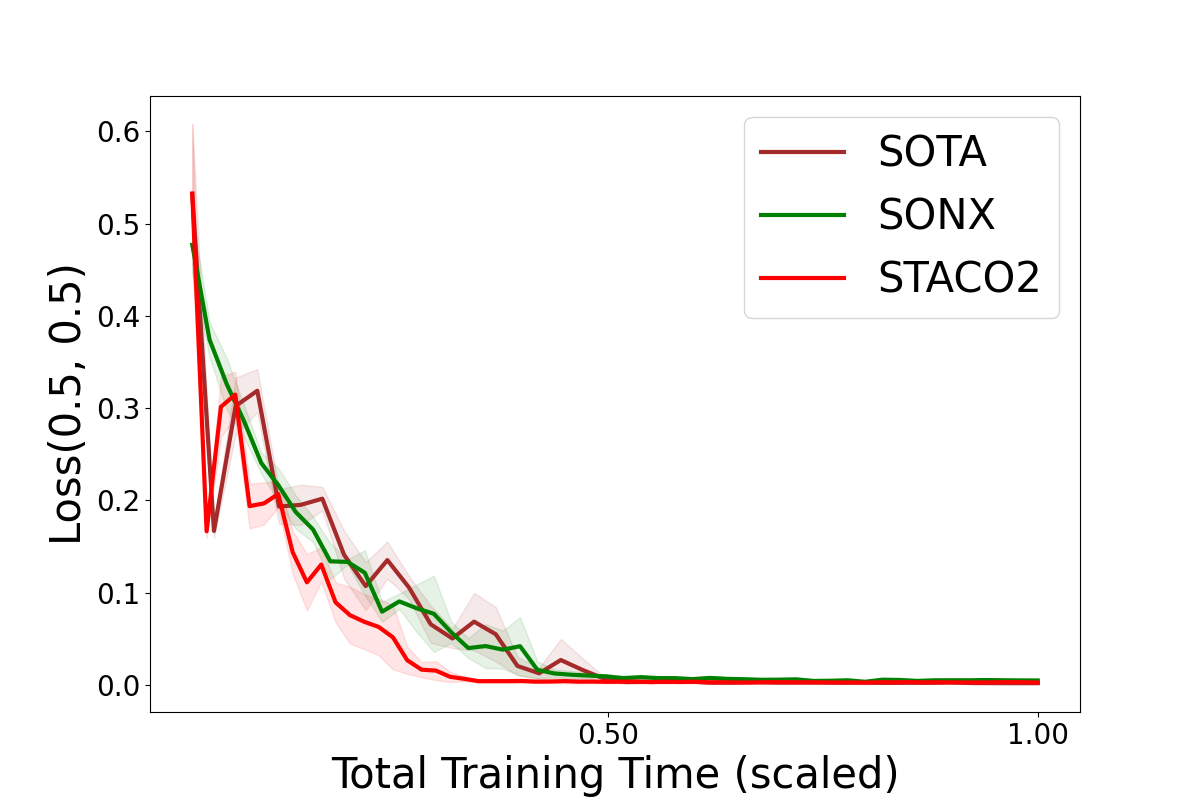}}
        \subfloat[nodulemnist3d]
        {\includegraphics[width=0.245\textwidth]{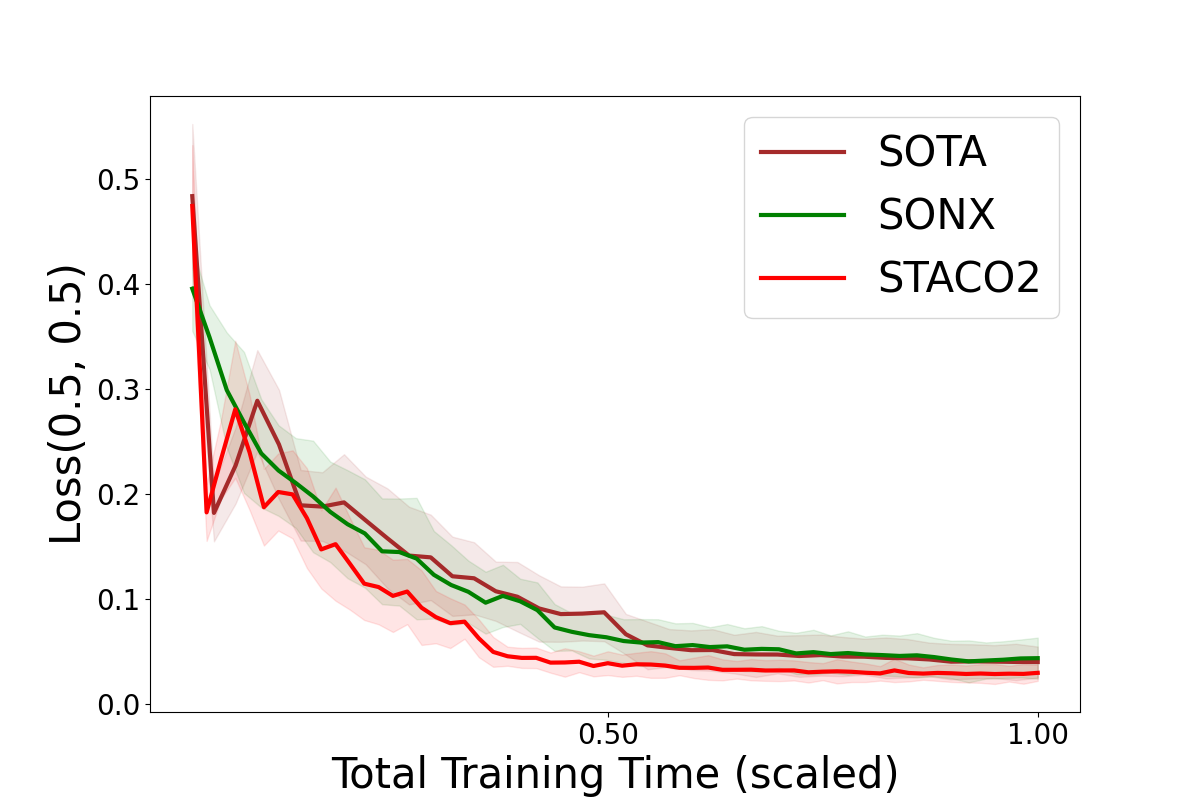}}\\
        \subfloat[molmuv(t1)]
        {\includegraphics[width=0.245\textwidth]{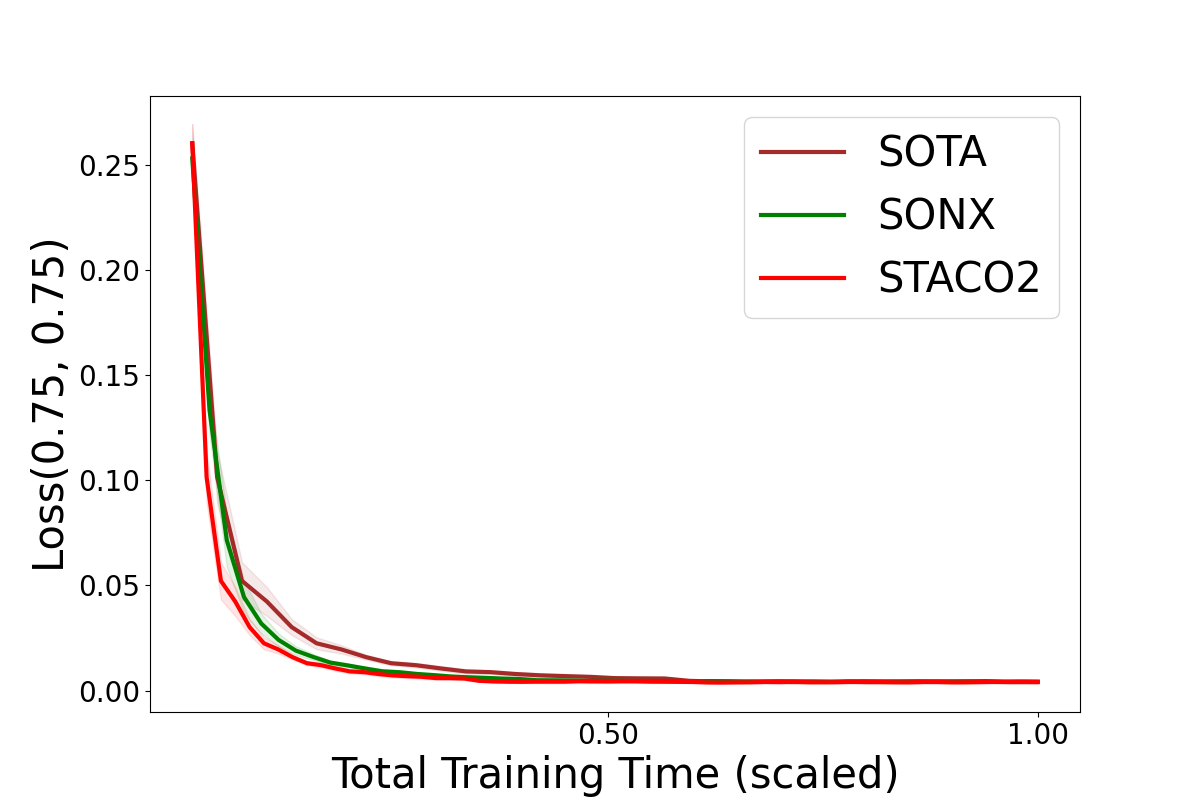}}
        \subfloat[moltox21(t0)]
        {\includegraphics[width=0.245\textwidth]{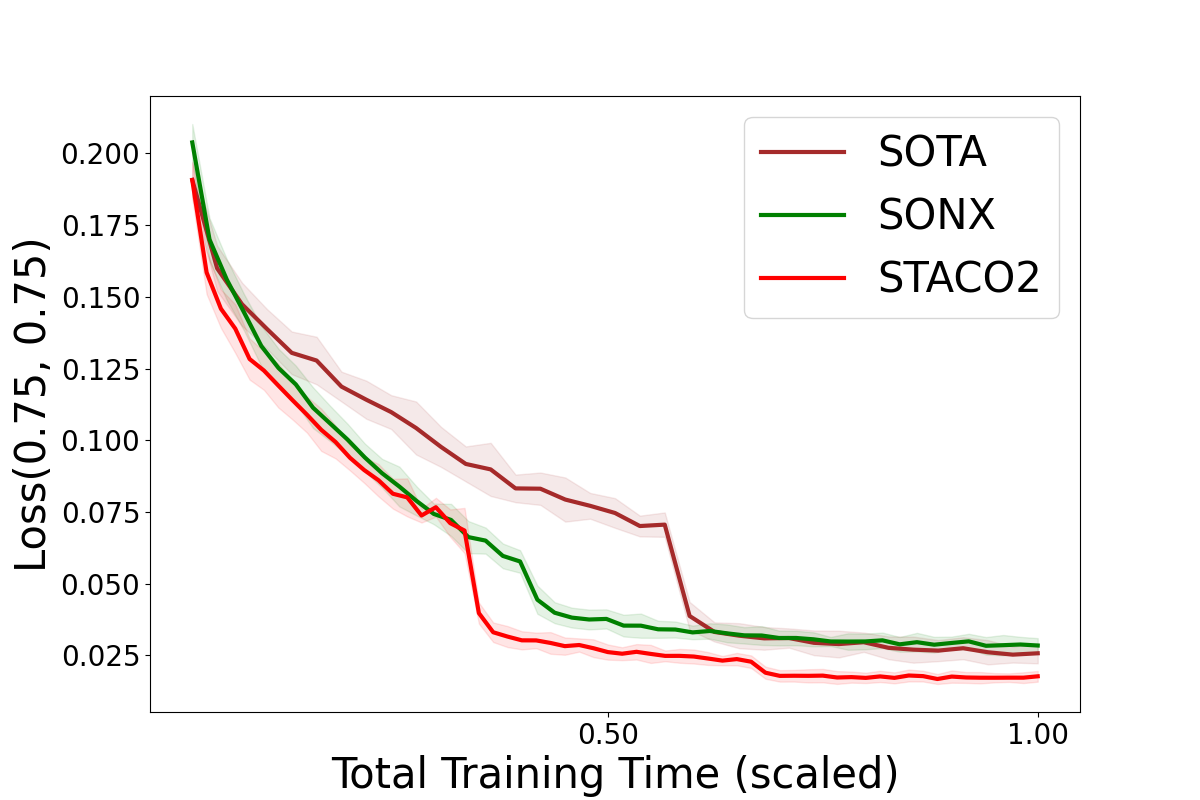}}
        \subfloat[adrenalmnist3d]
        {\includegraphics[width=0.245\textwidth]{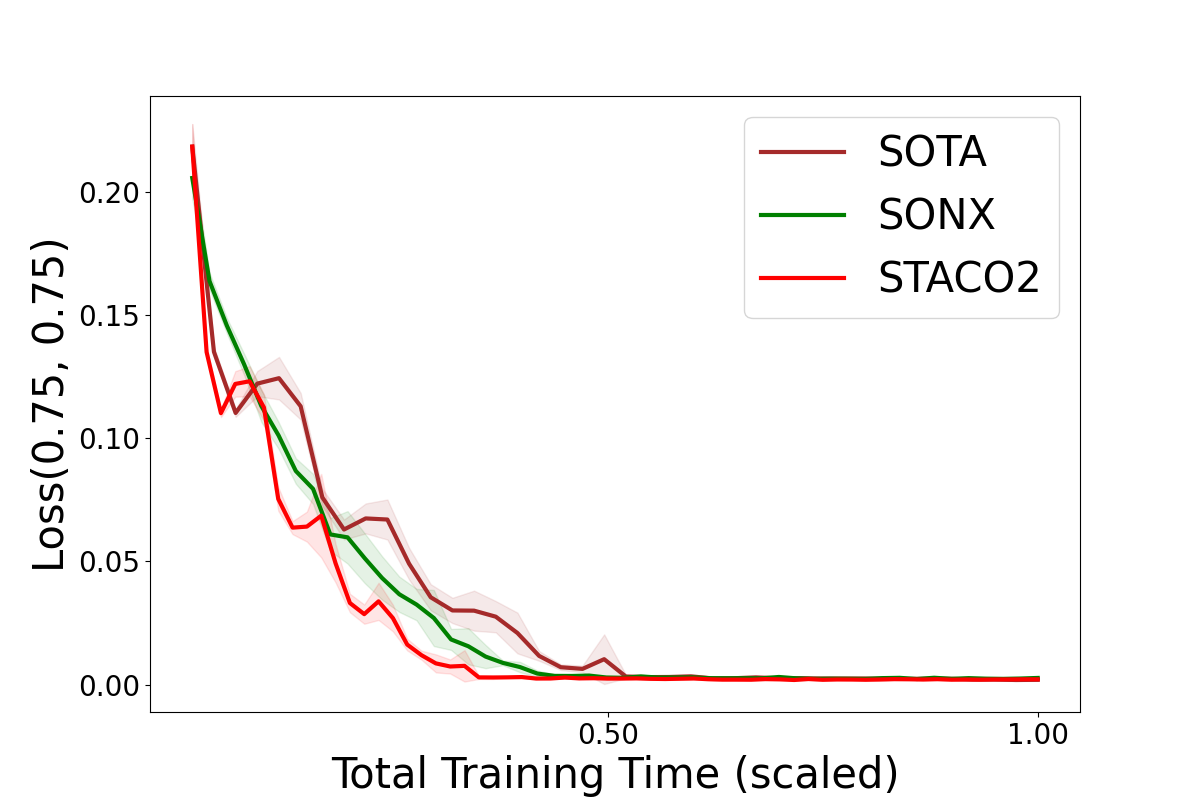}}
        \subfloat[nodulemnist3d]
        {\includegraphics[width=0.245\textwidth]{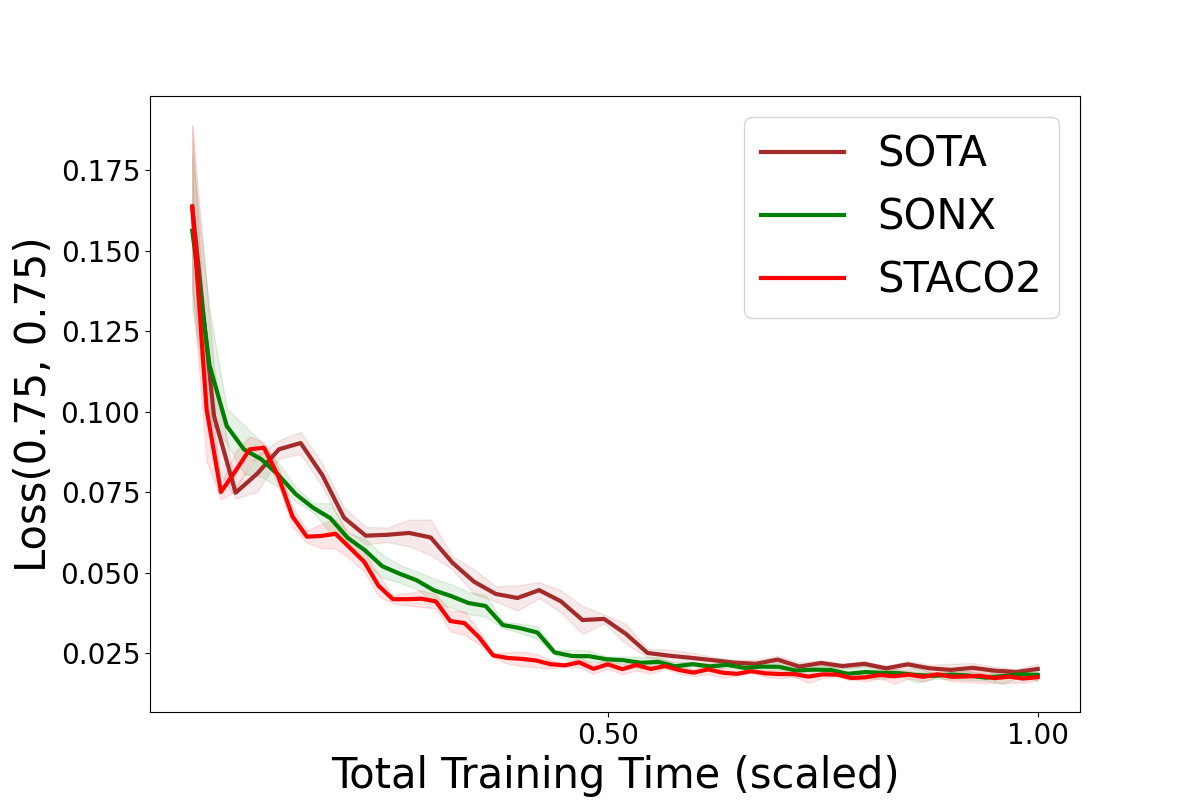}}
        \caption{Training Loss Curves of STACO2, SOTA, and SONX on four different datasets. The first row shows the Loss (0.5, 0.5) results, and the second row shows the Loss (0.75, 0.75) results.}
        \label{fig:nonlinear_loss}

        \centering
        \subfloat[Batch Size=16]
        {\includegraphics[width=0.25\textwidth]{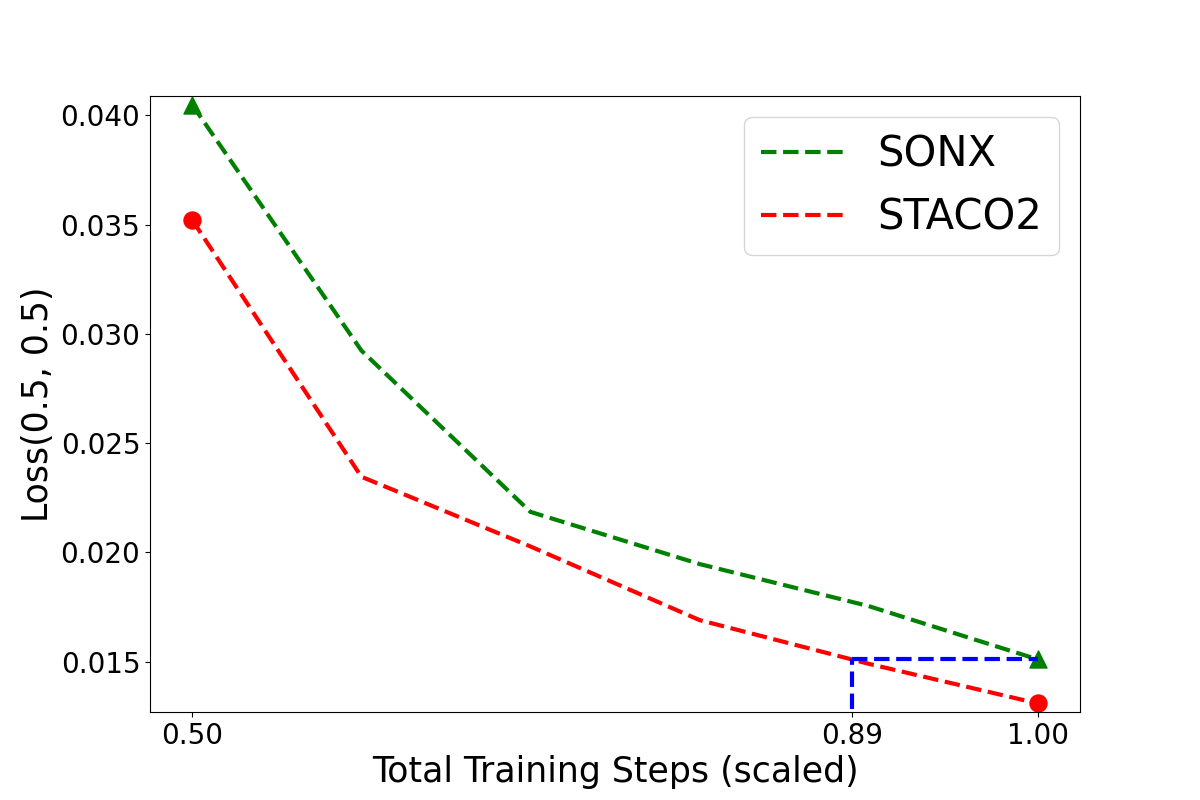}}
        \subfloat[Batch Size=64]
        {\includegraphics[width=0.25\textwidth]{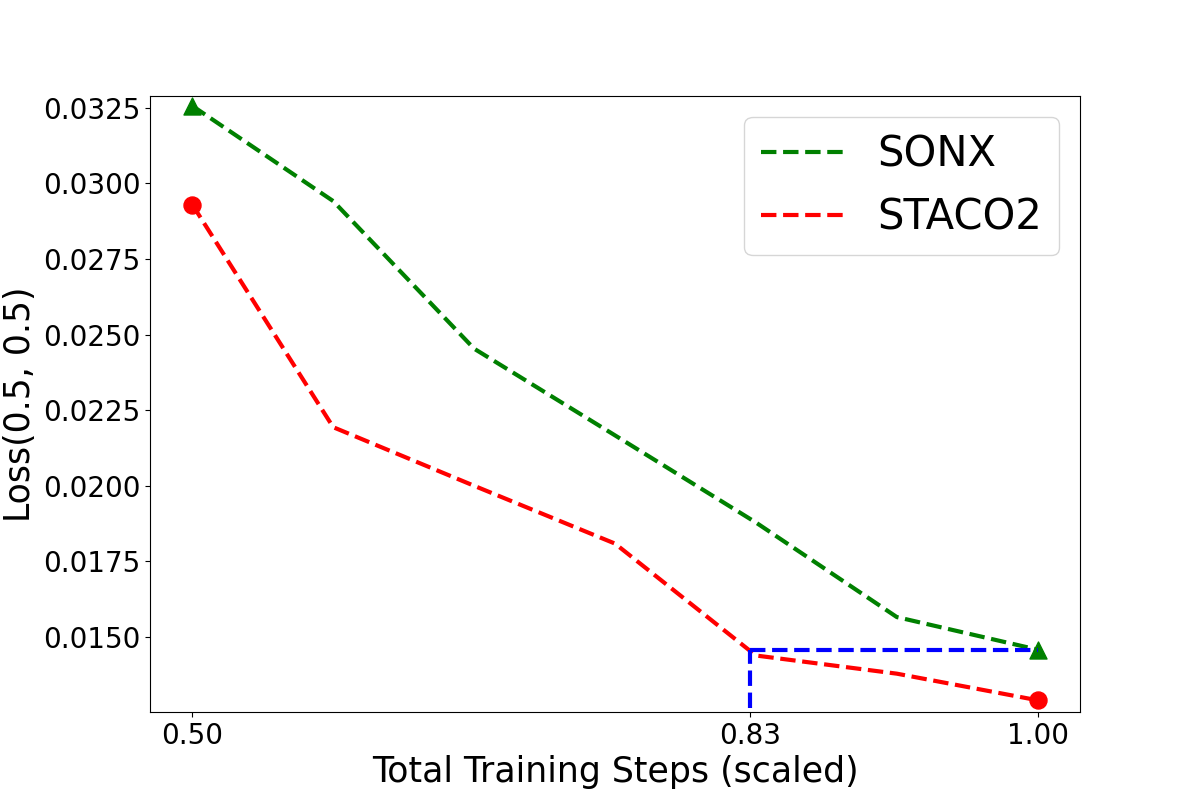}}
        \subfloat[Batch Size=256]
        {\includegraphics[width=0.25\textwidth]{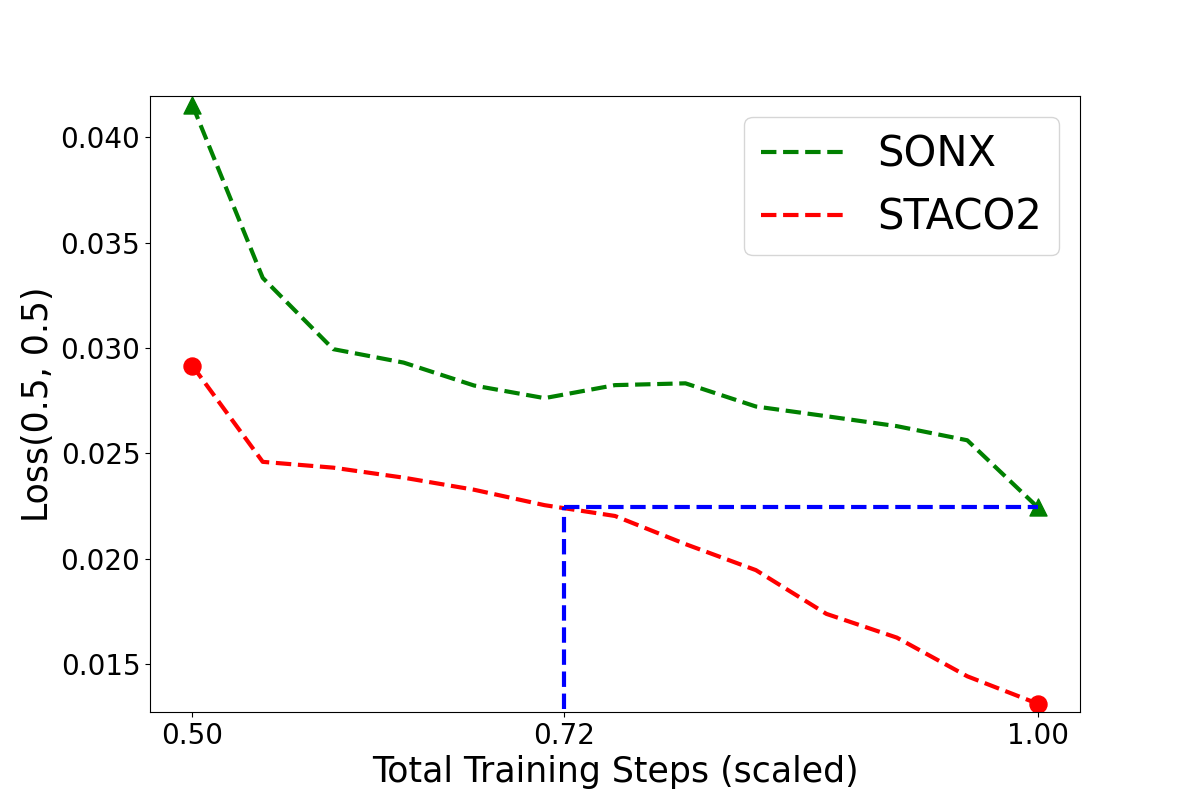}}
        \caption{Negative sample batch size ($B$) benefits of STACO2 over SONX for training on ogbg-molmuv (t1) at 16, 64, and 256 batch size.}
        \label{fig:nonlinear_tr_0.5_bs}
    \end{figure*}

    \begin{table*}
    \centering
    \caption{TPAUC on the test data of linear and deep models. $(\theta_0, \theta_1)$ represents TPR $\geq 1-\theta_0$, FPR $\leq \theta_1$. Results are reported as mean(std).}
           \scalebox{0.73}{	
    \begin{tabular}{c|l|ccc|cccc}
    \toprule
    \multirow{2}{*}{\textbf{Metrics}} & \multirow{2}{*}{\textbf{Methods}} & \multicolumn{3}{c|}{{\bf Linear Model}} & \multicolumn{4}{c}{{\bf Deep Model}}\\\cline{3-9}
    & & HIGGS                & SUSY          & ijcnn1  & molmuv(t1) & moltox21(t0) &  nodulemnist3d &  adrenalmnist3d  \\ \hline
                                           & CE     & 0.041(0.001)                  & 0.300(0.010)          & 0.230(0.017)                                      & 0.715(0.166)           & 0.267(0.042)             & 0.657(0.037)                 & 0.507(0.094)                  \\ 
                                           & AUCM             & 0.122(0.001)         & 0.512(0.015)          & 0.487(0.098)                       & 0.722(0.114)           & 0.279(0.038)             & 0.672(0.021)                 & \textbf{0.554(0.022)}                  \\ 
    \multirow{2}{*}{$(0.5, 0.5)$}     & SOTAs                     & 0.108(0.001)     & 0.484(0.001)          & 0.637(0.030)                      & 0.821(0.110)           & 0.325(0.030)             & 0.688(0.019)                 & 0.498(0.090)                  \\ 
                                           & PAUCI	 & 0.138(0.002)	& 0.519(0.001)	& 0.664(0.018)	& 0.820(0.046)	& 0.283(0.032)	& 0.684(0.021)	& 0.541(0.042) \\
                                           & SONX            & 0.110(0.009)              & 0.516(0.001)          & 0.633(0.094)                               & 0.865(0.061)           & 0.286(0.023)             & 0.654(0.035)                 & 0.540(0.042)                  \\ 
                                           & {\bf STACO}     & \textbf{0.158(0.003)}       & \textbf{0.520(0.001)} & \textbf{0.682(0.054)}                            & \textbf{0.904(0.048)}  & \textbf{0.325(0.023)}    & \textbf{0.707(0.005)}        & 0.546(0.047)         \\ \hline
                                           & CE                          & 0.354(0.002)                   & 0.612(0.006)          & 0.581(0.014)                       & 0.871(0.058)            & 0.627(0.035)              & 0.825(0.016)                 & 0.750(0.055)         \\ 
                                           & AUCM                    & 0.435(0.004)                  & 0.726(0.002)          & 0.728(0.061)                       & 0.851(0.066)            & 0.630(0.027)              & 0.831(0.016)                 & 0.772(0.014)                  \\ 
    \multirow{2}{*}{$(0.75, 0.75)$}   & SOTAs  & 0.441(0.002)                 & 0.746(0.009)          & 0.813(0.016)                                       & 0.821(0.070)            & 0.614(0.056)              & 0.838(0.012)                 & 0.763(0.054)                  \\
                                           & PAUCI	& 0.4742(0.003)	& 0.749(0.007)	& 0.830(0.030)	& 0.883(0.024)	& 0.616(0.030)	& 0.823(0.014)	&0.7642(0.015) \\
                                           & SONX                    & 0.447(0.009)                   & 0.748(0.000)          & 0.810(0.049)                & 0.927(0.029)            & 0.626(0.028)              & 0.832(0.013)                 & 0.772(0.021)                 \\ 
                                           & {\bf STACO}         & \textbf{0.484(0.004)} & \textbf{0.752(0.000)} & \textbf{0.839(0.024)}                  & \textbf{0.945(0.024)}   & \textbf{0.638(0.041)}     & \textbf{0.856(0.003)}        & \textbf{0.780(0.013)}                 \\ 
                                           \bottomrule
    \end{tabular}}
    \label{tab:te}
    \vspace*{-0.1in}
    \end{table*}

    \textbf{Training Results.} Under two different metrics, we compare the training performance of the linear model between STACO1 and SONX in Figure \ref{fig:linear_tpauc}, and the deep learning model among STACO2, SOTA, and SONX in Figure \ref{fig:nonlinear_loss}. We exclude SOTA from linear model experiments since SOTA is designed for optimizing deep learning models. In the linear model experiments as shown in Figure \ref{fig:linear_tpauc}, we plot the TPAUC values throughout the training process. The results demonstrate that STACO1 exhibits strong and stable performance, consistently outperforming SONX on both the HIGGS and SUSY datasets in the (0.5, 0.5) and (0.75, 0.75) settings. These findings indicate that STACO1 is more efficient than SONX in maximizing TPAUC. 
    {We also observed that in Figure \ref{fig:linear_tpauc}, across all datasets, there is an abrupt drop and subsequent rise in performance. This is due to the excessively large step size. Once the step size is reduced,  training returns to normal.}
    
    In the nonlinear model experiments as shown in Figure \ref{fig:nonlinear_loss}, STACO2 demonstrates competitive performance in terms of training loss reduction across all four datasets compared to SONX and SOTA. In both the (0.5, 0.5) and (0.75, 0.75) settings, STACO2 achieves lower or comparable loss values while maintaining a stable training trajectory. These results indicate that STACO2 is effective in minimizing loss and optimizing model performance, further supporting its advantage over SONX and SOTA.
    
    Due to space limit, we present more training results in Figure \ref{fig:linear_loss}, \ref{fig:nonlinear_tpauc} in Appendix \ref{sec:add}.

    \textbf{Testing Results.} Under two different metrics, we present the testing results for linear and deep learning models in Table \ref{tab:te}. For the linear model, STACO1 consistently outperforms the baseline methods across various datasets, demonstrating its robustness and strong generalization capability across different datasets and evaluation criteria. Similarly, for the nonlinear model, STACO2 achieves significant improvements over existing methods. Notably, compared to SONX, STACO2 exhibits a more pronounced advantage in testing performance than in training, suggesting superior generalization when optimizing the exact TPAUC loss.

    We do not include SOTA \citep{zhu2022auc} in the above comparison, since SOTA is quite similar to STACO2 thus they have similar testing results. However, we must point out that the convergence of SOTA is much slower than STACO2 since it has to update all the coordinates of $\s$ in problem (\ref{eq:tpauc_ori}). As shown in Figure \ref{fig:nonlinear_loss}, STACO2 is significantly faster than SOTA.

    \begin{figure}[t]
        \centering

        \centering
        \subfloat[moltox21(t0)]
        {\includegraphics[width=0.245\textwidth]{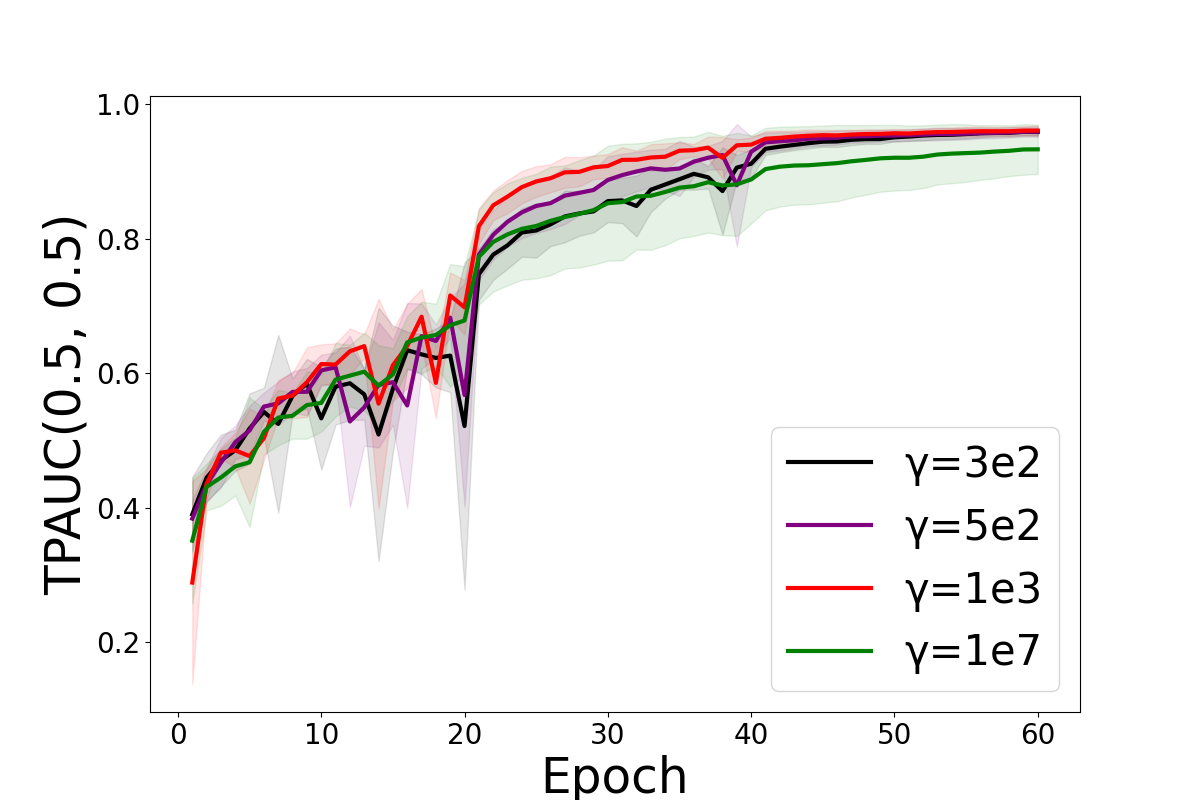}}
        \subfloat[nodulemnist3d]
        {\includegraphics[width=0.245\textwidth]{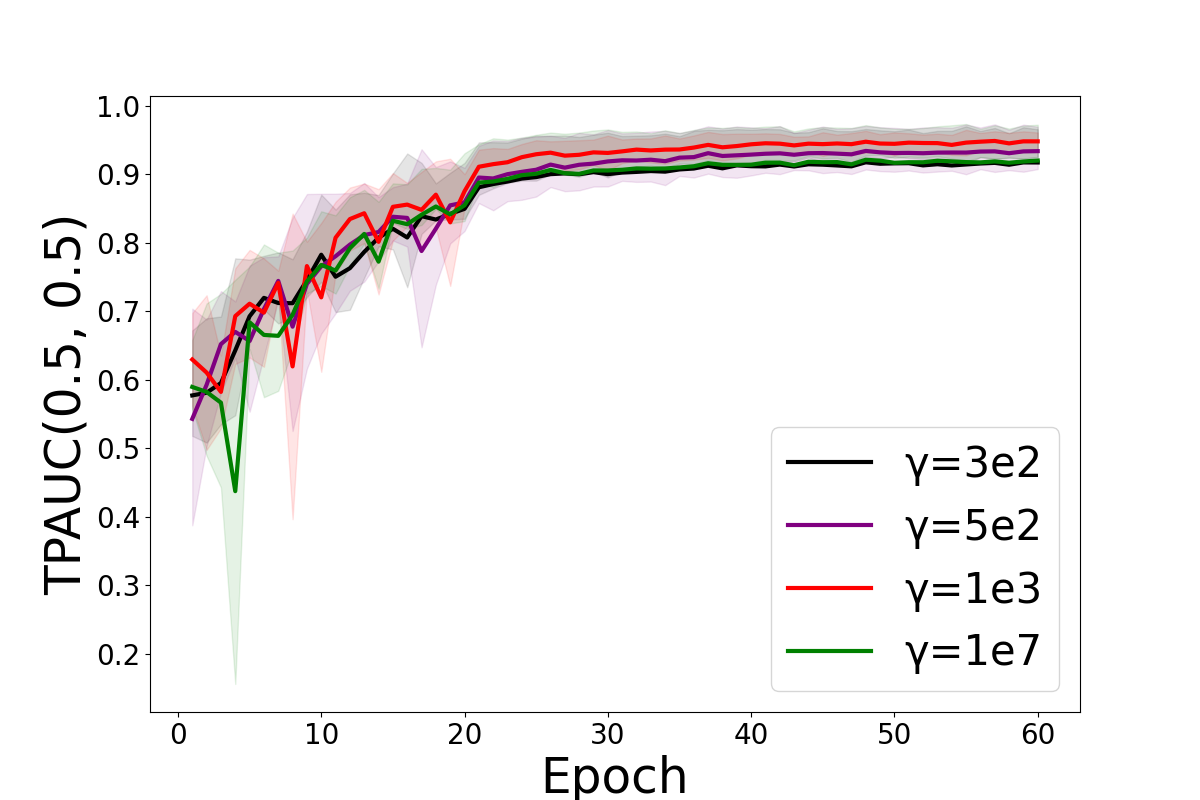}}
        \subfloat[moltox21(t0)]
        {\includegraphics[width=0.245\textwidth]{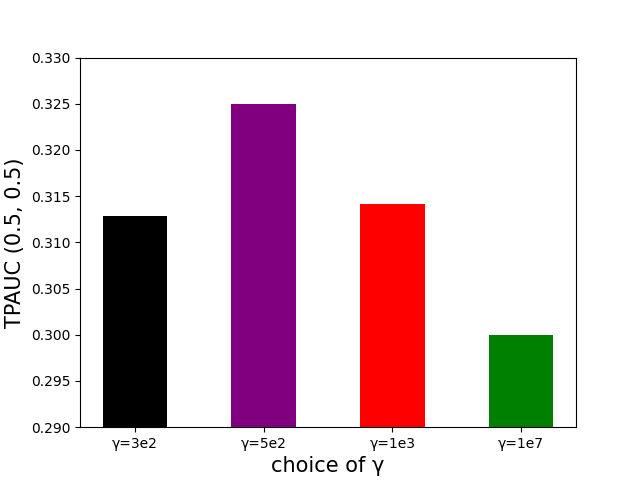}}
        \subfloat[nodulemnist3d]
        {\includegraphics[width=0.245\textwidth]{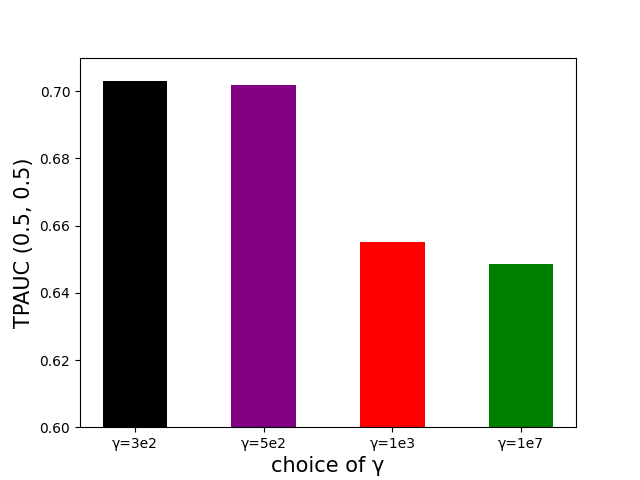}}
        \caption{First two figures shows the TPAUC (0.5,0.5) training curves of STACO2 with different $\gamma$; last two figures shows the TPAUC (0.5, 0.5) testing results of STACO2 with different $\gamma$. The experiment is conducted on datasets ogbg-moltox21(t0) and nodulemnist3d.}
        \label{fig:nonlinear_tr_0.5_gamma}
    \end{figure}

    \subsection{Ablation Study}

    \textbf{Effect of Batch Size.} We examine the impact of negative batch size $B$ on the performance of STACO2 and SONX to verify the mini-batch speedup of STACO2 over SONX. Specifically, we tune the negative batch size $B$ in [16, 64, 256]. In Figure \ref{fig:nonlinear_tr_0.5_bs}, we present the training loss curve for STACO2 and SONX on dataset ogbg-molmuv (t1). Our results show that as batch size increases, STACO2 exhibits greater convergence improvement compared to SONX, indicating that it benefits more from a larger batch size. This observation is consistent with Theorem \ref{thm:iteration_complexity_nvx}, i.e., STACO2 can achieve full mini-batch speedup than SONX. 

    \textbf{Effect of Epoch Decay Factor.} We examine the impact of epoch decay parameter $\gamma$ on the training performance of STACO2. In Theorem $\ref{thm:iteration_complexity_nvx}$, $\gamma$ must be less or equal than $\frac{1}{2C_f \rho}$, where $C_f$ is the Lipschitz constant for function $f_i$ and $\rho$ is the {weakly-convexity} parameter for function $g_i$. In TPAUC maximization problem, $C_f$ is $1$. However, $\rho$ in practice is difficult to determine. Therefore, we tune $\gamma$ in the range \{300, 500, 1000\} in the experiment. Additionally, we conduct $\gamma=$1e7 case for our ablation study. Notably, STACO2 reduces to STACO1 if $\gamma$ equals an infinitely large number. The results are presented in Figure \ref{fig:nonlinear_tr_0.5_gamma}. We observe that an appropriate value of $\gamma$ can yield better training results, verifying Theorem $\ref{thm:iteration_complexity_nvx}$ and demonstrating the importance of the epoch decay parameter $\gamma$ for primal-dual algorithms in deep learning. 

    {\textbf{Effect of Surrogate Loss $\ell$.} We investigate how the choice of surrogate loss function $\ell$ influences the final experimental results. Specifically, we consider three common losses: square hinge loss, square loss, and hinge loss, and evaluate their performance across various datasets. The results show that our algorithm STACO performs consistently well and remains stable across different surrogate losses, indicating that the choice of $\ell$ has limited impact on the final performance.

    \begin{table}[ht]
    \centering
    \caption{Comparison of performance metrics using different nonsmooth losses across datasets. Each entry is reported as mean(std).}
    \scalebox{0.66}{
    \begin{tabular}{l|ccc|ccc|ccc}
    \toprule
    \multirow{2}{*}{\textbf{Methods}} & \multicolumn{3}{c|}{\textbf{HIGGS}} & \multicolumn{3}{c|}{\textbf{SUSY}} & \multicolumn{3}{c}{\textbf{ijcnn1}} \\
    \cline{2-10}
    & hinge square & square & hinge & hinge square & square & hinge & hinge square & square & hinge \\
    \midrule
    CE     & 0.354(0.002) & 0.376(0.003) & 0.341(0.004) & 0.612(0.004) & 0.590(0.004) & 0.639(0.005) & 0.581(0.003) & 0.560(0.004) & 0.604(0.003) \\
    AUCM   & 0.435(0.002) & 0.462(0.003) & 0.411(0.004) & 0.726(0.003) & 0.748(0.004) & 0.699(0.004) & 0.728(0.004) & 0.752(0.003) & 0.701(0.005) \\
    SOTAs  & 0.441(0.003) & 0.467(0.004) & 0.415(0.004) & 0.746(0.003) & 0.773(0.003) & 0.719(0.002) & 0.813(0.002) & 0.840(0.004) & 0.787(0.004) \\
    PAUCI  & 0.474(0.003) & 0.500(0.004) & 0.451(0.003) & 0.749(0.004) & 0.724(0.003) & \textbf{0.777(0.004)} & 0.830(0.002) & 0.857(0.003) & 0.808(0.003) \\
    SONX   & 0.447(0.003) & 0.472(0.003) & 0.420(0.004) & 0.748(0.003) & 0.773(0.003) & 0.723(0.003) & 0.810(0.002) & 0.785(0.003) & \textbf{0.836(0.004)} \\
    STACO  & \textbf{0.484(0.004)} & \textbf{0.511(0.003)} & \textbf{0.458(0.004)} & \textbf{0.752(0.000)} & \textbf{0.779(0.003)} & 0.725(0.004) & \textbf{0.839(0.024)} & \textbf{0.866(0.004)} & 0.812(0.004) \\
    \bottomrule
    \end{tabular}}
    \label{tab:nonsmooth_losses}
    \end{table}}

    {
    \subsection{Training Efficiency}
        To demonstrate the training efficiency of our algorithm, we compare the per-iteration runtime of STACO, PAUCI, SONX, and SOTA across four benchmark datasets, as shown in Table~\ref{tab:time}. STACO consistently achieves the lowest runtime per iteration across all datasets. Notably, it surpasses the second-best method, SONX, by a substantial margin, particularly on larger datasets such as molmuv and moltox21. These results highlight the superior computational efficiency of STACO, making it a compelling choice for large-scale or time-sensitive applications.

    \begin{table}[]
    \centering
    \caption{Training time per iteration (in seconds) on different datasets.}
    \begin{tabular}{c|cccc}
    \hline
    \textbf{Methods} & \textbf{molmuv} & \textbf{moltox21} & \textbf{nodulemnist3d} & \textbf{adrenalmnist3d} \\ \hline
    SOTA                                  & 14.80           & 8.01              & 2.02                   & 2.23                    \\
    SONX                                  & 9.78            & 4.54              & 1.62                   & 1.76                    \\
    PAUCI                                 & 12.54           & 5.72              & 2.51                   & 2.74                    \\
    STACO                                 & \textbf{8.72}   & \textbf{3.96}     & \textbf{1.40}          & \textbf{1.48}           \\ \hline
    \end{tabular}
    \label{tab:time}
    \end{table}
    }

\section{Conclusion}
    In this paper, we proposed two novel stochastic primal-dual double block-coordinate algorithms for optimizing two-way partial AUC (TPAUC), effectively addressing imbalanced data classification. By leveraging stochastic updates for both primal and dual variables, our methods achieve improved convergence rates in both convex and non-convex settings. Empirical results demonstrate faster convergence and superior generalization across benchmark datasets, establishing a new state-of-the-art in TPAUC optimization for real-world applications.

\section*{Acknowledgments}
We are grateful to the reviewers' comments. LZ, BW, TY were partially supported by NSF grants \#2306572 and \#2147253. 

\bibliography{main}

\begin{thebibliography}{38}
\providecommand{\natexlab}[1]{#1}
\providecommand{\url}[1]{\texttt{#1}}
\expandafter\ifx\csname urlstyle\endcsname\relax
  \providecommand{\doi}[1]{doi: #1}\else
  \providecommand{\doi}{doi: \begingroup \urlstyle{rm}\Url}\fi

\bibitem[Alacaoglu et~al.(2022)Alacaoglu, Cevher, and Wright]{alacaoglu2022complexity}
Ahmet Alacaoglu, Volkan Cevher, and Stephen~J Wright.
\newblock On the complexity of a practical primal-dual coordinate method.
\newblock \emph{arXiv preprint arXiv:2201.07684}, 2022.

\bibitem[Chang \& Lin(2011)Chang and Lin]{chang2011libsvm}
C.-C. Chang and C.-J. Lin.
\newblock Libsvm: a library for support vector machines.
\newblock \emph{TIST}, 2\penalty0 (3):\penalty0 27, 2011.

\bibitem[Davis \& Drusvyatskiy(2018)Davis and Drusvyatskiy]{davis2018stochastic}
Damek Davis and Dmitriy Drusvyatskiy.
\newblock Stochastic subgradient method converges at the rate $o(k^{-1/4})$ on weakly convex functions.
\newblock \emph{arXiv preprint arXiv:1802.02988}, 2018.

\bibitem[Davis \& Grimmer(2019)Davis and Grimmer]{davis2019proximally}
Damek Davis and Benjamin Grimmer.
\newblock Proximally guided stochastic subgradient method for nonsmooth, nonconvex problems.
\newblock \emph{SIAM Journal on Optimization}, 29\penalty0 (3):\penalty0 1908--1930, 2019.

\bibitem[Hamedani et~al.(2023)Hamedani, Jalilzadeh, and Aybat]{hamedani2023randomized}
Erfan~Yazdandoost Hamedani, Afrooz Jalilzadeh, and Necdet~S Aybat.
\newblock Randomized primal-dual methods with adaptive step sizes.
\newblock In \emph{International Conference on Artificial Intelligence and Statistics}, pp.\  11185--11212. PMLR, 2023.

\bibitem[Hanley \& McNeil(1982)Hanley and McNeil]{hanley1982meaning}
James~A Hanley and Barbara~J McNeil.
\newblock The meaning and use of the area under a receiver operating characteristic (roc) curve.
\newblock \emph{Radiology}, 143\penalty0 (1):\penalty0 29--36, 1982.

\bibitem[He et~al.(2016)He, Zhang, Ren, and Sun]{he2016deep}
Kaiming He, Xiangyu Zhang, Shaoqing Ren, and Jian Sun.
\newblock Deep residual learning for image recognition.
\newblock In \emph{Proceedings of the IEEE conference on computer vision and pattern recognition}, pp.\  770--778, 2016.

\bibitem[Herbrich et~al.(1999)Herbrich, Graepel, and Obermayer]{Herbrich:1999gl}
R.~Herbrich, T.~Graepel, and K.: Obermayer.
\newblock Support vector learning for ordinal re- gression.
\newblock In \emph{International Conference on Neural Networks}, 1999.

\bibitem[Hu et~al.(2023)Hu, Zhu, and Yang]{hu2023non}
Quanqi Hu, Dixian Zhu, and Tianbao Yang.
\newblock Non-smooth weakly-convex finite-sum coupled compositional optimization.
\newblock \emph{arXiv preprint arXiv:2310.03234}, 2023.

\bibitem[Hu et~al.(2020)Hu, Fey, Zitnik, Dong, Ren, Liu, Catasta, and Leskovec]{hu2020open}
Weihua Hu, Matthias Fey, Marinka Zitnik, Yuxiao Dong, Hongyu Ren, Bowen Liu, Michele Catasta, and Jure Leskovec.
\newblock Open graph benchmark: Datasets for machine learning on graphs.
\newblock \emph{Advances in neural information processing systems}, 33:\penalty0 22118--22133, 2020.

\bibitem[Jalilzadeh et~al.(2019)Jalilzadeh, Hamedani, and Aybat]{jalilzadeh2019doubly}
Afrooz Jalilzadeh, Erfan~Yazdandoost Hamedani, and Necdet~S Aybat.
\newblock A doubly-randomized block-coordinate primal-dual method for large-scale saddle point problems.
\newblock \emph{arXiv preprint arXiv:1907.03886}, 2019.

\bibitem[Juditsky et~al.(2011)Juditsky, Nemirovski, and Tauvel]{juditsky2011solving}
Anatoli Juditsky, Arkadi Nemirovski, and Claire Tauvel.
\newblock Solving variational inequalities with stochastic mirror-prox algorithm.
\newblock \emph{Stochastic Systems}, 1\penalty0 (1):\penalty0 17--58, 2011.

\bibitem[Kar et~al.(2014)Kar, Narasimhan, and Jain]{kar2014online}
Purushottam Kar, Harikrishna Narasimhan, and Prateek Jain.
\newblock Online and stochastic gradient methods for non-decomposable loss functions.
\newblock \emph{Advances in Neural Information Processing Systems}, 27, 2014.

\bibitem[Narasimhan \& Agarwal(2013)Narasimhan and Agarwal]{narasimhan2013svmpauctight}
Harikrishna Narasimhan and Shivani Agarwal.
\newblock Svmpauctight: a new support vector method for optimizing partial auc based on a tight convex upper bound.
\newblock In \emph{Proceedings of the 19th ACM SIGKDD international conference on Knowledge discovery and data mining}, pp.\  167--175, 2013.

\bibitem[Nesterov et~al.(2018)]{nesterov2018lectures}
Yurii Nesterov et~al.
\newblock \emph{Lectures on convex optimization}, volume 137.
\newblock Springer, 2018.

\bibitem[Ogryczak \& Tamir(2003)Ogryczak and Tamir]{Ogryczak:2003dl}
W.~Ogryczak and A.~Tamir.
\newblock Minimizing the sum of the k largest functions in linear time.
\newblock \emph{Information Processing Letters}, 85\penalty0 (3):\penalty0 117--122, 2003.

\bibitem[Qi et~al.(2021)Qi, Luo, Xu, Ji, and Yang]{qi2021stochastic}
Qi~Qi, Youzhi Luo, Zhao Xu, Shuiwang Ji, and Tianbao Yang.
\newblock Stochastic optimization of areas under precision-recall curves with provable convergence.
\newblock \emph{Advances in neural information processing systems}, 34:\penalty0 1752--1765, 2021.

\bibitem[Rafique et~al.(2022)Rafique, Liu, Lin, and Yang]{rafique2022weakly}
Hassan Rafique, Mingrui Liu, Qihang Lin, and Tianbao Yang.
\newblock Weakly-convex--concave min--max optimization: provable algorithms and applications in machine learning.
\newblock \emph{Optimization Methods and Software}, 37\penalty0 (3):\penalty0 1087--1121, 2022.

\bibitem[Shao et~al.(2022)Shao, Xu, Yang, Bao, and Huang]{shao2022asymptotically}
Huiyang Shao, Qianqian Xu, Zhiyong Yang, Shilong Bao, and Qingming Huang.
\newblock Asymptotically unbiased instance-wise regularized partial auc optimization: Theory and algorithm.
\newblock \emph{Advances in Neural Information Processing Systems}, 35:\penalty0 38667--38679, 2022.

\bibitem[Shao et~al.(2023)Shao, Xu, Yang, Wen, Peifeng, and Huang]{shao2023weighted}
Huiyang Shao, Qianqian Xu, Zhiyong Yang, Peisong Wen, Gao Peifeng, and Qingming Huang.
\newblock Weighted roc curve in cost space: Extending auc to cost-sensitive learning.
\newblock \emph{Advances in Neural Information Processing Systems}, 36:\penalty0 17357--17368, 2023.

\bibitem[Wang \& Yang()Wang and Yang]{wangnear}
Bokun Wang and Tianbao Yang.
\newblock A near-optimal single-loop stochastic algorithm for convex finite-sum coupled compositional optimization.
\newblock In \emph{Forty-second International Conference on Machine Learning}.

\bibitem[Wang et~al.(2017{\natexlab{a}})Wang, Fang, and Liu]{wang2017stochastic}
Mengdi Wang, Ethan~X Fang, and Han Liu.
\newblock Stochastic compositional gradient descent: algorithms for minimizing compositions of expected-value functions.
\newblock \emph{Mathematical Programming}, 161\penalty0 (1-2):\penalty0 419--449, 2017{\natexlab{a}}.

\bibitem[Wang et~al.(2017{\natexlab{b}})Wang, Liu, and Fang]{wang2017accelerating}
Mengdi Wang, Ji~Liu, and Ethan~X Fang.
\newblock Accelerating stochastic composition optimization.
\newblock \emph{Journal of Machine Learning Research}, 18:\penalty0 1--23, 2017{\natexlab{b}}.

\bibitem[Xie et~al.(2024)Xie, Liu, He, Li, and Zhou]{xie2024weakly}
Zheng Xie, Yu~Liu, Hao-Yuan He, Ming Li, and Zhi-Hua Zhou.
\newblock Weakly supervised auc optimization: a unified partial auc approach.
\newblock \emph{IEEE Transactions on Pattern Analysis and Machine Intelligence}, 2024.

\bibitem[Xu et~al.(2018)Xu, Hu, Leskovec, and Jegelka]{xu2018powerful}
Keyulu Xu, Weihua Hu, Jure Leskovec, and Stefanie Jegelka.
\newblock How powerful are graph neural networks?
\newblock \emph{arXiv preprint arXiv:1810.00826}, 2018.

\bibitem[Yang et~al.(2019)Yang, Lu, Lyu, and Hu]{yang2019two}
Hanfang Yang, Kun Lu, Xiang Lyu, and Feifang Hu.
\newblock Two-way partial auc and its properties.
\newblock \emph{Statistical methods in medical research}, 28\penalty0 (1):\penalty0 184--195, 2019.

\bibitem[Yang et~al.(2023{\natexlab{a}})Yang, Shi, Wei, Liu, Zhao, Ke, Pfister, and Ni]{medmnistv2}
Jiancheng Yang, Rui Shi, Donglai Wei, Zequan Liu, Lin Zhao, Bilian Ke, Hanspeter Pfister, and Bingbing Ni.
\newblock Medmnist v2-a large-scale lightweight benchmark for 2d and 3d biomedical image classification.
\newblock \emph{Scientific Data}, 10\penalty0 (1):\penalty0 41, 2023{\natexlab{a}}.

\bibitem[Yang \& Ying(2022)Yang and Ying]{yang2022auc}
Tianbao Yang and Yiming Ying.
\newblock Auc maximization in the era of big data and ai: A survey.
\newblock \emph{ACM Computing Surveys}, 55\penalty0 (8):\penalty0 1--37, 2022.

\bibitem[Yang et~al.(2021)Yang, Xu, Bao, He, Cao, and Huang]{yang2021all}
Zhiyong Yang, Qianqian Xu, Shilong Bao, Yuan He, Xiaochun Cao, and Qingming Huang.
\newblock When all we need is a piece of the pie: A generic framework for optimizing two-way partial auc.
\newblock In \emph{International Conference on Machine Learning}, pp.\  11820--11829. PMLR, 2021.

\bibitem[Yang et~al.(2022)Yang, Xu, Bao, He, Cao, and Huang]{yang2022optimizing}
Zhiyong Yang, Qianqian Xu, Shilong Bao, Yuan He, Xiaochun Cao, and Qingming Huang.
\newblock Optimizing two-way partial auc with an end-to-end framework.
\newblock \emph{IEEE Transactions on Pattern Analysis and Machine Intelligence}, 45\penalty0 (8):\penalty0 10228--10246, 2022.

\bibitem[Yang et~al.(2023{\natexlab{b}})Yang, Xu, Bao, Wen, He, Cao, and Huang]{yang2023auc}
Zhiyong Yang, Qianqian Xu, Shilong Bao, Peisong Wen, Yuan He, Xiaochun Cao, and Qingming Huang.
\newblock Auc-oriented domain adaptation: From theory to algorithm.
\newblock \emph{IEEE Transactions on Pattern Analysis and Machine Intelligence}, 45\penalty0 (12):\penalty0 14161--14174, 2023{\natexlab{b}}.

\bibitem[Yuan et~al.(2021{\natexlab{a}})Yuan, Guo, Chawla, and Yang]{yuan2021compositional}
Zhuoning Yuan, Zhishuai Guo, Nitesh Chawla, and Tianbao Yang.
\newblock Compositional training for end-to-end deep auc maximization.
\newblock In \emph{International Conference on Learning Representations}, 2021{\natexlab{a}}.

\bibitem[Yuan et~al.(2021{\natexlab{b}})Yuan, Guo, Xu, Ying, and Yang]{yuan2021federated}
Zhuoning Yuan, Zhishuai Guo, Yi~Xu, Yiming Ying, and Tianbao Yang.
\newblock Federated deep auc maximization for hetergeneous data with a constant communication complexity.
\newblock In \emph{International Conference on Machine Learning}, pp.\  12219--12229. PMLR, 2021{\natexlab{b}}.

\bibitem[Zhang \& Xiao(2022)Zhang and Xiao]{zhang2022stochastic}
Junyu Zhang and Lin Xiao.
\newblock Stochastic variance-reduced prox-linear algorithms for nonconvex composite optimization.
\newblock \emph{Mathematical Programming}, pp.\  1--43, 2022.

\bibitem[Zhang et~al.(2023)Zhang, Zhang, Yang, Souvenir, and Gao]{zhang2023federated}
Xinwen Zhang, Yihan Zhang, Tianbao Yang, Richard Souvenir, and Hongchang Gao.
\newblock Federated compositional deep auc maximization.
\newblock \emph{Advances in Neural Information Processing Systems}, 36:\penalty0 9648--9660, 2023.

\bibitem[Zhang \& Xiao(2015)Zhang and Xiao]{zhang2015stochastic}
Yuchen Zhang and Lin Xiao.
\newblock Stochastic primal-dual coordinate method for regularized empirical risk minimization.
\newblock In \emph{ICML}, pp.\  353--361, 2015.

\bibitem[Zhang \& Lan(2020)Zhang and Lan]{zhang2020optimal}
Zhe Zhang and Guanghui Lan.
\newblock Optimal algorithms for convex nested stochastic composite optimization.
\newblock \emph{arXiv preprint arXiv:2011.10076}, 2020.

\bibitem[Zhu et~al.(2022)Zhu, Li, Wang, Wu, and Yang]{zhu2022auc}
Dixian Zhu, Gang Li, Bokun Wang, Xiaodong Wu, and Tianbao Yang.
\newblock When auc meets dro: Optimizing partial auc for deep learning with non-convex convergence guarantee.
\newblock In \emph{International Conference on Machine Learning}, pp.\  27548--27573. PMLR, 2022.

\end{thebibliography}
\bibliographystyle{tmlr}

\appendix
\newpage
\section{Vanilla Algorithm}
\label{app:vanilla}

    \begin{algorithm}[H]
    \caption{Simplified STACO1}
    \label{alg:single}
        \begin{algorithmic}[1] 
            \State Initialize $\u_0\in\U$, $\s_0\in\S$, $\y_0 \in \Y$
            \For{$t=0,1,\dotsc,T-1$}
                \State Sample a batch $\S_t\subset \{1,\dotsc,n\}$, $|\S_t| = S$ 
                \For{each $i\in\S_t$} 
                    \State Sample independent size-$B$ mini-batches $\B_t^{(i)}, \tilde{\B}_{t}^{(i)}$ from $\mathbb{P}_i$
                    \State Compute $\hat{g}_t^{(i)}(\B_t^{(i)}) = g_i(\u_t,\s_t^{(i)};\B_t^{(i)})$
                    \State Compute $\hat{G}_{t,1}^{(i)}(\tilde{\B}_t^{(i)})  \in \partial_{\u} g_i (\u_t,\s_t^{(i)};\tilde{\B}_t^{(i)})$, $\hat{G}_{t,2}^{(i)}(\tilde{\B}_t^{(i)})  \in \partial_{\s^{(i)}} g_i (\u_t,\s_t^{(i)};\tilde{\B}_t^{(i)})$
                    \State $\y_{t+1}^{(i)} = \argmax_{\y^{(i)}\in\Y_i}\left\{\y^{(i)} \hat{g}_t^{(i)}({\B}_t^{(i)})  - f_i^*(\y^{(i)}) - \frac{1}{2\alpha} \left(\y^{(i)} - \y_t^{(i)}\right)^2\right\}$
                    \State $\s_{t+1}^{(i)} = \s_t^{(i)} - \beta\frac{1}{S}\sum_{i\in\S_t} \y_{t+1}^{(i)} \hat{G}_{t,2}^{(i)}(\tilde{\B}_t^{(i)})$
                \EndFor
                \State For each $i\notin \S_t$, $\y_{t+1}^{(i)} = \y_t^{(i)},\s_{t+1}^{(i)} = \s_t^{(i)}$
                \State $\u_{t+1} = \u_t - \eta\frac{1}{S}\sum_{i\in\S_t} \y_{t+1}^{(i)} \hat{G}_{t,1}^{(i)}(\tilde{\B}_t^{(i)})$
            \EndFor
            \State $\bar{\u}=\frac{1}{T}\sum_{t=0}^{T-1}\u_{t+1},\bar{\s}=\frac{1}{T}\sum_{t=0}^{T-1}\s_{t+1}$
            \State Return $\bar{\u},\bs$
        \end{algorithmic}
    \end{algorithm}
    
    \begin{algorithm}[H]
        \caption{Simplified STACO2}
        \label{alg:double}
        \begin{algorithmic}[1] 
            \State Initialize $\u_0\in\U,\s_0\in\S$
            \For{$t=0,1,\dotsc,T-1$}
                \State Initialize $\y_{t,0} \in \Y$
                \State Set $\u_{t,0}=\u_{t}, \s_{t,0}=\s_{t}$
                \For{$k=0,1,\dotsc,K_t-1$}
                    \State Sample a batch $\S_{t,k}\subset \{1,\dotsc,n\}$, $|\S_{t,k}| = S$ 
                    \For{each $i\in\S_{t,k}$} 
                        \State Sample independent size-$B$ mini-batches $\B_{t,k}^{(i)}$, $\tilde{\B}_{t,k}^{(i)}$ from $\mathbb{P}_i$
                        \State Compute $\hat{g}_{t,k}^{(i)}(\B_{t,k}^{(i)}) = g_i(\u_{t,k},\s_{t,k}^{(i)};\B_{t,k}^{(i)})$
                        \State Compute $\hat{G}_{t,k,1}^{(i)}(\tilde{\B}_{t,k}^{(i)})  \in \partial_{\u} g_i (\u_{t,k},\s_{t,k}^{(i)};\tilde{\B}_{t,k}^{(i)}), \hat{G}_{t,k,2}^{(i)}(\tilde{\B}_{t,k}^{(i)})  \in \partial_{\s^{(i)}} g_i (\u_{t,k},\s_{t,k}^{(i)};\tilde{\B}_{t,k}^{(i)})$
                        \State $\y_{t, k+1}^{(i)} = \argmax_{\y^{(i)}\in\Y_i}\left\{\y^{(i)} {\hat{g}_{t,k}^{(i)}({\B}_{t,k}^{(i)})} - f_i^*(\y^{(i)}) - \frac{1}{2\alpha_{t}} \left(\y^{(i)} - \y_{t,k}^{(i)}\right)^2\right\}$
                        \State ${\s_{t,k+1}^{(i)} = \argmin_{\s^{(i)}\in\S_i}\left\{\inner{\s^{(i)}}{\frac{1}{S}\sum_{i\in\S_{t,k}} \y_{t,k+1}^{(i)} \hat{G}_{t,k,2}^{(i)}(\tilde{\B}_{t,k}^{(i)})+\frac{1}{\gamma}(\s_{t,k}^{(i)}-\s_{t,0}^{(i)})} + \frac{1}{2\beta_t} \prt{\s^{(i)} - \s_{t,k}^{(i)}}^2\right\}}$
                    \EndFor
                    \State For each $i\notin \S_{t,k}$, $\y_{t,k+1}^{(i)} = \y_{t,k}^{(i)},\s_{t,k+1}^{(i)} = \s_{t,k}^{(i)}$  
                    \State ${\u_{t,k+1} = \argmin_{\u\in\U}\left\{\inner{\u}{\frac{1}{S}\sum_{i\in\S_{t,k}} \y_{t,k+1}^{(i)} \hat{G}_{t,k,1}^{(i)}(\tilde{\B}_{t,k}^{(i)})+\frac{1}{\gamma}(\u_{t,k}-\u_{t,0})} + \frac{1}{2\eta_t} \Norm{\u - \u_{t,k}}^2\right\}}$
                \EndFor
                \State Compute $\bar{\u}_{t} = \frac{1}{K_t}\sum_{k=0}^{K_t - 1} \u_{t,k+1},\bar{\s}_{t} = \frac{1}{K_t}\sum_{k=0}^{K_t - 1} \s_{t,k+1}$
                \State Set $\u_{t+1}=\bar{\u}_{t}, \s_{t+1}=\bar{\s}_{t}$
            \EndFor
            \State Return $\u_{T},\s_{T}$
    \end{algorithmic}
    \end{algorithm}
\clearpage 

    \begin{table}[H]
    \centering
    \caption{Datasets Statistics (for nodulemnist3d and adrenalmnist3d, we follow the given training, validation and testing split). The percentage in parenthesis represents the proportion of positive samples.}
    \begin{tabular}{ccc}
    \hline
    Dataset            & Train (Validation)           & Test            \\ \hline
    HIGGS              & 4157561 (0.5\%)               & 1039299 (0.5\%) \\
    SUSY               & 2181312 (0.5\%)               & 544489 (0.5\%)  \\
    ijcnn1             & 49990 (9.71\%)                & 91701 (9.5\%)   \\
    ogbg-moltox21 (t0) & 6556 (4.2\%)                  & 709 (4.5\%)     \\
    ogbg-molmuv (t1)   & 13025 (0.17\%)                & 1709 (0.35\%)   \\
    nodulemnist3d      & 1,158 (25.4\%) / 165 (25.4\%) & 310 (20.6\%)    \\
    adrenalmnist3d     & 1,188 (21.8\%) / 98 (22.4\%)  & 298 (23.1\%)    \\ \hline
    \end{tabular}
    \label{tb:dataset_sta}
    \end{table}

    \begin{figure}[H]
        \centering
        \subfloat[ijcnn1]
        {\includegraphics[width=0.28\textwidth]{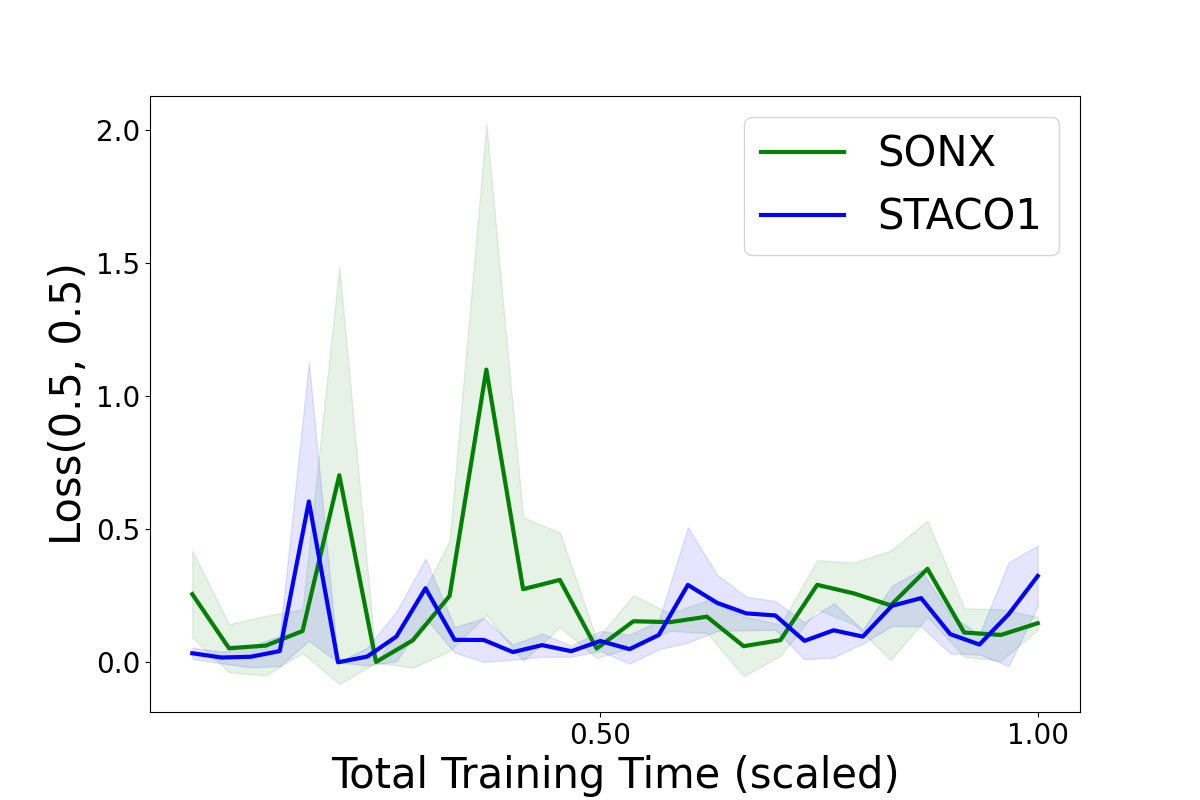}}
        \subfloat[SUSY]
        {\includegraphics[width=0.28\textwidth]{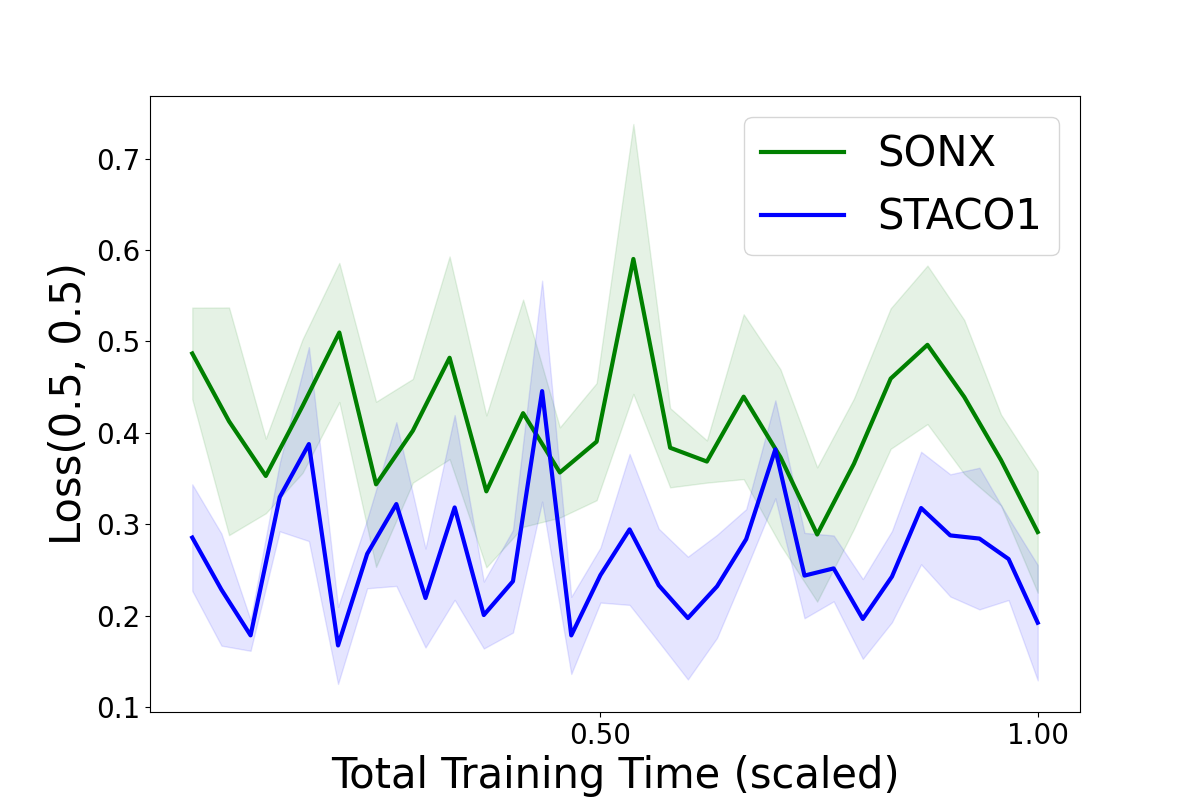}}
        \subfloat[higgs]
        {\includegraphics[width=0.28\textwidth]{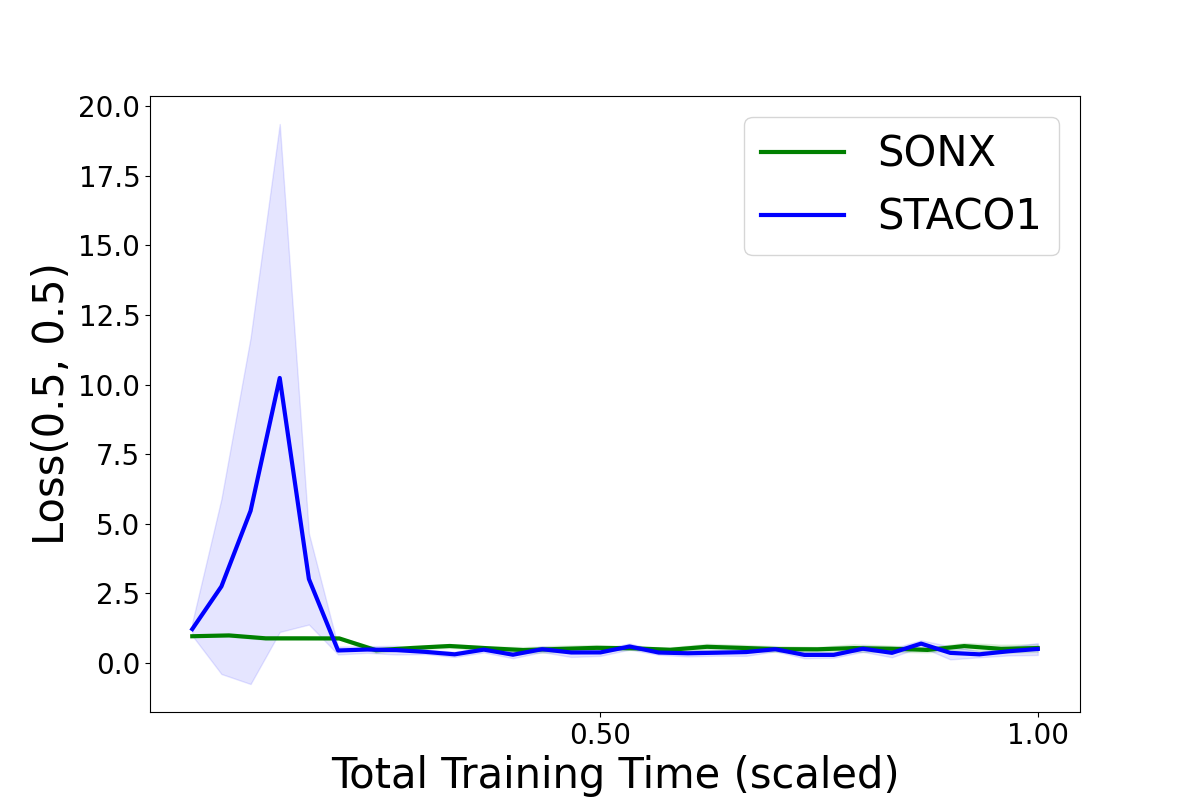}}\\
        \subfloat[ijcnn1]
        {\includegraphics[width=0.28\textwidth]{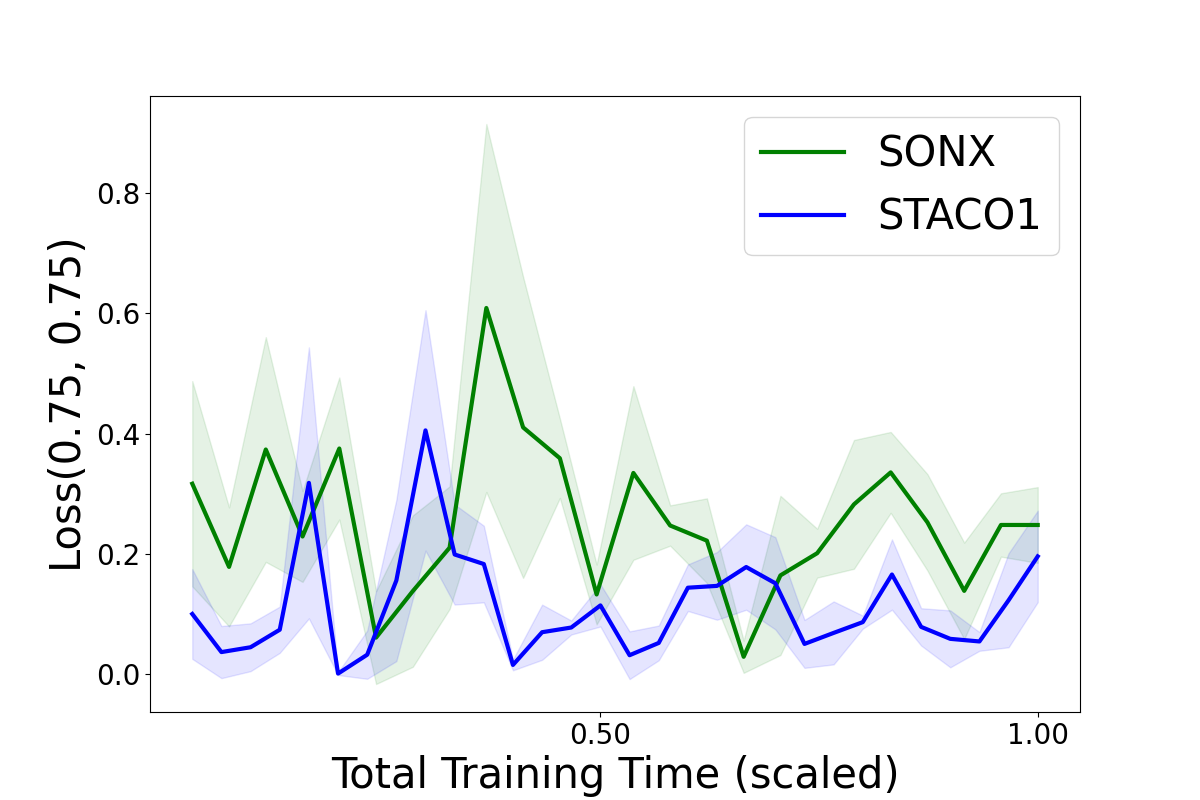}}
        \subfloat[SUSY]
        {\includegraphics[width=0.28\textwidth]{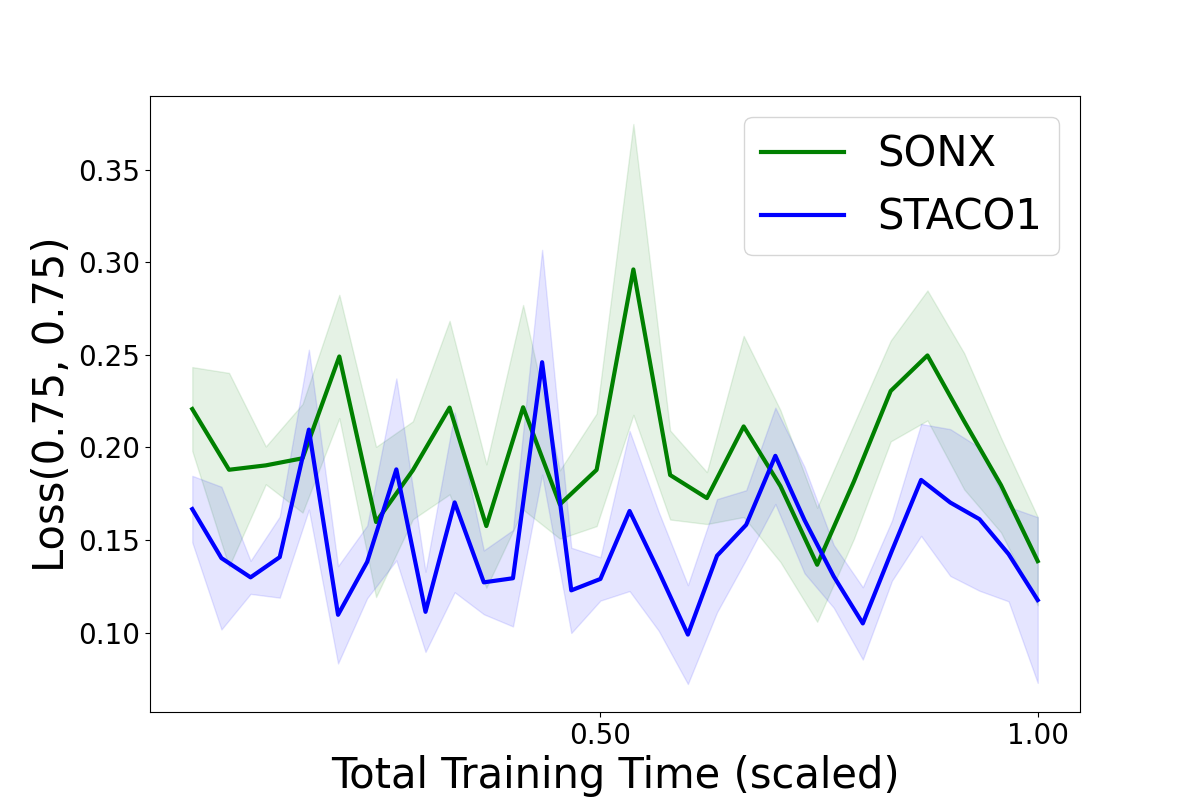}}
        \subfloat[higgs]
        {\includegraphics[width=0.28\textwidth]{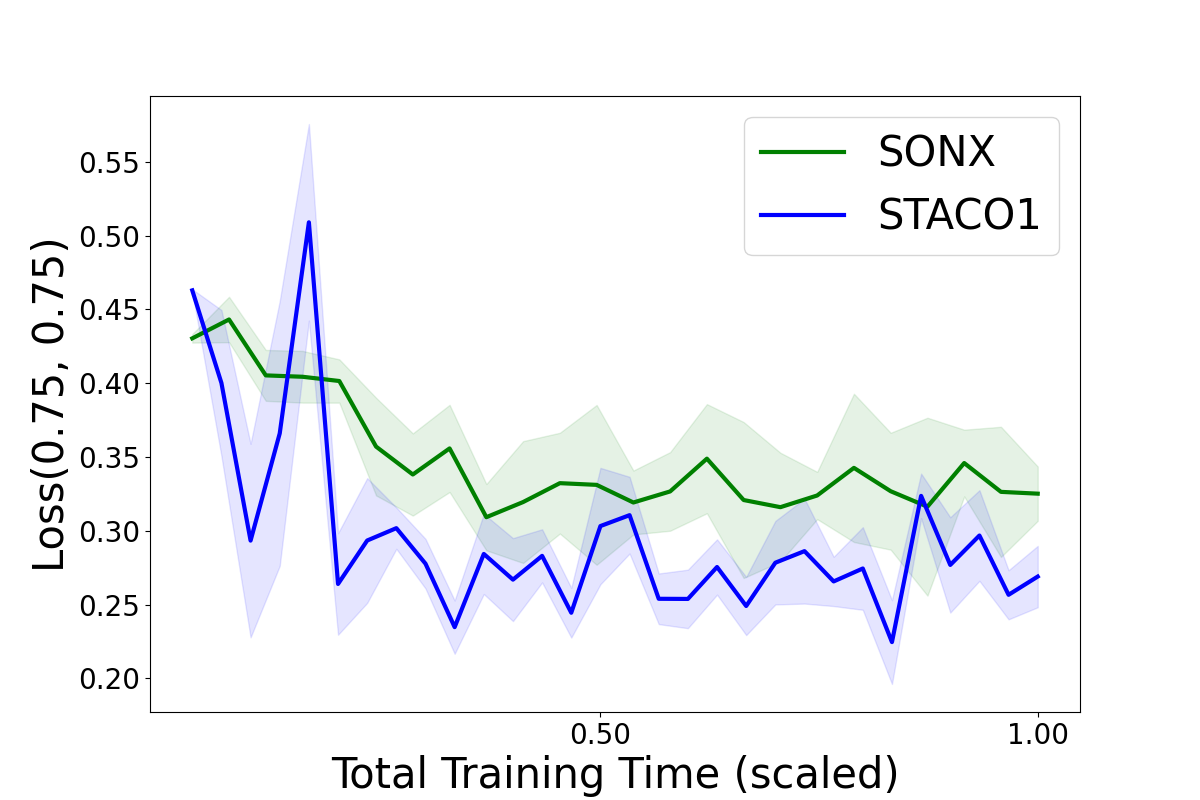}}
        \caption{Training loss Curves of STACO1 and SONX on three different datasets. The first row shows the Loss (0.5, 0.5) results, and the second row shows the Loss (0.75, 0.75) results.}
        \label{fig:linear_loss}

        \centering
        \subfloat[molmuv(t1)]
        {\includegraphics[width=0.245\textwidth]{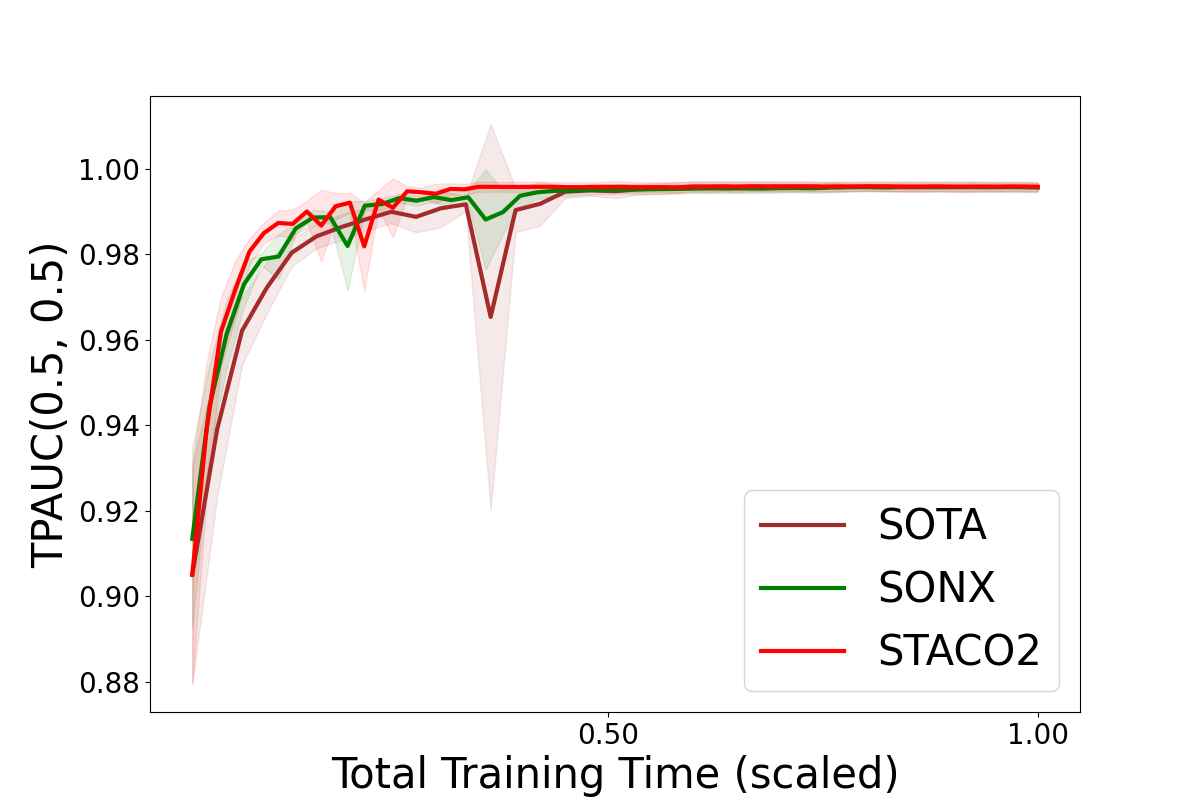}}
        \subfloat[moltox21(t0)]
        {\includegraphics[width=0.245\textwidth]{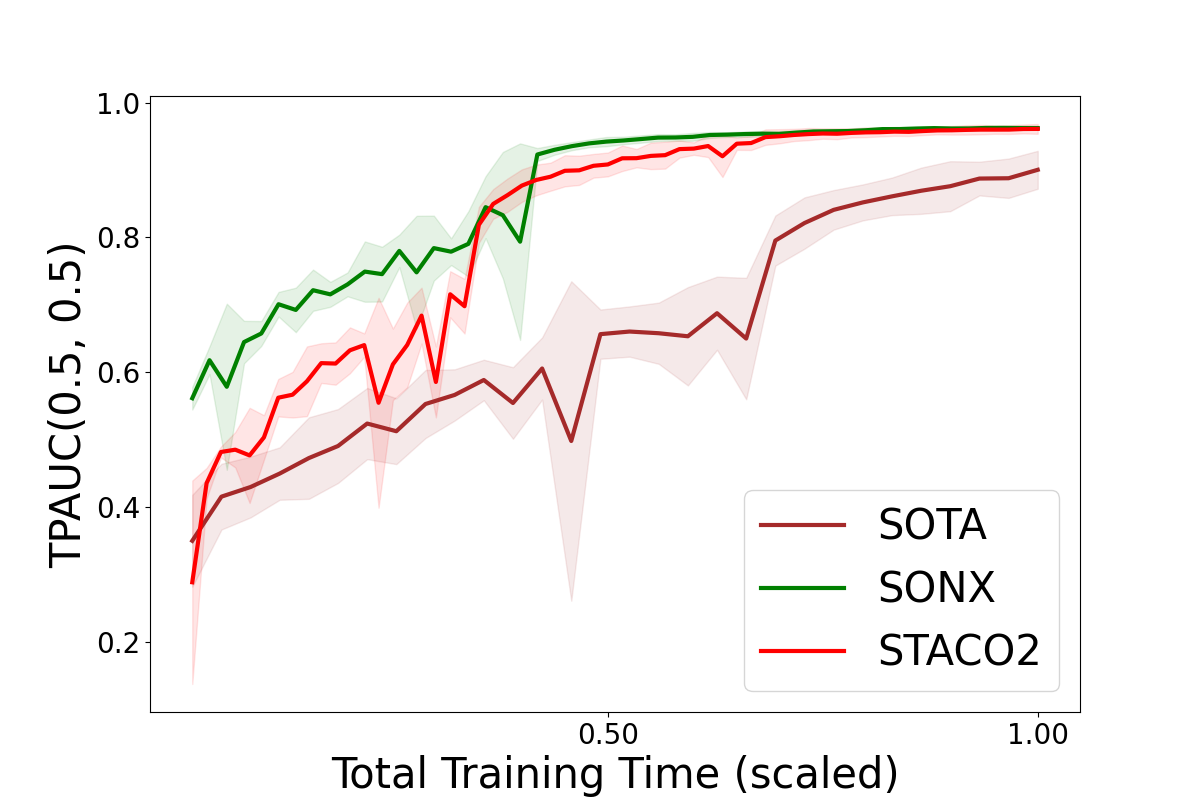}}
        \subfloat[nodulemnist3d]
        {\includegraphics[width=0.245\textwidth]{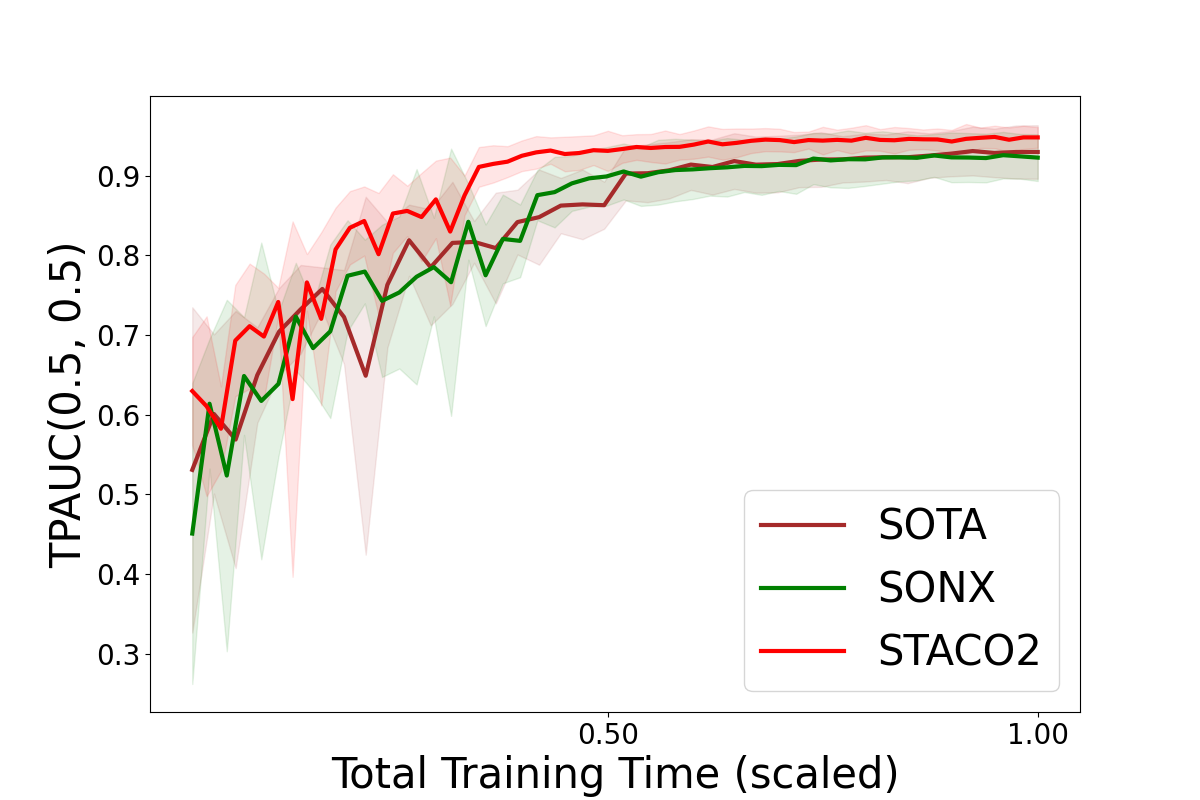}}
        \subfloat[adrenalmnist3d]
        {\includegraphics[width=0.245\textwidth]{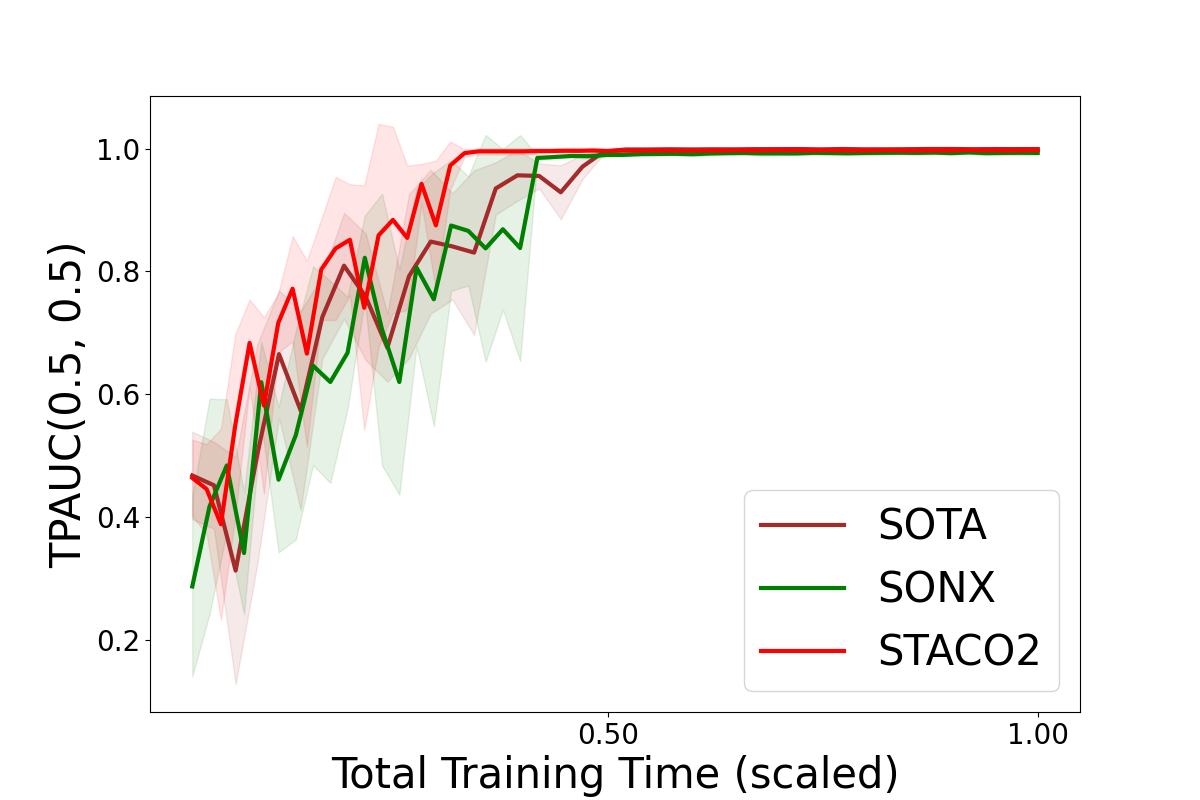}}\\
        \subfloat[molmuv(t1)]
        {\includegraphics[width=0.245\textwidth]{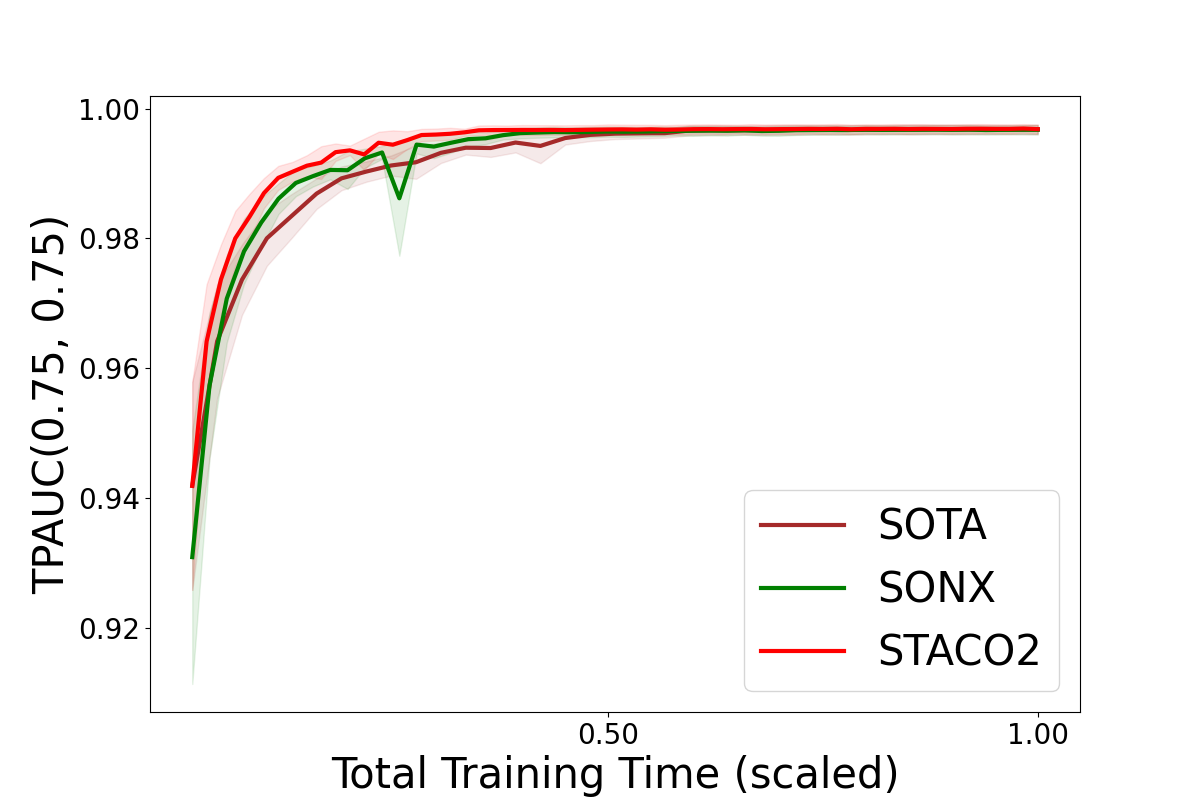}}
        \subfloat[moltox21(t0)]
        {\includegraphics[width=0.245\textwidth]{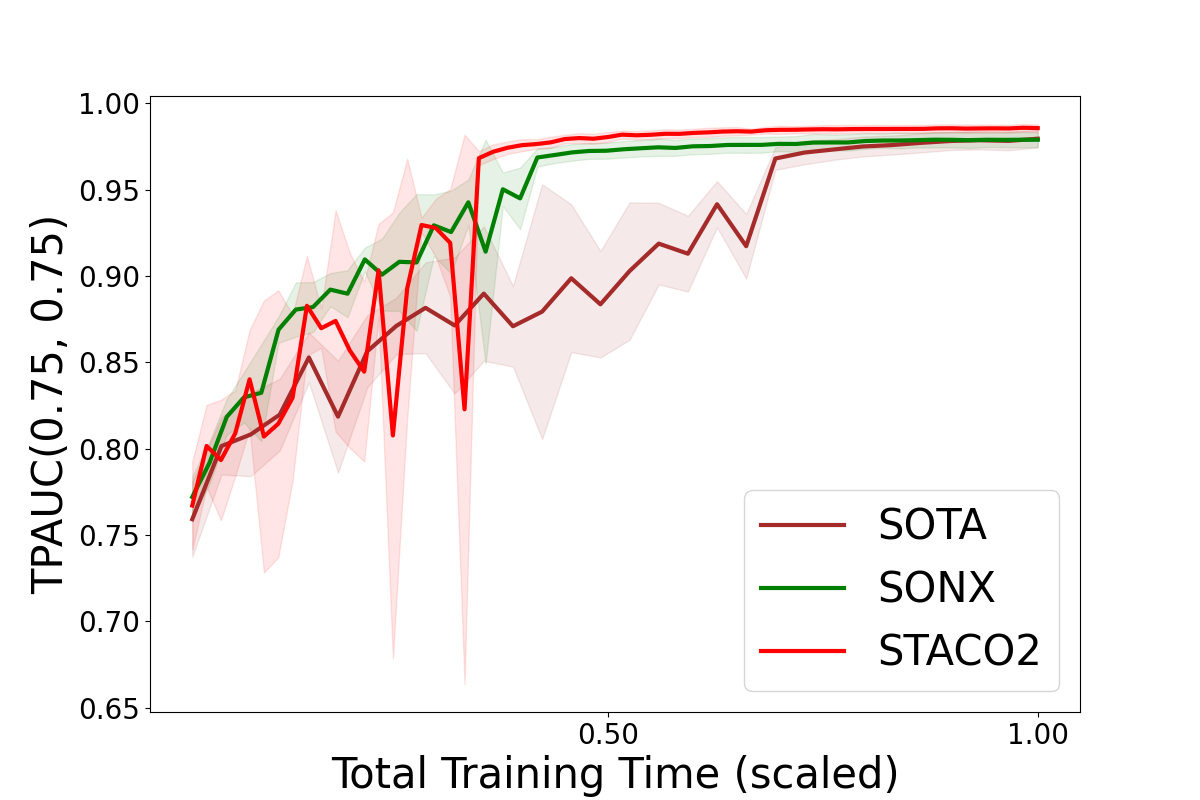}}
        \subfloat[nodulemnist3d]
        {\includegraphics[width=0.245\textwidth]{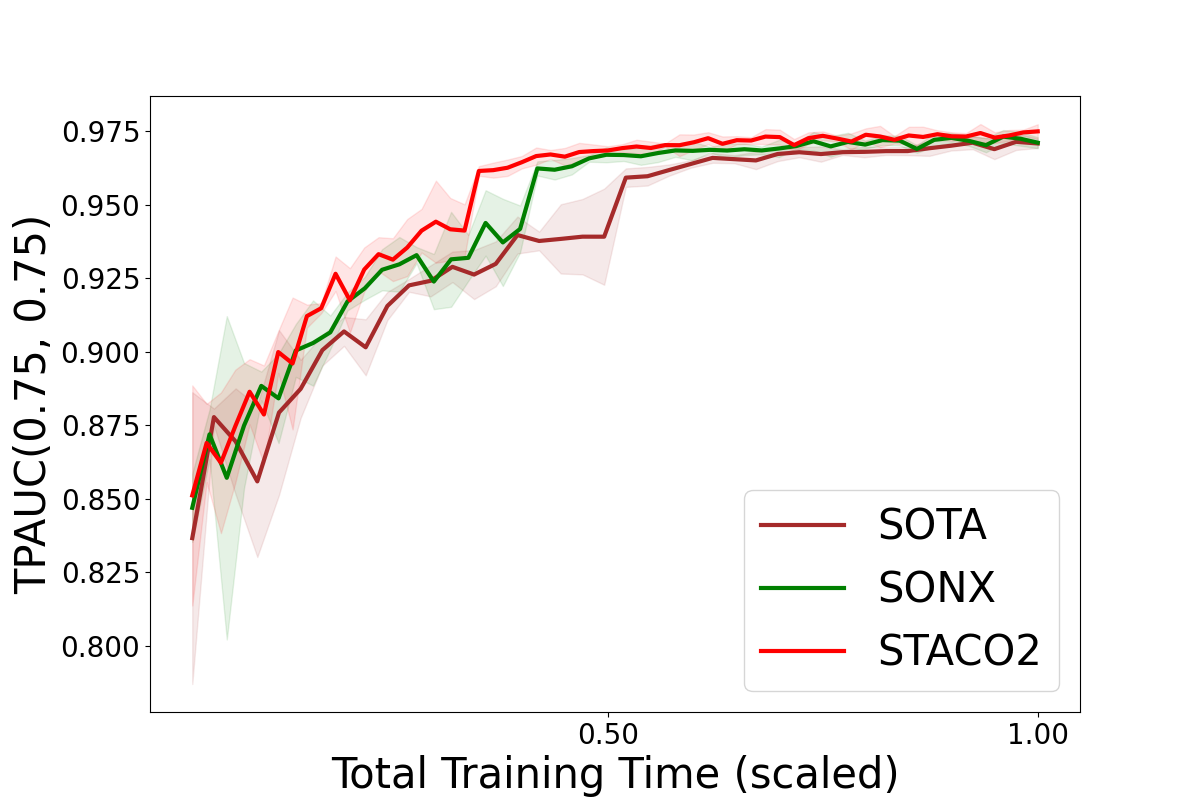}}
        \subfloat[adrenalmnist3d]
        {\includegraphics[width=0.245\textwidth]{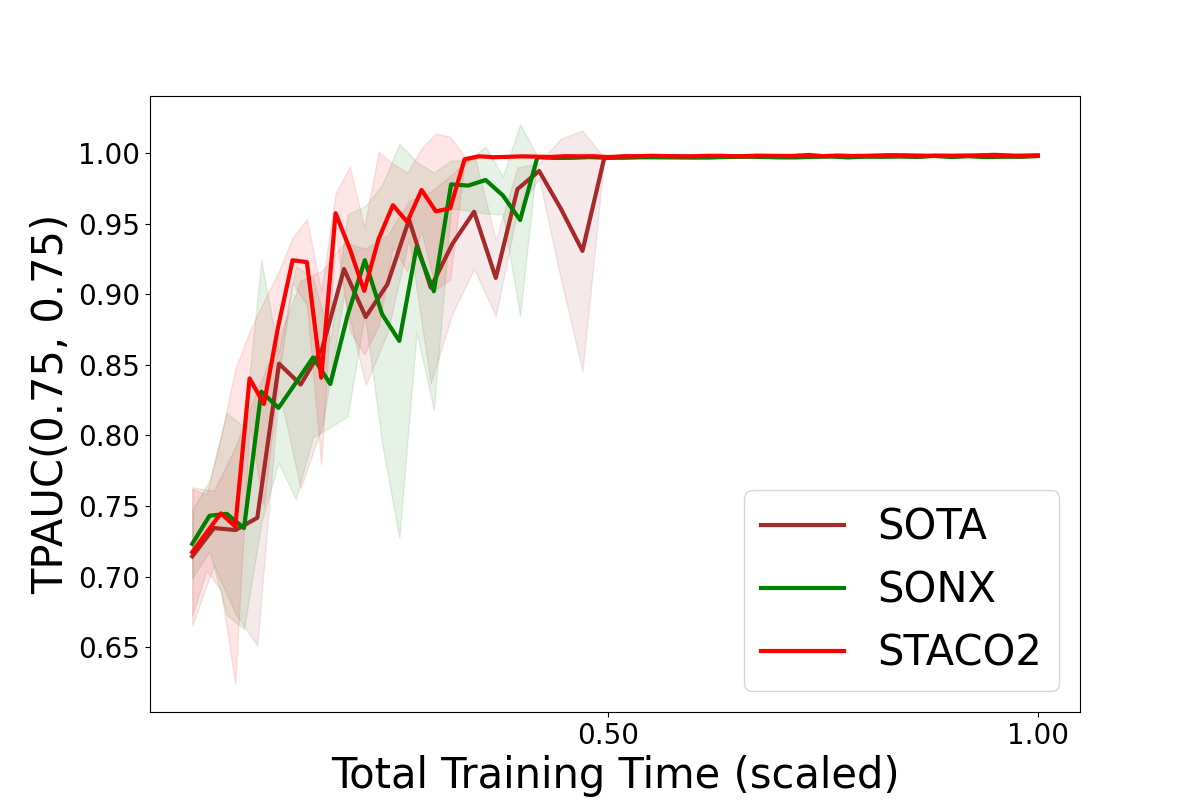}}
        \caption{Training TPAUC Curves of STACO2, SOTA, and SONX on four different datasets. The first row shows the TPAUC (0.5, 0.5) results, and the second row shows the TPAUC (0.75, 0.75) results.}
        \label{fig:nonlinear_tpauc}
    \end{figure}

\section{More Experiment Results}
\label{sec:add}

\subsection{Additional plots for training loss curves}
    Figure \ref{fig:linear_loss} presents the training loss curves of STACO1 and SONX across three different datasets under two evaluation settings, (0.5, 0.5) and (0.75, 0.75). The results indicate that STACO1 consistently achieves lower and more stable loss values compared to SONX across all datasets. Notably, the variance in training loss is lower for STACO1, suggesting improved stability during optimization. {The only exception is Figure 5c, which corresponds to optimizing the loss with $(0.5, 0.5)$ weights. We believe this is a limitation of the primal-dual algorithm, which involves two learning rates. In practical applications, improper tuning of these rates may lead to training instability.} Besides, it is important to note that the loss curve is less stable compared to deep learning experiments, primarily due to the absence of pretraining for the linear model.

\subsection{Additional plots for training TPAUC curves}
     Figure \ref{fig:nonlinear_tpauc} presents the training TPAUC curves for STACO2, SOTA, and SONX across the four datasets. In both TPAUC (0.5, 0.5) and TPAUC (0.75, 0.75) settings, STACO2 demonstrates competitive performance compared to SOTA and SONX, with better stability and faster convergence speed. Specifically, in some cases, STACO2 achieves superior results, particularly in later training stages, indicating its effectiveness in optimizing TPAUC objectives.

\section{Proof}

\subsection{Preliminary Lemmas}
    Throughout the proof, for a space $\X$, we define its diameter with respect to the measure $\psi(\cdot)=\frac{1}{2}\Norm{\cdot}^2$ as $D_{\X} \coloneqq \brk{\max_{\x\in\X} \psi(\x) - \min_{\x\in\X} \psi(\x)}^{1/2}$. Besides, $a\asymp b$ means that there exists $c,C >0$ such that $cb\leq a\leq Cb$. We first present a lemma here that will be useful in our later analysis.

    \begin{lem}[Lemma 4 in \citet{wangnear}]
    \label{lem:proximal_update}
        Suppose that the function $\phi: \X \rightarrow \R$ is on a convex, closed domain $\X$ and $\phi$ is $\mu$-convex with respect to Euclidean distance function $d(\x,\y) \coloneqq \frac{1}{2}\Norm{\x-\y}^2$ for any $\x,\x' \in \X$, i.e., $\phi(\x) \geq \phi(\x')+\inner{\phi'(\x')}{\x-\x'}+\mu d(\x,\x')$, $\forall \x,\x' \in\X$. For $\hat{\x} = \argmin_{\x\in\X} \{\phi(\x) + \eta d(\underline{\x}, \x)\}$, we obtain
        \begin{align}
            \phi(\hat{\x})-\phi(\x) \leq \eta d(\x,\underline{\x})-(\eta+\mu)d(\x,\hat{\x})-\eta d(\hat{\x},\underline{\x}),\quad \forall \x\in\X.
        \end{align}
    \end{lem}

\subsection{Convex Case}

    In this section, we present the proof of the convex case. We begin by defining virtual sequences for Algorithm $\ref{alg:single}$. {The virtual sequences $\by$ and $\bs$ are calculated with full coordinates, which is easier to bound in analyze. Thus, we also hope to bound the difference between true sequences and virtual sequences.}
    \begin{definition}[virtual sequence]
        In Algorithm $\ref{alg:single}$, a virtual sequence $\{\by_{t}\}$ is defined as follows: 
        \begin{align}
            \bar{\y}_{t+1}^{(i)} &= \argmax_{\y^{(i)}\in\Y_i}\left\{\y^{(i)} \hat{g}_t^{(i)}({\B}_{t}^{(i)})  - f_i^*(\y^{(i)}) - \frac{1}{2\alpha} \left(\y^{(i)} - \y_{t}^{(i)}\right)^2\right\} \quad i \in [n].
        \end{align}        
        Additionally, a virtual sequence $\{\bs_{t}\}$ is defined as follows: 
        \begin{align}
            \bar{\s}_{t+1}^{(i)} &= \argmin_{\s^{(i)}\in\S_i}\left\{\prt{\y_{t+1}^{(i)}\hat{G}_{t,2}^{(i)}(\tilde{\B}_{t}^{(i)})}\cdot{\s^{(i)}}
            + \frac{1}{2\beta}\left(\s^{(i)} - \s_{t}^{(i)}\right)^2\right\} \quad i \in [n].
        \end{align}
    \end{definition}

    Next, we present a useful lemma, {which is helpful in bounding $\y$ related error term}.
    \begin{lem}[Lemma 9 in \citet{wangnear}]
    \label{lem:y_bound_single}
        Suppose $\{\by_{t}\}, \{\hat{\y}_{t}\}$ are virtual sequences for any $t\geq 0$ in Algorithm \ref{alg:single}. Then, for any $\lambda_1 > 0, \y \in \Y$, it follows that:
        \begin{align}
            \E\brk{\frac{1}{2n\alpha}\prt{\Norm{\y-\y_{t}}^2 - \Norm{\y-\by_{t}}^2 - \Norm{\by_{t}-\y_{t}}^2}} &\leq \frac{1}{2\alpha S}\prt{\Norm{\y-\y_{t}}^2 - \Norm{\y-\y_{t+1}}^2} + \frac{\lambda_1}{2\alpha S}\prt{\Norm{\y-\hat{\y}_{t}}^2 - \Norm{\y-\hat{\y}_{t+1}}^2}\nonumber\\
            &\quad- \frac{1}{2\alpha n}(1-\frac{1}{\lambda_1 S})\Norm{\by_{t+1}-\y_{t}}^2.
        \end{align}
    \end{lem}


     {We define that $\mathcal{G}_t$ is the $\sigma$-algebra generated by $\{\mathcal{B}_0,\mathcal{S}_0, \cdots, \mathcal{B}_{t-1},\mathcal{S}_{t-1},\mathcal{B}_t\}$ and $\mathcal{F}_t$ is the $\sigma$-algebra generated by $\{\mathcal{B}_0,\mathcal{S}_0, \cdots, \mathcal{B}_{t-1},\mathcal{S}_{t-1},\mathcal{B}_t,\mathcal{S}_t\}$. Note that $\mathcal{G}_t \subset \mathcal{F}_t$ and $\y_{t+1}$ is $\mathcal{F}_t$-measurable. Now we proceed to show the descent lemma.}

    \begin{lem}[Descent Lemma]
    \label{lem:descent_inner_single}
        Under Assumption \ref{asm:lip} and \ref{asm:var}, suppose $\{\by_{t}\},\{\ty_{t}\},\{\hy_{t}\},\{\bs_{t}\}$ are virtual sequences for Algorithm \ref{alg:single}. Then, for any $t \in [0,T-1]$, taking expectation over $\F_t$, it holds that:
        \begin{align}
            &\E\brk{L(\u_{t+1},\s_{t+1},\y)-L(\u,\s,\by_{t+1})} \nonumber\\
            &\leq \frac{1}{2\eta}\prt{\Norm{\u-\u_{t}}^2-\E\Norm{\u-\u_{t+1}}^2}
            + \frac{1}{2\beta S}\left(\Norm{\s-\s_{t}}^2
            -\E\Norm{\s-\s_{t+1}}^2\right) + \frac{1}{\alpha S}\prt{\Norm{\y-\y_{t}}^2-\E\Norm{\y-\y_{t+1}}^2} \nonumber\\
            &\quad+ \frac{1}{\alpha S}\prt{\Norm{\y-\hy_{t}}^2-\E\Norm{\y-\hy_{t+1}}^2} + \frac{1}{\alpha S}\left(\Norm{\y-\ty_{t}}^2 - \E\Norm{\y-\ty_{t+1}}^2\right) \nonumber\\
            &\quad+ 64\Omega C_f^2 C_g^2 + \frac{S\alpha\sigma_0^2}{2Bn} + \frac{\alpha\sigma_0^2}{2B} + \frac{\eta C_f^2 \sigma_1^2}{B} + \frac{\eta \delta^2}{S} + \frac{\beta C_f^2\sigma_2^2}{B},
        \end{align}
        where $\Omega=\max\{\eta,\beta\}$.
    \end{lem}

    \begin{proof}
        See Appendix \ref{app:descent_inner_single}.
    \end{proof}

\subsubsection{Proof of Lemma \ref{lem:descent_inner_single}}
\label{app:descent_inner_single}

    \begin{proof}
    By Definition, we have
        \begin{align}
            &L(\u_{t+1},\s_{t+1},\y) - L(\u,\s,\by_{t+1}) \nonumber\\
            &= \frac{1}{n}\sum_{i=1}^n \left(\y^{(i)}g_i(\u_{t+1},\s_{t+1}) - f_i^*(\y^{(i)})\right) - \frac{1}{n}\sum_{i=1}^n \left(\by_{t+1}^{(i)}g_i(\u,\s) - f_i^*(\by_{t+1}^{(i)})\right) \nonumber\\
            &= \frac{1}{n}\sum_{i=1}^n  (\y^{(i)}-\by_{t+1}^{(i)})g_i(\u_{t+1},\s_{t+1}) - \frac{1}{n}\sum_{i=1}^{n}f_i^*(\y^{(i)}) + \frac{1}{n}\sum_{i=1}^{n}f_i^*(\by_{t+1}^{(i)}) \nonumber\\
            &\quad + \frac{1}{n}\sum_{i=1}^n\by_{t+1}^{(i)}\prt{g_i(\u_{t+1},\s_{t+1}) - g_i(\u_{t},\s_t)} + \frac{1}{n}\sum_{i=1}^n\by_{t+1}^{(i)}\prt{g_i(\u_{t},\s_{t}) - g_i(\u,\s)}.
        \end{align}
        {Using the convexity of $g_i$, we obtain the following upper bound:}
        \begin{align}
            &L(\u_{t+1},\s_{t+1},\y) - L(\u,\s,\by_{t+1}) \nonumber\\
            &\leq \underbrace{\frac{1}{n}\sum_{i=1}^n  (\y^{(i)}-\by_{t+1}^{(i)})g_i(\u_{t+1},\s_{t+1})}_{\textrm{I}} - \frac{1}{n}\sum_{i=1}^{n}f_i^*(\y^{(i)}) + \frac{1}{n}\sum_{i=1}^{n}f_i^*(\by_{t+1}^{(i)}) \nonumber\\
            &\quad + \frac{1}{n}\sum_{i=1}^n\by_{t+1}^{(i)}\prt{g_i(\u_{t+1},\s_{t+1}) - g_i(\u_{t},\s_{t})} \underbrace{+ \frac{1}{n}\sum_{i=1}^n\by_{t+1}^{(i)}\prt{\inner{G_{t,1}^{(i)}}{\u_{t}-\u} + {G_{t,2}^{(i)}}(\s_{t}-\s)}}_{\textrm{II}}.
        \end{align}
        {We now analyze terms I and II separately. For term I, we decompose as follows:}
        \begin{align}
            \textrm{I} &= \frac{1}{n}\sum_{i=1}^n  (\y^{(i)}-\by_{t+1}^{(i)})g_i(\u_{t+1},\s_{t+1}) \nonumber\\
            &= \frac{1}{n}\sum_{i=1}^n (\y^{(i)}-\by_{t+1}^{(i)})\hat{g}_{t}^{(i)}({\B}_{t}^{(i)}) + \frac{1}{n}\sum_{i=1}^n (\y^{(i)}-\by_{t+1}^{(i)})\prt{g_i(\u_{t+1},\s_{t+1})-\hat{g}_{t}^{(i)}({\B}_{t}^{(i)})} \nonumber
        \end{align}
        {For term II, by decomposition we obtain:}
        \begin{align}
            \textrm{II} &= \inner{\frac{1}{S}\sum_{i\in\S_{t}}\y_{t+1}^{(i)}\hat{G}_{t,1}^{(i)}(\tilde{\B}_{t}^{(i)}) - \frac{1}{n}\sum_{i=1}^n\by_{t+1}^{(i)}G_{t,1}^{(i)}}{\u-\u_{t+1}} - \inner{\frac{1}{S}\sum_{i\in\S_{t}}\y_{t+1}^{(i)}\hat{G}_{t,1}^{(i)}(\tilde{\B}_{t}^{(i)})}{\u-\u_{t+1}} \nonumber\\
            &\quad + \frac{1}{n}\sum_{i=1}^n\inner{\by_{t+1}^{(i)}G_{t,1}^{(i)}}{\u_{t}-\u_{t+1}} \nonumber\\
            &\quad + \frac{1}{n}\sum_{i=1}^n\inner{\by_{t+1}^{(i)}\prt{\hat{G}_{t,2}^{(i)}(\tilde{\B}_{t}^{(i)})-G_{t,2}^{(i)}}}{\s-\s_{t+1}} - \frac{1}{n}\sum_{i=1}^n\inner{\by_{t+1}^{(i)}\hat{G}_{t,2}^{(i)}(\tilde{\B}_{t}^{(i)})}{\s-\s_{t+1}} \nonumber\\
            &\quad + \frac{1}{n}\sum_{i=1}^n\inner{\by_{t+1}^{(i)}G_{t,2}^{(i)}}{\s_{t}-\s_{t+1}}.
        \end{align}
        {Combining all terms above, we arrive at the key inequality:}
        \begin{align}
        \label{eq:eq11}
            &L(\u_{t+1},\s_{t+1},\y) - L(\u,\s,\by_{t+1}) \nonumber\\
            &\leq \underbrace{\frac{1}{n}\sum_{i=1}^n (\y^{(i)}-\by_{t+1}^{(i)})\hat{g}_{t}^{(i)}({\B}_{t}^{(i)}) - \frac{1}{n}\sum_{i=1}^{n}f_i^*(\y^{(i)}) + \frac{1}{n}\sum_{i=1}^{n}f_i^*(\by_{t+1}^{(i)})}_{\C_1} \nonumber\\
            &\quad + \frac{1}{n}\sum_{i=1}^n (\y^{(i)}-\by_{t+1}^{(i)})\prt{g_i(\u_{t+1},\s_{t+1})-\hat{g}_{t}^{(i)}({\B}_{t}^{(i)})} + \frac{1}{n}\sum_{i=1}^n\by_{t+1}^{(i)}\prt{g_i(\u_{t+1},\s_{t+1}) - g_i(\u_{t},\s_{t,k})} \nonumber\\
            &\quad + \inner{\frac{1}{S}\sum_{i\in\S_{t}}\y_{t+1}^{(i)}\hat{G}_{t,1}^{(i)}(\tilde{\B}_{t}^{(i)}) - \frac{1}{n}\sum_{i=1}^n\by_{t+1}^{(i)}G_{t,1}^{(i)}}{\u-\u_{t+1}} \underbrace{- \inner{\frac{1}{S}\sum_{i\in\S_{t}}\y_{t+1}^{(i)}\hat{G}_{t,1}^{(i)}(\tilde{\B}_{t}^{(i)})}{\u-\u_{t+1}}}_{\C_2} \nonumber\\
            &\quad + \frac{1}{n}\sum_{i=1}^n\inner{\by_{t+1}^{(i)}G_{t,1}^{(i)}}{\u_{t}-\u_{t+1}} \nonumber\\
            &\quad + \frac{1}{n}\sum_{i=1}^n\inner{\by_{t+1}^{(i)}\prt{\hat{G}_{t,2}^{(i)}(\tilde{\B}_{t}^{(i)})-G_{t,2}^{(i)}}}{\s-\s_{t+1}} \underbrace{- \frac{1}{n}\sum_{i=1}^n\inner{\by_{t+1}^{(i)}\hat{G}_{t,2}^{(i)}(\tilde{\B}_{t}^{(i)})}{\s-\s_{t+1}}}_{\C_3} \nonumber\\
            &\quad + \frac{1}{n}\sum_{i=1}^n\inner{\by_{t+1}^{(i)}G_{t,2}^{(i)}}{\s_{t}-\s_{t+1}}.
        \end{align}
        {We now analyze the upper bounds of $\mathcal{C}_1$, $\mathcal{C}_2$, and $\mathcal{C}_3$ in turn.} For $\C_1$, invoking Lemma \ref{lem:proximal_update} and Lemma \ref{lem:y_bound_single}, it holds that
        \begin{align}
            \C_1 &\mathop{\leq}_{\text{Lemma } \ref{lem:proximal_update}} \brk{\frac{1}{2n\alpha}\prt{\Norm{\y-\y_{t}}^2 - \Norm{\y-\by_{t+1}}^2 - \Norm{\by_{t+1}-\y_{t}}^2}} \nonumber\\
            &\mathop{\leq}_{\text{Lemma } \ref{lem:y_bound_single}}  \frac{1}{2\alpha S}\prt{\Norm{\y-\y_{t}}^2 - \Norm{\y-\y_{t+1}}^2} + \frac{\lambda_1}{2\alpha S}\prt{\Norm{\y-\hat{\y}_{t}}^2 - \Norm{\y-\hat{\y}_{t+1}}^2} \nonumber\\
            &\qquad - \frac{1}{2\alpha_t n}(1-\frac{1}{\lambda_1 S})\Norm{\by_{t+1}-\y_{t}}^2.
        \end{align}
        For $\C_2$, noticing here we have
        \begin{align}
            \phi(\u) &\coloneqq \inner{\frac{1}{S}\sum_{i\in\S_{t}}\y_{t+1}^{(i)}\hat{G}_{t,1}^{(i)}(\tilde{\B}_{t}^{(i)})}{\u},
        \end{align}
        and $\phi(\cdot)$ is convex in Lemma \ref{lem:proximal_update}, it follows that
        \begin{align}
            \C_2 \mathop{\leq}_{\text{Lemma }\ref{lem:proximal_update}} \frac{1}{2\eta}\prt{\Norm{\u-\u_{t}}^2-\Norm{\u-\u_{t+1}}^2} - \frac{1}{2\eta}\Norm{\u_{t+1}-\u_{t}}^2.
        \end{align}
        For $\C_3$, in a similar manner, we can obtain
        \begin{align}
            \C_3 \mathop{\leq}_{\text{Lemma }\ref{lem:proximal_update}} \frac{1}{2n\beta}\prt{\Norm{\s-\s_{t}}^2-\Norm{\s-\bs_{t+1}}^2} - \frac{1}{2n\beta}\Norm{\bs_{t+1}-\s_{t}}^2.
        \end{align}
        Substituting the above inequality into (\ref{eq:eq11}), and taking expectation over $\mathcal{F}_{t}$, we can get
        \begin{align}
        \label{eq:eq22}
            &\E\brk{L(\u_{t+1},\s_{t+1},\y) - L(\u,\s,\by_{t+1})} \nonumber\\
            &\leq \frac{1}{2\alpha S}\prt{\Norm{\y-\y_{t}}^2 - \E\Norm{\y-\y_{t+1}}^2} + \frac{\lambda_1}{2\alpha S}\prt{\Norm{\y-\hat{\y}_{t}}^2 - \E\Norm{\y-\hat{\y}_{t+1}}^2} - \frac{1}{2\alpha n}(1-\frac{1}{\lambda_1 S})\E\Norm{\by_{t+1}-\y_{t}}^2 \nonumber\\
            &\quad + \frac{1}{2\eta}\prt{\Norm{\u-\u_{t}}^2-\E\Norm{\u-\u_{t+1}}^2} - \frac{1}{2\eta}\E\Norm{\u_{t+1}-\u_{t}}^2 + \underbrace{\frac{1}{2n\beta}\prt{\Norm{\s-\s_{t}}^2-\E\Norm{\s-\bs_{t+1}}^2} - \frac{1}{2n\beta}\E\Norm{\bs_{t+1}-\s_{t}}^2}_{\D_1} \nonumber\\
            &\quad + \underbrace{\frac{1}{n}\sum_{i=1}^n \E\brk{(\y^{(i)}-\by_{t+1}^{(i)})\prt{g_i(\u_{t+1},\s_{t+1})-\hat{g}_{t}^{(i)}({\B}_{t}^{(i)})}}}_{\D_2} + \underbrace{\frac{1}{n}\sum_{i=1}^n\E\brk{\by_{t+1}^{(i)}\prt{g_i(\u_{t+1},\s_{t+1}) - g_i(\u_{t},\s_{t})}}}_{\D_3} \nonumber\\
            &\quad + \underbrace{\E{\inner{\frac{1}{n}\sum_{i=1}^n\by_{t+1}^{(i)}{G_{t,1}^{(i)}}}{\u_{t}-\u_{t+1}}} + \E{\inner{\frac{1}{n}\sum_{i=1}^n\by_{t+1}^{(i)}{G_{t,2}^{(i)}}}{\s_{t}-\s_{t+1}}}}_{\D_4} \nonumber\\
            &\quad + \underbrace{\E\inner{\frac{1}{S}\sum_{i\in\S_{t}}\y_{t+1}^{(i)}\hat{G}_{t,1}^{(i)}(\tilde{\B}_{t}^{(i)}) - \frac{1}{n}\sum_{i=1}^n\by_{t+1}^{(i)}G_{t,1}^{(i)}}{\u-\u_{t+1}}}_{\D_5} + \underbrace{\frac{1}{n}\sum_{i=1}^n\inner{\by_{t+1}^{(i)}\prt{\hat{G}_{t,2}^{(i)}(\tilde{\B}_{t}^{(i)})-G_{t,2}^{(i)}}}{\s-\s_{t+1}}}_{\D_6}.
        \end{align}
        For $\mathcal{\D}_1$, noticing that $\E\brk{(\s^{(i)}-\bs_{t+1}^{(i)})^2}=\frac{S}{n}(\s^{(i)}-\bs_{t+1}^{(i)})^2+\frac{n-S}{n}(\s^{(i)}-\s_{t}^{(i)})^2$ for any $i \in [n]$, then we obtain
        \begin{align}
            \mathcal{\D}_1 \leq \frac{1}{2S\beta}\prt{\Norm{\s-\s_{t}}^2-\E\Norm{\s-\s_{t+1}}^2} - \frac{1}{2n\beta}\E\Norm{\bs_{t+1}-\s_{t}}^2.
        \end{align}
        {Inspired by Lemma 10 in \citet{wangnear}, we bound $\mathcal{\D}_2$ as following.
        \begin{align}
        \label{eq:D2_1}
            \D_2 &= \frac{1}{n}\sum_{i=1}^n \E\brk{(\y^{(i)}-\by_{t+1}^{(i)})\prt{g_i(\u_{t+1},\s_{t+1})-\hat{g}_{t}^{(i)}({\B}_{t}^{(i)})}} \nonumber\\
            &= \frac{1}{n}\sum_{i=1}^n \E\brk{(\y^{(i)}-\by_{t+1}^{(i)})\prt{g_i(\u_{t+1},\s_{t+1})-g_i(\u_{t},\s_{t})}} - \frac{1}{n}\sum_{i=1}^n \E\brk{(\y^{(i)}-\by_{t+1}^{(i)})\prt{g_i(\u_{t},\s_{t})-\hat{g}_{t}^{(i)}({\B}_{t}^{(i)})}} \nonumber\\
            &\leq \frac{1}{n}\sum_{i=1}^n \E\brk{\Norm{\y^{(i)}-\by_{t+1}^{(i)}}\Norm{g_i(\u_{t+1},\s_{t+1})-g_i(\u_{t},\s_{t})}} - \frac{1}{n}\sum_{i=1}^n \E\brk{(\y^{(i)}-\by_{t+1}^{(i)})\prt{g_i(\u_{t},\s_{t})-\hat{g}_{t}^{(i)}({\B}_{t}^{(i)})}} \nonumber\\
            &\leq \frac{C_g^2}{\lambda_4}\E\Norm{\u_{t+1}-\u_{t}}^2 + \frac{S C_g^2}{\lambda_4 n^2}\E\Norm{\bs_{t+1}-\s_{t}}^2 + 4\lambda_4 C_f^2 - \frac{1}{n}\sum_{i=1}^n \E\brk{(\y^{(i)}-\by_{t+1}^{(i)})\prt{g_i(\u_{t},\s_{t})-\hat{g}_{t}^{(i)}({\B}_{t}^{(i)})}}.
        \end{align}
        The last term in (\ref{eq:D2_1}) is bounded as
        \begin{align}
        \label{eq:D2_2}
            &\frac{1}{n}\sum_{i=1}^n \E\brk{(\y^{(i)}-\by_{t+1}^{(i)})\prt{g_i(\u_{t},\s_{t})-\hat{g}_{t}^{(i)}({\B}_{t}^{(i)})}} \nonumber\\
            &= \frac{1}{n}\sum_{i=1}^n \E\brk{(\y^{(i)}-\y_{t}^{(i)})\prt{g_i(\u_{t+1},\s_{t+1})-\hat{g}_{t}^{(i)}({\B}_{t}^{(i)})}} - \frac{1}{n}\sum_{i=1}^n \E\brk{(\y_{t}^{(i)}-\by_{t+1}^{(i)})\prt{g_i(\u_{t+1},\s_{t+1})-\hat{g}_{t}^{(i)}({\B}_{t}^{(i)})}}.
        \end{align}
        To bound the first term in (\ref{eq:D2_2}), we have $\E\brk{\y_{t}^{(i)}\prt{g_i(\u_{t+1},\s_{t+1})-\hat{g}_{t}^{(i)}({\B}_{t}^{(i)})}|\mathcal{F}_t}=0$. Besides, according to Corollary 12 in \citet{juditsky2011solving}, for some $\lambda_2>0$, we have
        \begin{align}
            \E\brk{\y^{(i)}\prt{g_i(\u_{t+1},\s_{t+1})-\hat{g}_{t}^{(i)}({\B}_{t}^{(i)})}} \leq \lambda_2\E\Norm{\y^{(i)}-\ty_t^{(i)}}^2 - \lambda_2\E\Norm{\y^{(i)}-\ty_t^{(i)}}^2 + \frac{1}{2\lambda_2}\E\Norm{g_i(\u_{t+1},\s_{t+1})-\hat{g}_{t}^{(i)}({\B}_{t}^{(i)})}^2 \nonumber
        \end{align}
        such that
        \begin{align}
        \label{eq:D2_3}
            \frac{1}{n}\sum_{i=1}^n\E\brk{\y^{(i)}\prt{g_i(\u_{t+1},\s_{t+1})-\hat{g}_{t}^{(i)}({\B}_{t}^{(i)})}} 
            \leq \frac{\lambda_2}{n}\E\Norm{\y^{(i)}-\ty_t^{(i)}}^2 - \frac{\lambda_2}{n}\E\Norm{\y^{(i)}-\ty_t^{(i)}}^2 + \frac{\sigma_0^2}{2\lambda_2B},
        \end{align}        
        where $\{\ty_{t}\}$ is also a virtual sequence for Algorithm \ref{alg:single}. For any $\lambda_3>0$, the second term can be bounded as:
        \begin{align}
        \label{eq:D2_4}
            \frac{1}{n}\sum_{i=1}^n \E\brk{(\y_{t}^{(i)}-\by_{t+1}^{(i)})\prt{g_i(\u_{t+1},\s_{t+1})-\hat{g}_{t}^{(i)}({\B}_{t}^{(i)})}} \leq \frac{\lambda_3\sigma_0^2}{2B} + \frac{\E\Norm{\by_{t+1}-\y_{t}}^2}{4\lambda_3 n}.
        \end{align}
        Put (\ref{eq:D2_1}), (\ref{eq:D2_2}), (\ref{eq:D2_3}), and (\ref{eq:D2_4}) together,
        \begin{align}
        \label{eq:D2_5}
            \D_2 &\leq \frac{C_g^2}{\lambda_4}\E\Norm{\u_{t+1}-\u_{t}}^2 + \frac{S C_g^2}{\lambda_4 n^2}\E\Norm{\bs_{t+1}-\s_{t}}^2 + 4\lambda_4 C_f^2 + \frac{\lambda_2}{2n}\prt{\E\Norm{\y-\tilde{\y}_{t}}^2 - \E\Norm{\y-\tilde{\y}_{t+1}}^2} \nonumber\\
            &\quad + \frac{\sigma_0^2}{2B\lambda_2} + \frac{\lambda_3\sigma_0^2}{2B} + \frac{\E\Norm{\by_{t+1}-\y_{t}}^2}{4\lambda_3 n}.
        \end{align}}
        For $\mathcal{\D}_3$, under Assumption \ref{asm:lip}, for some $\lambda_5, \lambda_6 > 0$, we have
        \begin{align}
            \D_3 &\leq C_f C_g \E\Norm{\u_{t+1}-\u_{t}} + \frac{S C_f C_g}{n^2}\E\brk{\sum_{i=1}^n\abs{\bs_{t+1}^{(i)}-\s_{t}^{(i)}}} \nonumber\\
            &\leq \frac{\lambda_5}{2\eta}\E\Norm{\u_{t+1}-\u_{t}}^2 + \frac{\eta C_f^2 C_g^2}{2\lambda_5} + \frac{S\lambda_6}{2n^2\beta}\E\Norm{\bs_{t+1}-\s_{t}}^2 + \frac{S\beta C_f^2 C_g^2}{2n\lambda_6}.
        \end{align}
        For $\mathcal{\D}_4$, similar to the derivations on $\mathcal{\D}_3$, it holds that 
        \begin{align}
            \D_4 \leq \frac{\lambda_5}{2\eta}\E\Norm{\u_{t+1}-\u_{t}}^2 + \frac{\eta C_f^2 C_g^2}{2\lambda_5} + \frac{\lambda_6}{2n\beta}\E\Norm{\bs_{t+1}-\s_{t}}^2 + \frac{\beta C_f^2 C_g^2}{2\lambda_6}.
        \end{align} 
        {Finally invoking Lemma 4 in \citet{juditsky2011solving}(as well as Lemma 7 in \citet{zhang2020optimal}) on $\D_5$ and $\D_6$}, we have
        \begin{align}
            \D_5 &= -\E\inner{\frac{1}{S}\sum_{i\in\S_{t}}\y_{t+1}^{(i)}\hat{G}_{t,1}^{(i)}(\tilde{\B}_{t}^{(i)}) - \frac{1}{n}\sum_{i=1}^n\by_{t+1}^{(i)}G_{t,1}^{(i)}}{\u_{t+1}} \leq \frac{\eta C_f^2\sigma_1^2}{B} + \frac{\eta\delta^2}{S} \nonumber\\
            \D_6 &= -\E\inner{\frac{1}{n}\sum_{i=1}^n\by_{t+1}^{(i)}\hat{G}_{t,2}^{(i)}(\tilde{\B}_{t}^{(i)}) - \frac{1}{n}\sum_{i=1}^n\by_{t+1}^{(i)}G_{t,2}^{(i)}}{\s_{t+1}} \leq \frac{\beta C_f^2\sigma_2^2}{B}.
        \end{align}
        Supposing $\lambda_1 = 1+\frac{1}{S},\lambda_2 = \frac{n}{S\alpha},\lambda_3 =  \alpha$,$\lambda_4 = 8C_g^2 \max\{\beta,\eta\},\lambda_5=\lambda_6 = \frac{1}{8}$ and substituting $\D_1,\D_2,\D_3,\D_4,\D_5,\D_6$ into equation (\ref{eq:eq22}) yields desired result.
    \end{proof}

\subsubsection{Proof of Theorem \ref{thm:iteration_complexity_cvx}}
\label{app:iteration_complexity_cvx}

    \begin{proof}
        {Fix any $t \geq 0$. Applying Lemma~\ref{lem:descent_inner_single} with $(\u, \s) = (\u^*, \s^*)$, where $(\u^*, \s^*) \coloneqq \argmin_{\u \in \U, \s \in \S} F(\u, \s)$}, and summing from $t=0$ to $T-1$, taking expectation on $\F_{0}$, we obtain
        \begin{align}
            &\sum_{t=0}^{T-1} \E\brk{L(\u_{t+1},\s_{t+1},\y)-L(\u^*,\s^*,\by_{t+1})} \nonumber\\
            &\leq \frac{1}{2\eta}{\Norm{\u^*-\u_{0}}^2} + \frac{1}{2S\beta}{\Norm{\s^*-\s_{0}}^2} +  \frac{1}{\alpha S}\prt{\Norm{\y-\y_{0}}^2 + \Norm{\y-\hy_{0}}^2 + \Norm{\y-\ty_{0}}^2} \nonumber\\
            &\quad+ T\prt{64\Omega C_g^2 C_f^2 + \frac{S\alpha\sigma_0^2}{2Bn} + \frac{\alpha\sigma_0^2}{2B}  + \frac{\eta C_f^2 \sigma_1^2}{B} + \frac{\eta \delta^2}{S} + \frac{\beta C_f^2 \sigma_2^2}{B}} .          
        \end{align}
        Since $L(\u,\s,\y)$ is convex on $\u,\s$ and linear on $\y$, we have
        \begin{align}
            &\max_{y} L(\bar{\u},\bar{\s},\y)-L(\u^*,\s^*,\bar{\by})
            \leq
            \max_{\y}\frac{1}{T}\sum_{t=0}^{T-1} L(\u_{t+1},\s_{t+1},\y)-L(\u^{*},\s^{*},\by_{t+1}),
        \end{align}
        where $\bar{\u}=\frac{1}{T}\sum_{t=0}^{T-1}\u_{t+1}, \bar{\s}=\frac{1}{T}\sum_{t=0}^{T-1}\s_{t+1}, \bar{\by}=\frac{1}{T}\sum_{t=0}^{T-1}\by_{t+1}$. Next, consider the left-hand side (LHS):
        \begin{align}
            L(\bar{\u},\bar{\s},\y)-L(\u^{*},\s^{*},\bar{\by}) = \frac{1}{n}\sum_{i=1}^n \left(\y^{(i)}g_i(\bar{\u},\bar{\s}^{(i)}) - f_i^*(\y^{(i)})\right) - \frac{1}{n}\sum_{i=1}^n \left(\bar{\by}^{(i)}g_i(\u^{*},\s^{*(i)}) - f_i^*(\bar{\by}^{(i)})\right).
        \end{align}
        Choose $\y^{(i)} = \tilde{\y}^{(i)} \in \argmax_{\v}\{\v^{(i)}g_i(\bar{\u},\bar{\s}^{(i)}) - f_i^*(\v^{(i)})\}$, {By the definition of conjugate,} we have $\y^{(i)}g_i(\bar{\u},\bar{\s}^{(i)}) - f_i^*(\y^{(i)}) = f_i(g_i(\bar{\u},\bar{\s}^{(i)}))$. By Fenchel-Young inequality, it holds that $\bar{\by}^{(i)}g_i(\u^{*},\s^{*(i)}) - f_i^*(\bar{\by}^{(i)}) \leq f_i(g_i(\u^{*},\s^{*(i)}))$. {Combining the above} $F(\bar{\u},\bar{\s})-F(\u^{*},\s^{*}) \leq \max_{\y}\frac{1}{T}\sum_{t=0}^{T-1} L(\u_{t+1},\s_{t+1},\y)-L(\u^{*},\s^{*},\by_{t+1})$, it follows that
        \begin{align}
            \E F(\bar{\u},\bar{\s})-F(\u^{*},\s^{*}) &\leq \frac{1}{2\eta T}{\Norm{\u^*-\u_{0}}^2} + \frac{1}{2S\beta T}{\Norm{\s^*-\s_{0}}^2} + \frac{3D_{\Y}^2}{\alpha ST} + 64\Omega C_g^2 C_f^2 \nonumber\\
            &\quad+ \frac{\alpha\sigma_0^2}{B} + \frac{\eta C_f^2 \sigma_1^2}{B} + \frac{\eta \delta^2}{S} + \frac{\beta C_f^2 \sigma_2^2}{B}. 
        \end{align}
        Choose $\alpha\asymp \frac{B\epsilon}{\sigma_0^2}, \eta\asymp \min\{\frac{\epsilon}{C_g^2 C_f^2},\frac{B\epsilon}{C_f^2 \sigma_1^2},\frac{S\epsilon}{\delta^2}\}, \beta\asymp \min\{\frac{\epsilon}{C_g^2 C_f^2},\frac{B\epsilon}{C_f^2 \sigma_2^2}\}$ \\
        and $T\asymp \max\{\frac{C_g^2 C_f^2}{\epsilon^2}, \frac{\D_{\S}^2C_g^2 C_f^2}{S\epsilon^2}, \frac{C_f^2 \sigma_1^2}{B\epsilon^2}, \frac{\D_{\S}^2C_f^2 \sigma_2^2}{BS\epsilon^2}, \frac{\delta^2}{S\epsilon^2}, \frac{D_{\Y}^2\sigma_0^2}{BS\epsilon^2}\}$ completes the proof.
    \end{proof}

\subsection{Non-convex Case}
    
    In this section, we present the proof of the non-convex case. The key to the analysis is to apply the convergence analysis of STACO1 for the regularized problem at each stage. However, there is a gap as STACO1 requires $g_i(\w, \s^{(i)})$ to be convex. To address this gap, we reformulate the problem in~(\ref{eq:sub}) as the following:
    \begin{align}
        L_t(\u,\s,\y)
        &= \frac{1}{n}\sum_{i=1}^n \y^{(i)}\left(g_i(\u,\s^{(i)})+\frac{1}{2\tau^{(i)}}\Norm{\u-\u_{t,0}}^2 + \frac{1}{2\tau^{(i)}}\prt{\s^{(i)}-\s^{(i)}_{t,0}}^2\right) - f_i^*(\y^{(i)}) \nonumber\\
        &\quad+ (\frac{1}{2\gamma}-\frac{1}{2n}\sum_{i=1}^n\frac{\y^{(i)}}{\tau^{(i)}})\Norm{\u-\u_{t,0}}^2 + \frac{1}{2n\gamma}\Norm{\s-\s_{t,0}}^2 - \frac{1}{2n}\sum_{i=1}^n\frac{\y^{(i)}}{\tau^{(i)}}\prt{\s^{(i)}-\s_{t,0}^{(i)}}^2, 
    \end{align}
    where $\tau^{(i)}$ is a proper constant. By carefully choosing the value of $\tau^{(i)}$, we can make $g_i(\u,\s^{(i)})+\frac{1}{2\tau^{(i)}}\Norm{\u-\u_{t,0}}^2+\frac{1}{2\tau^{(i)}}\prt{\s^{(i)}-\s^{(i)}_{t,0}}^2$ to be convex in terms of $\u, \s^{(i)}$ such that we can leverage the convergence analysis of STACO1. Nevertheless, our algorithm does not depend on $\tau^{(i)}$ as computing the gradient of $\u$ and $\s^{(i)}$ will remove $\tau^{(i)}$. We now introduce some definitions and notations for our later analysis.
    \begin{align}
        \Phi_\gamma(\u,\s;\u',\s') &\coloneqq F(\u,\s) + \frac{1}{2\gamma}\Norm{\u-\u'}^2 + \frac{1}{2n\gamma}\Norm{\s-\s'}^2 \nonumber\\
        \u_t &= \u_{t,0} \nonumber\\
        \s_t &= \s_{t,0} \nonumber\\
        (\u^\dagger_t,\s^\dagger_t) &= \argmin_{\u\in\U,\s\in\S} \left\{F(\u,\s) + \frac{1}{2\gamma}\Norm{\u-\u_t}^2 + \frac{1}{2n\gamma}\Norm{\s-\s_t}^2\right\}.
    \end{align}

    Since $f_i$ is convex and $g_i$ is non-convex, the function $F(\cdot,\cdot)$ is non-convex with respect to $\u \in \U$ and $\s \in \S$. 

    \begin{lem}
    \label{lem:weak_compositional}
        Under Assumption \ref{asm:lip} and \ref{asm:lip_ncvx}, $F(\cdot,\cdot)$ is $C_f\rho$-weakly convex on $\u\in\U$ and $\frac{C_f\rho}{n}$-weakly convex on $\s\in\S$.
    \end{lem}

    \begin{proof}
        See Appendix \ref{app:weak_compositional}.
    \end{proof}
    {Now we define the virtual sequence for inner loop update in STACO2.}
    \begin{definition}[virtual sequence]
        In Algorithm $\ref{alg:double}$, for any $t$, a virtual sequence $\{\by_{t,k}\}_k$ is defined as follows: 
        \begin{align}
            \bar{\y}_{t,k+1}^{(i)} &= \argmax_{\y^{(i)}\in\Y_i}\left\{\y^{(i)} \hat{g}_t^{(i)}({\B}_{t,k}^{(i)})  - f_i^*(\y^{(i)}) - \frac{1}{2\alpha_{t}} \left(\y^{(i)} - \y_{t,k}^{(i)}\right)^2\right\} \quad i \in [n], \quad \forall k\geq0 \nonumber\\
            \bar{\y}_{t,0}^{(i)} &= \y_{t,0}^{(i)}\quad i \in [n],
        \end{align}        
        and a virtual sequence $\{\bs_{t,k}\}_k$ is defined as follows: 
        \begin{align}
            \bar{\s}_{t,k+1}^{(i)} &=  \argmin_{\s^{(i)}\in\S_i}\left\{\prt{\y_{t,k+1}^{(i)}\hat{G}_{t,k,2}^{(i)}(\tilde{\B}_{t,k}^{(i)})+\frac{1}{\gamma}(\s_{t,k}^{(i)}-\s_{t,0}^{(i)})}\cdot{\s^{(i)}} + \frac{1}{2\beta_t}\left(\s^{(i)} - \s_{t,k}^{(i)}\right)^2\right\} \quad i \in [n],\quad \forall k\geq0 \nonumber\\
            \bar{\s}_{t,0}^{(i)} &= \s_{t,0}^{(i)}\quad i \in [n].
        \end{align}
    \end{definition}

    {Lemma \ref{lem:y_bound} is similar to Lemma \ref{lem:y_bound_single}, but this is for the inner loop in STACO2.}
    \begin{lem}[Lemma 9 in \citet{wangnear}]
    \label{lem:y_bound}
        Suppose $\{\by_{t,k}\}_k, \{\hat{\y}_{t,k}\}_k$ are virtual sequences for any $t\geq 0$ in Algorithm \ref{alg:double}. Then, for any $\lambda_1 > 0, \y \in \Y, t \in [0,T-1]$, the following holds:
        \begin{align}
            &\E\brk{\frac{1}{2n\alpha_t}\prt{\Norm{\y-\y_{t,k}}^2 - \Norm{\y-\by_{t,k+1}}^2 - \Norm{\by_{t,k+1}-\y_{t,k}}^2}} \nonumber\\
            &\leq \frac{1}{2\alpha_t S}\prt{\Norm{\y-\y_{t,k}}^2 - \Norm{\y-\y_{t,k+1}}^2} + \frac{\lambda_1}{2\alpha_t S}\prt{\Norm{\y-\hat{\y}_{t,k}}^2 - \Norm{\y-\hat{\y}_{t,k+1}}^2}\nonumber\\
            &\quad- \frac{1}{2\alpha_t n}(1-\frac{1}{\lambda_1 S})\Norm{\by_{t,k+1}-\y_{t,k}}^2.
        \end{align}
    \end{lem}

    There are two loops update in Algorithm \ref{alg:double}. We first present the descent lemma of the inner loop. Its analysis is similar to that of Lemma \ref{lem:descent_inner_single}. However, since $g_i$ is not convex on $(\u,\s)$, we cannot directly apply Lemma \ref{lem:descent_inner_single}. By carefully reformulating the regularized problem, we can leverage the convergence analysis from the convex case.     {For the inner loop, we define that $\mathcal{G}_{t,k}$ is the $\sigma$-algebra generated by $\{\mathcal{B}_{t,0},\mathcal{S}_{t,0}, \cdots, \mathcal{B}_{t,k-1},\mathcal{S}_{t-1},\mathcal{B}_{t,k}\}$ and $\mathcal{F}_{t,k}$ is the $\sigma$-algebra generated by $\{\mathcal{B}_{t,0},\mathcal{S}_{t,0}, \cdots, \mathcal{B}_{t,k-1},\mathcal{S}_{t,k-1},\mathcal{B}_{t,k},\mathcal{S}_{t,k}\}$. Note that $\mathcal{G}_{t,k} \subset \mathcal{F}_{t,k}$ and $\y_{t,k+1}$ is $\mathcal{F}_{t,k}$-measurable.}
    \begin{lem}[Descent Lemma for Inner Loop]
    \label{lem:descent_inner}
        Under Assumption \ref{asm:lip},\ref{asm:var} and \ref{asm:lip_ncvx}, suppose that $\{\by_{t,k}\}_k,\{\ty_{t,k}\}_k,\{\hy_{t,k}\}_k,\{\bs_{t,k}\}_k$ are virtual sequences for Algorithm \ref{alg:double}, and let $\gamma \leq \frac{1}{2C_f\rho}$ and $\eta_t,\beta_t \leq \frac{\gamma}{8}$. Then, for any $t \in [0,T-1]$ and $k \in [0, K_T-2]$, the following holds:
        \begin{align}
            &\E\brk{L_{\gamma}(\u_{t,k+1},\s_{t,k+1},\y;\u_{t,0},\s_{t,0})-L_{\gamma}(\u,\s,\by_{t,k+1};\u_{t,0},\s_{t,0})} \nonumber\\
            &\leq \frac{1}{2\eta_t}\prt{\Norm{\u-\u_{t,k}}^2-\E\Norm{\u-\u_{t,k+1}}^2} + \frac{1}{2S\beta_t}\left(\Norm{\s-\s_{t,k}}^2 - \E\Norm{\s-\s_{t,k+1}}^2\right) + \frac{1}{\alpha_t S}\prt{\Norm{\y-\y_{t,k}}^2-\E\Norm{\y-\y_{t,k+1}}^2} \nonumber\\
            &\quad + \frac{1}{\alpha_t S}\prt{\Norm{\y-\hy_{t,k}}^2-\E\Norm{\y-\hy_{t,k+1}}^2} + \frac{1}{\alpha_t S}\prt{\Norm{\y-\ty_{t,k}}^2-\E\Norm{\y-\ty_{t,k+1}}^2} \nonumber\\
            &\quad+ 64\Omega_t C_g^2C_f^2 + \frac{S\alpha_t\sigma_0^2}{2Bn} + \frac{\alpha_t\sigma_0^2}{2B} + \frac{\eta_t C_f^2 \sigma_1^2}{B} + \frac{\eta_t \delta^2}{S} + \frac{\beta_t C_f^2\sigma_2^2}{B},
        \end{align}
        where $\Omega_t=\max\{\eta_t,\beta_t\}$.
    \end{lem}

    \begin{proof}
        See Appendix \ref{app:descent_inner}.
    \end{proof}

    \begin{lem}[Proximal Error Bound]
    \label{lem:proximal_error}
    Under Assumption \ref{asm:lip},\ref{asm:var} and \ref{asm:lip_ncvx}, letting $\gamma \leq \frac{1}{2C_f\rho}$ and $\eta_t,\beta_t \leq \frac{\gamma}{8}$ for any $t$ in Algorithm \ref{alg:double}, the following holds:
        \begin{align}
            \E\Phi_\gamma(\u_{t+1},\s_{t+1};\u_t,\s_t) &\leq \Phi_\gamma(\u^\dagger_{t},\s^\dagger_{t};\u_t,\s_t) + \frac{1}{2\eta_{t} K_t}\Norm{\u^{\dagger}_{t}-\u_t}^2 + \frac{1}{2S\beta_t K_t}\Norm{\s^{\dagger}_{t}-\s_{t}}^2 \nonumber\\
            &\quad +\frac{3D_{\Y}^2}{\alpha_t SK_t} + 64\Omega_t C_g^2 C_f^2 + \frac{S\alpha_t\sigma_0^2}{2Bn} + \frac{\alpha_t\sigma_0^2}{2B} + \frac{\eta_t C_f^2 \sigma_1^2}{B} + \frac{\beta_t C_f^2 \sigma_2^2}{B} + \frac{\eta_t \delta^2}{S},
        \end{align}
        where $\Omega_t=\max\{\eta_t,\beta_t\}$.
    \end{lem}

    \begin{proof}
        See Appendix \ref{app:proximal_error}.
    \end{proof}

\subsubsection{Proof of Lemma \ref{lem:weak_compositional}}
\label{app:weak_compositional}

    \begin{proof}
        We first show $F(\u,\s)\coloneqq\frac{1}{n}\sum_{i=1}^n f_i(g_i(\u,\s^{(i)}))$ is weakly convex on $\u$. For convenience, we denote $g_i(\cdot,\s^{(i)})$ as $g_i(\cdot)$. Then $\forall i\in [n]$ and $\x,\y\in\U$, we want to establish: $\forall \lambda \in [0,1]$,
        \begin{align}
            f_i(g_i(\lambda \x + (1 - \lambda) \y)) \leq \lambda f_i(g_i(\x)) + (1 - \lambda) f_i(g_i(\y)) + \frac{\mu}{2} \lambda (1 - \lambda) \Norm{\x - \y}^2,    
        \end{align}
        for some $\mu > 0$. Since $g_i(\x)$ is $\rho$-weakly convex, for any $\x, \y \in \U$ and $\lambda \in [0, 1]$, we have
        \begin{align}
            g_i(\lambda \x + (1 - \lambda) \y) \leq \lambda g_i(\x) + (1 - \lambda) g_i(\y) + \frac{\rho}{2} \lambda (1 - \lambda) \Norm{\x - \y}^2.
        \end{align}
        Noticing $f_i(\cdot)$ is monotone non-decreasing, it holds that
        \begin{align}
        \label{eq:eq_mono}
            f_i(g_i(\lambda \x + (1 - \lambda) \y)) \leq f_i\left( \lambda g_i(\x) + (1 - \lambda) g_i(\y) + \frac{\rho}{2} \lambda (1 - \lambda) \Norm{\x - \y}^2 \right).
        \end{align}
        By the $C_f$-Lipschitz continuity of $f_i$, we have:
        \begin{align}
            f_i(a + \delta) \leq f_i(a) + C_f |\delta|,
        \end{align}
        where \( a = \lambda g_i(\x) + (1 - \lambda) g_i(\y) \) and \( \delta = \frac{\rho}{2} \lambda (1 - \lambda) \Norm{\x - \y}^2 \).
        Applying above inequality into (\ref{eq:eq_mono}), we can obtain:
        \begin{align}
            f_i(g_i(\lambda \x + (1 - \lambda) \y)) &\leq f_i(\lambda g_i(\x) + (1 - \lambda) g_i(\y)) + C_f \cdot \frac{\rho}{2} \lambda (1 - \lambda) \Norm{\x - \y}^2 \nonumber\\
            &\leq \lambda f_i(g_i(\x)) + (1 - \lambda) f_i(g_i(\y)) + C_f \cdot \frac{\rho}{2} \lambda (1 - \lambda) \Norm{\x - \y}^2,
        \end{align}
        where the last inequality holds due to the convexity of $f_i$. Therefore, $f_i(g_i(\u,\s))$ is $C_f \rho$-weakly convex on $\u$. Summing above inequality from $i=1$ to $n$ and averaging, we obtain that $F(\u,\s)$ is $C_f \rho$-weakly convex on $\u$. Next, we show that $F(\u,\s)$ is weakly convex on $\s$. By denoting $g_i(\u,\cdot)$ as $g_i(\cdot)$ and following the similar approach, for any $i\in[n]$ and $x,y \in\S_i$, we have
        \begin{align}
            f_i(g_i(\lambda x + (1 - \lambda) y)) \leq \lambda f_i(g_i(x)) + (1 - \lambda) f_i(g_i(y)) + C_f \cdot \frac{\rho}{2} \lambda (1 - \lambda)(x - y)^2.
        \end{align}
        Noticing $f_i(g_i(\u,\s))$ is weakly-convex to each coordinate of $\s$, $F(\u,\s)$ is $\frac{C_f \rho}{n}$-weakly convex on $\s$.
    \end{proof}

\subsubsection{Proof of Lemma \ref{lem:descent_inner}}
\label{app:descent_inner}

    \begin{proof}
    For analysis, we introduce two auxiliary variables $\tau_{t,k}$ and $\bar{\tau}_{t,k}$, where $\tau_{t,k}^{(i)}\coloneqq \gamma\y_{t,k+1}^{(i)}$ for any $i \in\S_{t,k}$, $\bar{\tau}_{t,k}^{(i)}\coloneqq \gamma\by_{t,k+1}^{(i)}$ for any $i \in [n]$. {Before delving into the formal proof of Lemma \ref{lem:descent_inner}, we briefly highlight its role: it provides a variation bound for stochastic gradients under block-coordinate updates, which is the central technical novelty enabling us to prove parallel mini-batch speedup. The proof proceeds in three steps: (i) decomposing the variation of block-coordinate updates, and (ii) bounding the dependence on different batches.}

\paragraph*{{Step 1: Decomposing the variation of block-coordinate updates.}}
    By definition, we have
    \label{proof:descent_inner}
        \begin{align}
            &L_{\gamma}(\u_{t,k+1},\s_{t,k+1},\y;\u_{t,0},\s_{t,0}) - L_{\gamma}(\u,\s,\by_{t,k+1};\u_{t,0},\s_{t,0}) \nonumber\\
            &= \frac{1}{n}\sum_{i=1}^n \left(\y^{(i)}g_i(\u_{t,k+1}, \s_{t,k+1}^{(i)}) - f_i^*(\y^{(i)})\right) + \frac{1}{2\gamma}\Norm{\u_{t,k+1}-\u_{t,0}}^2 + \frac{1}{2n\gamma}\Norm{\s_{t,k+1}-\s_{t,0}}^2\nonumber\\
            &\quad - \frac{1}{n}\sum_{i=1}^n \left(\by_{t,k+1}^{(i)}g_i(\u, \s) - f_i^*(\by_{t,k+1}^{(i)})\right) - \frac{1}{2\gamma}\Norm{\u-\u_{t,0}}^2 - \frac{1}{2n\gamma}\Norm{\s-\s_{t,0}}^2\nonumber\\
            &= \frac{1}{n}\sum_{i=1}^n \left(\y^{(i)}\left(g_i(\u_{t,k+1}, \s_{t,k+1}^{(i)})+\frac{1}{2\bar{\tau}_{t,k}^{(i)}}\Norm{\u_{t,k+1}-\u_{t,0}}^2+\frac{1}{2\bar{\tau}_{t,k}^{(i)}}(\s_{t,k+1}^{(i)}-\s_{t,0}^{(i)})^2\right) - f_i^*(\y^{(i)})\right) \nonumber\\
            &\quad - \frac{1}{n}\sum_{i=1}^n \left(\by_{t,k+1}^{(i)}\left(g_i(\u, \s)+\frac{1}{2\bar{\tau}_{t,k}^{(i)}}\Norm{\u-\u_{t,0}}^2+\frac{1}{2\bar{\tau}_{t,k}^{(i)}}(\s^{(i)}-\s_{t,0}^{(i)})^2\right) - f_i^*(\by_{t,k+1}^{(i)})\right) \nonumber\\
            &\quad + \frac{1}{2\gamma}\Norm{\u_{t,k+1}-\u_{t,0}}^2 - \frac{1}{n}\sum_{i=1}^n\frac{\y^{(i)}}{2\bar{\tau}_{t,k}^{(i)}}\Norm{\u_{t,k+1}-\u_{t,0}}^2 - \frac{1}{2\gamma}\Norm{\u-\u_{t,0}}^2 + \frac{1}{n}\sum_{i=1}^n\frac{\by_{t,k+1}^{(i)}}{2\bar{\tau}_{t,k}^{(i)}}\Norm{\u-\u_{t,0}}^2 \nonumber\\
            &\quad + \frac{1}{2n\gamma}\Norm{\s_{t,k+1}-\s_{t,0}}^2 - \frac{1}{n}\sum_{i=1}^n\frac{\y^{(i)}}{2\bar{\tau}_{t,k}^{(i)}}(\s_{t,k+1}^{(i)}-\s_{t,0}^{(i)})^2 - \frac{1}{2n\gamma}\Norm{\s-\s_{t,0}}^2 + \frac{1}{n}\sum_{i=1}^n\frac{\by_{t,k+1}^{(i)}}{2\bar{\tau}_{t,k}^{(i)}}(\s^{(i)}-\s_{t,0}^{(i)})^2 \nonumber\\
            &= \frac{1}{n}\sum_{i=1}^n  (\y^{(i)}-\by_{t,k+1}^{(i)})\prt{g_i(\u_{t,k+1}, \s_{t,k+1}^{(i)})+\frac{1}{2\bar{\tau}_{t,k}^{(i)}}\Norm{\u_{t,k+1}-\u_{t,0}}^2+\frac{1}{2\bar{\tau}_{t,k}^{(i)}}(\s_{t,k+1}^{(i)}-\s_{t,0}^{(i)})^2} \nonumber\\
            &\quad + \frac{1}{n}\sum_{i=1}^n\by_{t,k+1}^{(i)}\left(g_i(\u_{t,k+1}, \s_{t,k+1}^{(i)})+\frac{1}{2\bar{\tau}_{t,k}^{(i)}}\Norm{\u_{t,k+1}-\u_{t,0}}^2+\frac{1}{2\bar{\tau}_{t,k}^{(i)}}(\s_{t,k+1}^{(i)}-\s_{t,0}^{(i)})^2 \right.\nonumber\\
            &\qquad - \left.g_i(\u_{t,k}, \s_{t,k}^{(i)})-\frac{1}{2\bar{\tau}_{t,k}^{(i)}}\Norm{\u_{t,k}-\u_{t,0}}^2-\frac{1}{2\bar{\tau}_{t,k}^{(i)}}(\s_{t,k}^{(i)}-\s_{t,0}^{(i)})^2 \right) \nonumber\\
            &\quad + \frac{1}{n}\sum_{i=1}^n\by_{t,k+1}^{(i)}\left(g_i(\u_{t,k}, \s_{t,k}^{(i)})+\frac{1}{2\bar{\tau}_{t,k}^{(i)}}\Norm{\u_{t,k}-\u_{t,0}}^2+\frac{1}{2\bar{\tau}_{t,k}^{(i)}}(\s_{t,k}^{(i)}-\s_{t,0}^{(i)})^2 \right. \nonumber\\
            &\qquad - \left. g_i(\u, \s^{(i)})-\frac{1}{2\bar{\tau}_{t,k}^{(i)}}\Norm{\u-\u_{t,0}}^2-\frac{1}{2\bar{\tau}_{t,k}^{(i)}}(\s^{(i)}-\s_{t,0}^{(i)})^2\right) \nonumber\\
            &\quad + \frac{1}{2\gamma}\Norm{\u_{t,k+1}-\u_{t,0}}^2 - \frac{1}{n}\sum_{i=1}^n\frac{\y^{(i)}}{2\bar{\tau}_{t,k}^{(i)}}\Norm{\u_{t,k+1}-\u_{t,0}}^2 - \frac{1}{2\gamma}\Norm{\u-\u_{t,0}}^2 + \frac{1}{n}\sum_{i=1}^n\frac{\by_{t,k+1}^{(i)}}{2\bar{\tau}_{t,k}^{(i)}}\Norm{\u-\u_{t,0}}^2 \nonumber\\
            &\quad + \frac{1}{2n\gamma}\Norm{\s_{t,k+1}-\s_{t,0}}^2 - \frac{1}{n}\sum_{i=1}^n\frac{\y^{(i)}}{2\bar{\tau}_{t,k}^{(i)}}(\s_{t,k+1}^{(i)}-\s_{t,0}^{(i)})^2 - \frac{1}{2n\gamma}\Norm{\s-\s_{t,0}}^2 + \frac{1}{n}\sum_{i=1}^n\frac{\by_{t,k+1}^{(i)}}{2\bar{\tau}_{t,k}^{(i)}}(\s^{(i)}-\s_{t,0}^{(i)})^2 \nonumber\\
            &\quad- \frac{1}{n}\sum_{i=1}^{n}f_i^*(\y^{(i)}) + \frac{1}{n}\sum_{i=1}^{n}f_i^*(\by_{t,k+1}^{(i)}).
        \end{align}
        {Observe that for any $(\u,\s) \in (\U,\S)$, the function 
        $g_i(\u,\s^{(i)}) + \frac{1}{2\bar{\tau}_{t,k}^{(i)}}\Norm{\u-\u'}^2 + \frac{1}{2\bar{\tau}_{t,k}^{(i)}}(\s^{(i)}-\s'^{(i)})^2$ 
        is convex with respect to $\u$ and $\s^{(i)}$, for any fixed $(\u', \s') \in (\U, \S)$, since 
        $\frac{1}{\bar{\tau}_{t,k}^{(i)}} = \frac{1}{\gamma\by_{t,k+1}^{(i)}} \geq \rho$. This convexity enables us to apply a first-order approximation and derive the following bound:}
        \begin{align}
            &L_{\gamma}(\u_{t,k+1},\s_{t,k+1},\y;\u_{t,0},\s_{t,0}) - L_{\gamma}(\u,\s,\by_{t,k+1};\u_{t,0},\s_{t,0}) \nonumber\\
            &\leq \underbrace{\frac{1}{n}\sum_{i=1}^n  (\y^{(i)}-\by_{t,k+1}^{(i)})\prt{g_i(\u_{t,k+1}, \s_{t,k+1}^{(i)})+\frac{1}{2\bar{\tau}_{t,k}^{(i)}}\Norm{\u_{t,k+1}-\u_{t,0}}^2+\frac{1}{2\bar{\tau}_{t,k}^{(i)}}(\s_{t,k+1}^{(i)}-\s_{t,0}^{(i)})^2} }_{\textrm{I}} \nonumber\\
            &\quad \underbrace{+ \frac{1}{n}\sum_{i=1}^n\by_{t,k+1}^{(i)}\left(g_i(\u_{t,k+1}, \s_{t,k+1}^{(i)})+\frac{1}{2\bar{\tau}_{t,k}^{(i)}}\Norm{\u_{t,k+1}-\u_{t,0}}^2+\frac{1}{2\bar{\tau}_{t,k}^{(i)}}(\s_{t,k+1}^{(i)}-\s_{t,0}^{(i)})^2 \right.}_{\textrm{III}} \nonumber\\
            &\qquad \underbrace{- \left.g_i(\u_{t,k}, \s_{t,k}^{(i)})-\frac{1}{2\bar{\tau}_{t,k}^{(i)}}\Norm{\u_{t,k}-\u_{t,0}}^2-\frac{1}{2\bar{\tau}_{t,k}^{(i)}}(\s_{t,k}^{(i)}-\s_{t,0}^{(i)})^2 \right)}_{\textrm{III}} \nonumber\\
            &\quad \underbrace{+ \frac{1}{n}\sum_{i=1}^n\by_{t,k+1}^{(i)}\brk{\inner{G_{t,k,1}^{(i)}+\frac{1}{\bar{\tau}_{t,k}^{(i)}}(\u_{t,k}-\u_{t,0})}{\u_{t,k}-\u} + \prt{G_{t,k,2}^{(i)}+\frac{1}{\bar{\tau}_{t,k}^{(i)}}(\s_{t,k}^{(i)}-\s_{t,0}^{(i)})}(\s_{t,k}^{(i)}-\s^{(i)})}}_{\textrm{II}}   \nonumber\\
            &\quad \underbrace{- \frac{1}{n}\sum_{i=1}^n \frac{\y^{(i)}}{2\bar{\tau}_{t,k}^{(i)}}\Norm{\u_{t,k+1}-\u_{t,0}}^2}_{\textrm{I}} + \frac{1}{n}\sum_{i=1}^n\frac{\by_{t,k+1}^{(i)}}{2\bar{\tau}_{t,k}^{(i)}}\Norm{\u-\u_{t,0}}^2 \underbrace{- \frac{1}{n}\sum_{i=1}^n\frac{\by_{t,k+1}^{(i)}}{2\bar{\tau}_{t,k}^{(i)}}\Norm{\u_{t,k+1}-\u_{t,0}}^2}_{\textrm{III}} \nonumber\\ 
            &\quad \underbrace{+ \frac{1}{n}\sum_{i=1}^n\frac{\by_{t,k+1}^{(i)}}{2\bar{\tau}_{t,k}^{(i)}}\Norm{\u_{t,k+1}-\u_{t,0}}^2}_{\textrm{I}} - \frac{1}{n}\sum_{i=1}^n\frac{\by_{t,k+1}^{(i)}}{2\bar{\tau}_{t,k}^{(i)}}\Norm{\u_{t,k}-\u_{t,0}}^2 \underbrace{+ \frac{1}{n}\sum_{i=1}^n\frac{\by_{t,k+1}^{(i)}}{2\bar{\tau}_{t,k}^{(i)}}\Norm{\u_{t,k}-\u_{t,0}}^2}_{\textrm{III}} \nonumber\\
            &\quad \underbrace{- \frac{1}{n}\sum_{i=1}^n \frac{\y^{(i)}}{2\bar{\tau}_{t,k}^{(i)}}(\s_{t,k+1}^{(i)}-\s_{t,0}^{(i)})^2}_{\textrm{I}} + \frac{1}{n}\sum_{i=1}^n\frac{\by_{t,k+1}^{(i)}}{2\bar{\tau}_{t,k}^{(i)}}(\s^{(i)}-\s_{t,0}^{(i)})^2 \underbrace{- \frac{1}{n}\sum_{i=1}^n\frac{\by_{t,k+1}^{(i)}}{2\bar{\tau}_{t,k}^{(i)}}(\s_{t,k+1}^{(i)}-\s_{t,0}^{(i)})^2}_{\textrm{III}} \nonumber\\ 
            &\quad \underbrace{+ \frac{1}{n}\sum_{i=1}^n\frac{\by_{t,k+1}^{(i)}}{2\bar{\tau}_{t,k}^{(i)}}(\s_{t,k+1}^{(i)}-\s_{t,0}^{(i)})^2}_{\textrm{I}} - \frac{1}{n}\sum_{i=1}^n\frac{\by_{t,k+1}^{(i)}}{2\bar{\tau}_{t,k}^{(i)}}(\s_{t,k}^{(i)}-\s_{t,0}^{(i)})^2 \underbrace{+ \frac{1}{n}\sum_{i=1}^n\frac{\by_{t,k+1}^{(i)}}{2\bar{\tau}_{t,k}^{(i)}}(\s_{t,k}^{(i)}-\s_{t,0}^{(i)})^2}_{\textrm{III}} \nonumber\\
            &\quad + \frac{1}{2\gamma}\Norm{\u_{t,k+1}-\u_{t,0}}^2 - \frac{1}{2\gamma}\Norm{\u-\u_{t,0}}^2 + \frac{1}{2n\gamma}\Norm{\s_{t,k+1}-\s_{t,0}}^2 - \frac{1}{2n\gamma}\Norm{\s-\s_{t,0}}^2 \nonumber\\
            &\quad - \frac{1}{n}\sum_{i=1}^{n}f_i^*(\y^{(i)}) + \frac{1}{n}\sum_{i=1}^{n}f_i^*(\by_{t,k+1}^{(i)}).
        \end{align}
        Noticing
        \begin{align}
            \textrm{I} &= \frac{1}{n}\sum_{i=1}^n  (\y^{(i)}-\by_{t,k+1}^{(i)})g_i(\u_{t,k+1}, \s_{t,k+1}^{(i)}) \nonumber\\
            &= \frac{1}{n}\sum_{i=1}^n (\y^{(i)}-\by_{t,k+1}^{(i)})\hat{g}_{t,k}^{(i)}({\B}_{t,k}^{(i)}) + \frac{1}{n}\sum_{i=1}^n (\y^{(i)}-\by_{t,k+1}^{(i)})\prt{g_i(\u_{t,k+1}, \s_{t,k+1}^{(i)})-\hat{g}_{t,k}^{(i)}({\B}_{t,k}^{(i)})} \nonumber\\
            \textrm{II} &= \inner{\frac{1}{S}\sum_{i\in\S_{t,k}}\y_{t,k+1}^{(i)}\prt{\hat{G}_{t,k, 1}^{(i)}(\tilde{\B}_{t,k}^{(i)})+\frac{1}{\tau_{t,k}^{(i)}}(\u_{t,k}-\u_{t,0})} - \frac{1}{n}\sum_{i=1}^n\by_{t,k+1}^{(i)}\prt{G_{t,k,1}^{(i)}+\frac{1}{\bar{\tau}_{t,k}^{(i)}}(\u_{t,k}-\u_{t,0})}}{\u-\u_{t,k+1}} \nonumber\\
            &\quad - \frac{1}{S}\sum_{i\in\S_{t,k}}\inner{\y_{t,k+1}^{(i)}\prt{\hat{G}_{t,k,1}^{(i)}(\tilde{\B}_{t,k}^{(i)})+\frac{1}{\tau_{t,k}^{(i)}}(\u_{t,k}-\u_{t,0})}}{\u-\u_{t,k+1}} \nonumber\\
            &\quad + \frac{1}{n}\sum_{i=1}^n\inner{\by_{t,k+1}^{(i)}\prt{G_{t,k,1}^{(i)}+\frac{1}{\bar{\tau}_{t,k}^{(i)}}(\u_{t,k}-\u_{t,0})}}{\u_{t,k}-\u_{t,k+1}} \nonumber\\
            &\quad + \frac{1}{n}\sum_{i=1}^{n}{\by_{t,k+1}^{(i)}({\hat{G}_{t,k,2}^{(i)}(\tilde{\B}_{t,k}^{(i)}) - G_{t,k,2}^{(i)}})}(\s^{(i)}-\bs_{t,k+1}^{(i)}) \nonumber\\
            &\quad - \frac{1}{n}\sum_{i=1}^{n}{\by_{t,k+1}^{(i)}\prt{\hat{G}_{t,k,2}^{(i)}(\tilde{\B}_{t,k}^{(i)})+\frac{1}{\bar{\tau}_{t,k}^{(i)}}(\s_{t,k}^{(i)}-\s_{t,0}^{(i)})}}(\s^{(i)}-\bs_{t,k+1}^{(i)}) \nonumber\\
            &\quad + \frac{1}{n}\sum_{i=1}^n{\by_{t,k+1}^{(i)}\prt{G_{t,k,2}^{(i)}+\frac{1}{\bar{\tau}_{t,k}^{(i)}}(\s_{t,k}^{(i)}-\s_{t,0}^{(i)})}}(\s_{t,k}^{(i)}-\bs_{t,k+1}^{(i)}) \nonumber\\
            \textrm{III} &= \frac{1}{n}\sum_{i=1}^n\by_{t,k+1}^{(i)}\prt{g_i(\u_{t,k+1}, \s_{t,k+1}^{(i)}) - g_i(\u_{t,k}, \s_{t,k}^{(i)})},
        \end{align}
        and replacing ${\tau}_{t,k}^{(i)}=\gamma\y_{t,k+1}^{(i)}$ and $\bar{\tau}_{t,k}^{(i)}=\gamma\by_{t,k+1}^{(i)}$, it follows that
        \begin{align}
        \label{eq:eq1}
            &L_{\gamma}(\u_{t,k+1}, \s_{t,k+1}, \y;\u_{t,0},\s_{t,0}) - L_{\gamma}(\u, \s, \by_{t,k+1};\u_{t,0},\s_{t,0}) \nonumber\\
            &\leq \underbrace{\frac{1}{n}\sum_{i=1}^n (\y^{(i)}-\by_{t,k+1}^{(i)})\hat{g}_{t,k}^{(i)}({\B}_{t,k}^{(i)}) - \frac{1}{n}\sum_{i=1}^{n}f_i^*(\y^{(i)}) + \frac{1}{n}\sum_{i=1}^{n}f_i^*(\by_{t,k+1}^{(i)})}_{\C_1} \nonumber\\
            &\quad + \frac{1}{n}\sum_{i=1}^n (\y^{(i)}-\by_{t,k+1}^{(i)})\prt{g_i(\u_{t,k+1}, \s_{t,k+1}^{(i)})-\hat{g}_{t,k}^{(i)}({\B}_{t,k}^{(i)})} + \frac{1}{n}\sum_{i=1}^n\by_{t,k+1}^{(i)}\prt{g_i(\u_{t,k+1}, \s_{t,k+1}^{(i)}) - g_i(\u_{t,k}, \s_{t,k}^{(i)})} \nonumber\\
            &\quad + \inner{\frac{1}{S}\sum_{i\in\S_{t,k}}\prt{\y_{t,k+1}^{(i)}\hat{G}_{t,k,1}^{(i)}(\tilde{\B}_{t,k}^{(i)})+\frac{1}{\gamma}(\u_{t,k}-\u_{t,0})} - \frac{1}{n}\sum_{i=1}^n\prt{\by_{t,k+1}^{(i)}G_{t,k,1}^{(i)}+\frac{1}{\gamma}(\u_{t,k}-\u_{t,0})}}{\u-\u_{t,k+1}} \nonumber\\
            &\quad \underbrace{- \frac{1}{S}\sum_{i\in\S_{t,k}}\inner{\y_{t,k+1}^{(i)}{\hat{G}_{t,k,1}^{(i)}(\tilde{\B}_{t,k}^{(i)})+\frac{1}{\gamma}(\u_{t,k}-\u_{t,0})}}{\u-\u_{t,k+1}}}_{\C_2} \nonumber\\
            &\quad + \frac{1}{n}\sum_{i=1}^{n}\by_{t,k+1}^{(i)}(\hat{G}_{t,k,2}^{(i)}(\tilde{\B}_{t,k}^{(i)}) - G_{t,k,2}^{(i)})(\s^{(i)}-\bs_{t,k+1}^{(i)}) \nonumber\\
            &\quad \underbrace{- \frac{1}{n}\sum_{i=1}^{n}\prt{\by_{t,k+1}^{(i)}{\hat{G}_{t,k,2}^{(i)}(\tilde{\B}_{t,k}^{(i)})+\frac{1}{\gamma}(\s_{t,k}^{(i)}-\s_{t,0}^{(i)})}}(\s^{(i)}-\bs_{t,k+1}^{(i)})}_{\C_3} \nonumber\\
            &\quad + \frac{1}{n}\sum_{i=1}^n\inner{\by_{t,k+1}^{(i)}{G_{t,k,1}^{(i)}+\frac{1}{\gamma}(\u_{t,k}-\u_{t,0})}}{\u_{t,k}-\u_{t,k+1}} + \frac{1}{n}\sum_{i=1}^n\prt{\by_{t,k+1}^{(i)}{G_{t,k,2}^{(i)}+\frac{1}{\gamma}(\s_{t,k}^{(i)}-\s_{t,0}^{(i)})}}(\s_{t,k}^{(i)}-\bs_{t,k+1}^{(i)}) \nonumber\\
            &\quad + \frac{1}{2n\gamma}\sum_{i=1}^n\Norm{\u_{t,k+1}-\u_{t,0}}^2 - \frac{1}{2n\gamma}\sum_{i=1}^n\Norm{\u_{t,k}-\u_{t,0}}^2 + \frac{1}{2n\gamma}\sum_{i=1}^n(\bs_{t,k+1}^{(i)}-\s_{t,0}^{(i)})^2 - \frac{1}{2n\gamma}\sum_{i=1}^n(\s_{t,k}^{(i)}-\s_{t,0}^{(i)})^2.
        \end{align}
        {For $\C_1$, by applying Lemma~\ref{lem:proximal_update} followed by Lemma~\ref{lem:y_bound}, we obtain:}
        \begin{align}
            \C_1 &\mathop{\leq}_{\text{Lemma } \ref{lem:proximal_update}} \brk{\frac{1}{2n\alpha_t}\prt{\Norm{\y-\y_{t,k}}^2 - \Norm{\y-\by_{t,k+1}}^2 - \Norm{\by_{t,k+1}-\y_{t,k}}^2}} \nonumber\\
            &\mathop{\leq}_{\text{Lemma } \ref{lem:y_bound}}  \frac{1}{2\alpha_t S}\prt{\Norm{\y-\y_{t,k}}^2 - \Norm{\y-\y_{t,k+1}}^2} + \frac{\lambda_1}{2\alpha_t S}\prt{\Norm{\y-\hat{\y}_{t,k}}^2 - \Norm{\y-\hat{\y}_{t,k+1}}^2} \nonumber\\
            &\qquad - \frac{1}{2\alpha_t n}(1-\frac{1}{\lambda_1 S})\Norm{\by_{t,k+1}-\y_{t,k}}^2.
        \end{align}
        {For $\C_2$, we define the auxiliary function:
        \begin{align}
            \phi(\u) \coloneqq \left\langle \frac{1}{S} \sum_{i \in \S_{t,k}} \y_{t,k+1}^{(i)} \left( \hat{G}_{t,k,1}^{(i)}(\tilde{\B}_{t,k}^{(i)}) + \frac{1}{\tau_{t,k}^{(i)}}(\u_{t,k} - \u_{t,0}) \right), \u \right\rangle.
        \end{align}
        Substituting $\tau_{t,k}^{(i)} = \gamma \y_{t,k+1}^{(i)}$, this becomes
        \begin{align}
            \phi(\u) = \left\langle \frac{1}{S} \sum_{i \in \S_{t,k}} \y_{t,k+1}^{(i)} \hat{G}_{t,k,1}^{(i)}(\tilde{\B}_{t,k}^{(i)}) + \frac{1}{\gamma} (\u_{t,k} - \u_{t,0}), \u \right\rangle.
        \end{align}
        Since $\phi(\cdot)$ is convex, applying Lemma~\ref{lem:proximal_update} yields:
        \begin{align}
            \C_2 \leq \frac{1}{2\eta_t} \left( \Norm{\u - \u_{t,k}}^2 - \Norm{\u - \u_{t,k+1}}^2 \right) - \frac{1}{2\eta_t} \Norm{\u_{t,k+1} - \u_{t,k}}^2.
        \end{align}}
        For $\C_3$, following the similar manner with $\C_2$, for any $i$ in $\S_{t,k}$, we can get
        \begin{align}
            \C_3 \mathop{\leq}_{\text{Lemma }\ref{lem:proximal_update}} \frac{1}{2n\beta_t}\prt{\Norm{\s-\s_{t,k}}^2-\Norm{\s-\bs_{t,k+1}}^2} - \frac{1}{2n\beta_{t}}\Norm{\bs_{t,k+1}-\s_{t,k}}^2.
        \end{align}        
        {Substituting the above inequalities into (\ref{eq:eq1}) and taking expectation with respect to $\mathcal{F}_{t,k}$ yields:}
        \begin{align}
        \label{eq:eq2}
            &\E\brk{L_{\gamma}(\u_{t,k+1},\s_{t,k+1},\y;\u_{t,0},\s_{t,0}) - L_{\gamma}(\u,\s,\by_{t,k+1};\u_{t,0},\s_{t,0})} \nonumber\\
            &\leq \frac{1}{2\alpha_t S}\prt{\Norm{\y-\y_{t,k}}^2 - \E\Norm{\y-\y_{t,k+1}}^2} + \frac{\lambda_1}{2\alpha_t S}\prt{\Norm{\y-\hat{\y}_{t,k}}^2 - \E\Norm{\y-\hat{\y}_{t,k+1}}^2} \nonumber\\
            &\quad - \frac{1}{2\alpha_t n}(1-\frac{1}{\lambda_1 S})\E\Norm{\by_{t,k+1}-\y_{t,k}}^2 \nonumber\\
            &\quad + \frac{1}{2\eta_{t}}\prt{\Norm{\u-\u_{t,k}}^2-\E\Norm{\u-\u_{t,k+1}}^2} - \frac{1}{2\eta_{t}}\E\Norm{\u_{t,k+1}-\u_{t,k}}^2 \nonumber\\
            &\quad + \underbrace{\frac{1}{2n\beta_t}\prt{\Norm{\s-\s_{t,k}}^2-\E\Norm{\s-\bs_{t,k+1}}^2} - \frac{1}{2n\beta_{t}}\E\Norm{\bs_{t,k+1}-\s_{t,k}}^2}_{\D_1} \nonumber\\
            &\quad + \underbrace{\frac{1}{n}\sum_{i=1}^n \E\brk{(\y^{(i)}-\by_{t,k+1}^{(i)})\prt{g_i(\u_{t,k+1}, \s_{t,k+1}^{(i)})-\hat{g}_{t,k}^{(i)}({\B}_{t,k}^{(i)})}}}_{\D_2} \nonumber\\
            &\quad + \underbrace{\frac{1}{n}\sum_{i=1}^n\E\brk{\by_{t,k+1}^{(i)}\prt{g_i(\u_{t,k+1}, \s_{t,k+1}^{(i)}) - g_i(\u_{t,k}, \s_{t,k}^{(i)})}}}_{\D_3} \nonumber\\
            &\quad + \underbrace{\frac{1}{n}\sum_{i=1}^n\E{\inner{\by_{t,k+1}^{(i)}{G_{t,k,1}^{(i)}}}{\u_{t,k}-\u_{t,k+1}}} + {\frac{1}{n}\sum_{i=1}^n\E\inner{\by_{t,k+1}^{(i)}{G_{t,k,2}^{(i)}}}{\s_{t,k}^{(i)}-\bs_{t,k+1}^{(i)}}}}_{\D_4} \nonumber\\
            &\quad + \underbrace{\E\brk{\frac{1}{\gamma}\inner{\u_{t,k}-\u_{t,0}}{\u_{t,k}-\u_{t,k+1}}} + \E\brk{\frac{1}{2\gamma}\prt{\Norm{\u_{t,k+1}-\u_{t,0}}^2-\Norm{\u_{t,k}-\u_{t,0}}^2}}}_{\D_5} \nonumber\\
            &\quad + \underbrace{\E\brk{\frac{1}{n\gamma}\inner{\s_{t,k}-\s_{t,0}}{\s_{t,k}-\bs_{t,k+1}}} + \E\brk{\frac{1}{2n\gamma}\prt{\Norm{\bs_{t,k+1}-\s_{t,0}}^2-\Norm{\s_{t,k}-\s_{t,0}}^2}}}_{\D_6} \nonumber\\
            &\quad + \underbrace{\E\inner{\frac{1}{S}\sum_{i\in\S_{t,k}}\y_{t,k+1}^{(i)}\hat{G}_{t,k,1}^{(i)}(\tilde{\B}_{t,k}^{(i)}) - \frac{1}{n}\sum_{i=1}^n\by_{t,k+1}^{(i)}G_{t,k,1}^{(i)}}{\u-\u_{t,k+1}}}_{\D_7} \nonumber\\
            &\quad + \underbrace{\frac{1}{n}\sum_{i=1}^{n}\E\brk{\by_{t,k+1}^{(i)}\prt{\hat{G}_{t,k,2}^{(i)}(\tilde{\B}_{t,k}^{(i)}) - G_{t,k,2}^{(i)}}(\s^{(i)}-\bs_{t,k+1}^{(i)})}}_{\D_8}.
        \end{align}

        \paragraph*{{Step 2: Bounding the dependence on different batches.}}
        For $\mathcal{\D}_1$, notice that $\E\brk{(\s^{(i)}-\bs_{t,k+1}^{(i)})^2}=\frac{S}{n}(\s^{(i)}-\bs_{t,k+1}^{(i)})^2+\frac{n-S}{n}(\s^{(i)}-\s_{t,k}^{(i)})^2$ for any $i \in [n]$. Then, it holds that
        \begin{align}
            \mathcal{\D}_1 \leq \frac{1}{2S\beta_t}\prt{\Norm{\s-\s_{t,k}}^2-\E\Norm{\s-\s_{t,k+1}}^2} - \frac{1}{2n\beta_{t}}\E\Norm{\bs_{t,k+1}-\s_{t,k}}^2.
        \end{align}
        {For $\mathcal{\D}_2$, following a similar manner we show in (\ref{eq:D2_5}), for some $\lambda_2, \lambda_3, \lambda_4 > 0$, we have}
        \begin{align}
            \D_2 &\leq \frac{C_g^2}{\lambda_4}\E\Norm{\u_{t,k+1}-\u_{t,k}}^2 + \frac{S C_g^2}{\lambda_4 n^2}\E\Norm{\bs_{t,k+1}-\s_{t,k}}^2 + 4\lambda_4 C_f^2 + \frac{\lambda_2}{2n}\prt{\E\Norm{\y-\tilde{\y}_{t,k}}^2 - \E\Norm{\y-\tilde{\y}_{t,k+1}}^2} \nonumber\\
            &\quad + \frac{\sigma_0^2}{2B\lambda_2} + \frac{\lambda_3\sigma_0^2}{2B} + \frac{\E\Norm{\by_{t,k+1}-\y_{t,k}}^2}{4\lambda_3 n}.
        \end{align}
        For $\mathcal{\D}_3$, by invoking Assumption \ref{asm:lip}, for some $\lambda_5, \lambda_6 > 0$, we have
        \begin{align}
            \D_3 &\leq C_f C_g \E\Norm{\u_{t,k+1}-\u_{t,k}} + \frac{S C_f C_g}{n^2}\E\brk{\sum_{i=1}^n\abs{\bs_{t,k+1}^{(i)}-\s_{t,k}^{(i)}}} \nonumber\\
            &\leq \frac{\lambda_5}{2\eta_{t}}\E\Norm{\u_{t,k+1}-\u_{t,k}}^2 + \frac{\eta_{t} C_f^2 C_g^2}{2\lambda_5} + \frac{S\lambda_6}{2n^2\beta_{t}}\E\Norm{\bs_{t,k+1}-\s_{t,k}}^2 + \frac{S\beta_{t} C_f^2 C_g^2}{2n\lambda_6}.
        \end{align}
        For $\mathcal{\D}_4$, same to the derivations for $\mathcal{\D}_3$, it holds that 
        \begin{align}
            \D_4 \leq \frac{\lambda_5}{2\eta_{t}}\E\Norm{\u_{t,k+1}-\u_{t,k}}^2 + \frac{\eta_{t} C_f^2 C_g^2}{2\lambda_5} + \frac{\lambda_6}{2n\beta_{t}}\E\Norm{\bs_{t,k+1}-\s_{t,k}}^2 + \frac{\beta_{t} C_f^2 C_g^2}{2\lambda_6}.
        \end{align}
         For $\D_5$, noticing
        \begin{align}
            \frac{1}{2}\prt{\Norm{\u_{t,k+1}-\u_{t,0}}^2-\Norm{\u_{t,k}-\u_{t,0}}^2} \leq -\inner{\u_{t,k+1}-\u_{t,0}}{\u_{t,k}-\u_{t,k+1}},
        \end{align}
        then we have
        \begin{align}
            \D_5 &\leq \E\brk{\frac{1}{\gamma}\inner{\u_{t,k}-\u_{t,0}}{\u_{t,k}-\u_{t,k+1}}} - \E\brk{\frac{1}{\gamma}\inner{\u_{t,k+1}-\u_{t,0}}{\u_{t,k}-\u_{t,k+1}}} \nonumber\\
            &= \frac{1}{\gamma}\E\Norm{\u_{t,k}-\u_{t,k+1}}^2.
        \end{align}
        In the same manner as for $\D_5$, we obtain $\D_6 \leq \frac{1}{n\gamma}\E\Norm{\s_{t,k}-\bs_{t,k+1}}^2$. Next, applying Lemma 4 in \citet{juditsky2011solving}(as well as Lemma 7 in \citet{zhang2020optimal}) on $\D_7$, we have
        \begin{align}
            \D_7 &= -\E\inner{\frac{1}{S}\sum_{i\in\S_{t,k}}\y_{t,k+1}^{(i)}\prt{\hat{G}_{t,k,1}^{(i)}(\tilde{\B}_{t,k}^{(i)})+\frac{1}{\tau_{t,k}}(\u_{t,k}-\u_{t,0})} - \frac{1}{n}\sum_{i=1}^n\by_{t,k+1}^{(i)}\prt{G_{t,k,1}^{(i)}+\frac{1}{\bar{\tau}_{t,k}}(\u_{t,k}-\u_{t,0})}}{\u_{t,k+1}} \nonumber\\
            &\leq \frac{\eta_{t}C_f^2\sigma_1^2}{B} + \frac{\eta_{t}\delta^2}{S}.
        \end{align}
        Finally, for $\D_8$, similar to $\D_7$, it follows that
        \begin{align}
            \D_8 &= \frac{1}{n}\sum_{i=1}^{n}\E\inner{\by_{t,k+1}^{(i)}\prt{\hat{G}_{t,k,2}^{(i)}(\tilde{\B}_{t,k}^{(i)}) - G_{t,k,2}^{(i)}}}{\s^{(i)}-\bs_{t,k+1}^{(i)}} \nonumber\\
            &= \frac{1}{n}\sum_{i=1}^{n}\E\inner{\by_{t,k+1}^{(i)}\prt{\hat{G}_{t,k,2}^{(i)}(\tilde{\B}_{t,k}^{(i)}) - G_{t,k,2}^{(i)}}}{-\bs_{t,k+1}^{(i)}} \nonumber\\
            &\leq \frac{\beta_{t}C_f^2\sigma_2^2}{B}.
        \end{align}
        {By setting $\lambda_1 = 1+\frac{1}{S}$, $\lambda_2 = \frac{n}{S\alpha_t}$, $\lambda_3 = \alpha_t$, $\lambda_4 = 8C_g^2 \max\{\eta_t, \beta_t\}$, and $\lambda_5 = \lambda_6 = \frac{1}{8}$, and substituting the bounds of $\mathcal{D}_1$ through $\mathcal{D}_8$ into inequality~(\ref{eq:eq2}), we obtain the desired result.}
    \end{proof}

\subsubsection{Proof of Lemma \ref{lem:proximal_error}}
\label{app:proximal_error}

    \begin{proof}
        For any $t\geq 0$, by invoking Lemma \ref{lem:descent_inner}, choosing $\u=\u^\dagger_t,\s=\s^\dagger_t$, and summing from $k=0$ to $K_t-1$ while taking expectation over $\F_{t,0}$, it holds that
        \begin{align}
            &\sum_{k=0}^{K_t-1} \E\brk{L_{\gamma}(\u_{t,k+1},\s_{t,k+1},\y;\u_{t,0},\s_{t,0})-L_{\gamma}(\u^{\dagger}_t,\s^{\dagger}_t,\by_{t,k+1};\u_{t,0},\s_{t,0})} \nonumber\\
            &\leq \frac{1}{2\eta_t}{\Norm{\u^\dagger_t-\u_{t,0}}^2} + \frac{1}{2S\beta_t}\Norm{\s^{\dagger}_{t}-\s_{t,0}}^2 + \frac{1}{\alpha_t S}\prt{\Norm{\y-\y_{t,0}}^2 + \Norm{\y-\hy_{t,0}}^2 + \Norm{\y-\ty_{t,0}}^2} \nonumber\\
            &\quad+ K_t\prt{64\Omega_t C_g^2 C_f^2 + \frac{S\alpha_t\sigma_0^2}{2Bn} + \frac{\alpha_t\sigma_0^2}{2B} + \frac{\eta_t C_f^2 \sigma_1^2}{B} + \frac{\beta_{t}C_f^2\sigma_2^2}{B} + \frac{\eta_t \delta^2}{S}}.          
        \end{align}
        Since $L_{\gamma}(\u,\s,\y;\u', \s')$ is convex on $\u$ and $\s$, and linear on $\y$, we have
        \begin{align}
            &\max_{\y\in\Y} L_{\gamma}(\bar{\u}_{t},\bar{\s}_{t},\y;\u_{t,0},\s_{t,0})-L_{\gamma}(\u^{\dagger}_t,\s^{\dagger}_t,\bar{\by}_{t};\u_{t,0},\s_{t,0}) \nonumber\\
            &\leq
            \max_{\y\in\Y}\frac{1}{K_t}\sum_{k=0}^{K_t-1} L_{\gamma}(\u_{t,k+1},\s_{t,k+1},\y;\u_{t,0},\s_{t,0})-L_{\gamma}(\u^{\dagger}_t,\s^{\dagger}_t,\by_{t,k+1};\u_{t,0},\s_{t,0}),
        \end{align}
        where $\bar{\u}_t=\frac{1}{K_t}\sum_{k=0}^{K_t-1}\u_{t,k+1}, \bar{\s}_t=\frac{1}{K_t}\sum_{k=0}^{K_t-1}\s_{t,k+1}, \bar{\by}_{t}=\frac{1}{T}\sum_{k=0}^{K_t-1}\by_{t,k+1}$. Next, for the left-hand side (LHS), we have
        \begin{align}
            &L_{\gamma}(\bar{\u}_{t},\bar{\s}_{t},\y;\u_{t,0},\s_{t,0})-L_{\gamma}(\u^{\dagger}_t,\s^{\dagger}_t,\bar{\by}_{t};\u_{t,0},\s_{t,0}) \nonumber\\
            &= \frac{1}{n}\sum_{i=1}^n \left(\y^{(i)}g_i(\bar{\u}_{t},\bar{\s}_{t}^{(i)}) - f_i^*(\y^{(i)})\right) + \frac{1}{2\gamma}\Norm{\bar{\u}_{t}-{\u}_{t,0}}^2 + \frac{1}{2n\gamma}\Norm{\bar{\s}_{t}-{\s}_{t,0}}^2 \nonumber\\
            &\quad- \frac{1}{n}\sum_{i=1}^n \left(\bar{\by}_{t}^{(i)}g_i(\u^{\dagger}_t, \s_t^{\dagger(i)}) - f_i^*(\bar{\by}_{t}^{(i)})\right) - \frac{1}{2\gamma}\Norm{\u^{\dagger}_t-{\u}_{t,0}}^2 - \frac{1}{2n\gamma}\Norm{\s^{\dagger}_t-{\s}_{t,0}}^2.
        \end{align}
        Choose $\y^{(i)} = \tilde{\y}^{(i)}_{t} \in \argmax_{\v}\{\v^{(i)}g_i(\bar{\u}_{t},\bar{\s}_{t}^{(i)}) - f_i^*(\v^{(i)})\}$, then we have $\y^{(i)}g_i(\bar{\u}_{t}, \bar{\s}_{t}^{(i)}) - f_i^*(\y^{(i)}) = f_i(g_i(\bar{\u}_{t}, \bar{\s}_{t}^{(i)}))$. By Fenchel-Young inequality, it holds that $\bar{\by}_{t}^{(i)}g_i(\u^{\dagger}_t, \s_t^{\dagger(i)}) - f_i^*(\bar{\s}_{t}^{(i)}) \leq f_i(g_i(\u^{\dagger}_t, \s_t^{\dagger(i)}))$. Thus, we have $\Phi_\gamma(\bar{\u}_{t},\bar{\s}_{t};{\u}_{t,0},{\s}_{t,0})-\Phi_\gamma(\u^{\dagger}_t,\s^{\dagger}_t;{\u}_{t,0},{\s}_{t,0}) \leq \max_{\y}\frac{1}{K_t}\sum_{k=0}^{K_t-1} L_{\gamma}(\u_{t,k+1},\s_{t,k+1},\y;\u_{t,0},\s_{t,0})-L_{\gamma}(\u^{\dagger}_t,\s^{\dagger}_t,\by_{t,k+1};\u_{t,0},\s_{t,0})$. Dividing both sides by $K_t$ completes the proof.
    \end{proof}

\subsubsection{Proof of Theorem \ref{thm:iteration_complexity_nvx}}
\label{app:iteration_complexity_nvx}

    \begin{proof}
        {We begin by invoking Lemma~\ref{lem:proximal_error}, which yields:}
        \begin{align}
            \E F(\u_{t+1},\s_{t+1}) &\leq  F(\u^\dagger_{t},\s^\dagger_{t}) - \frac{1}{2\gamma}\E\Norm{\u_{t+1}-\u_t}^2 - \frac{1}{2n\gamma}\E\Norm{\s_{t+1}-\s_t}^2 + \prt{\frac{1}{2\eta_{t} K_t}+\frac{1}{2\gamma}}\Norm{\u^{\dagger}_{t}-\u_t}^2 \nonumber\\
            &\quad + \prt{\frac{1}{2S\beta_t K_t}+\frac{1}{2n\gamma}}\Norm{\s^{\dagger}_{t}-\s_{t}}^2 + \frac{1}{\alpha_t SK_t}\prt{\Norm{\y-\y_{t,0}}^2 + \Norm{\y-\hy_{t,0}}^2 + \Norm{\y-\ty_{t,0}}^2}\nonumber\\
            &\quad + 64\Omega_t C_g^2 C_f^2 + \frac{S\alpha_t\sigma_0^2}{2Bn} + \frac{\alpha_t\sigma_0^2}{2B} + \frac{\eta_t C_f^2 \sigma_1^2}{B} + \frac{\beta_t C_f^2 \sigma_2^2}{B} + \frac{\eta_t \delta^2}{S}.
        \end{align}
        {Based on inequality (6) from \citet{rafique2022weakly}, the following estimate holds:}
        \begin{align*}
            &\Norm{\u^\dagger_t - \u_t}^2 + \frac{1}{n}\Norm{\s^\dagger_t - \s_t}^2 - \Norm{\u_{t+1} - \u_t}^2 - \frac{1}{n}\Norm{\s_{t+1} - \s_t}^2 \nonumber\\
            &\leq \frac{1}{3}\Norm{\u^\dagger_t - \u_t}^2 + \frac{1}{3n}\Norm{\s^\dagger_t - \s_t}^2 + 4 \Norm{\u^\dagger_t - \u_{t+1}}^2 + \frac{4}{n}\Norm{\s^\dagger_t - \s_{t+1}}^2, 
        \end{align*}
        then
        \begin{align}
            \E F(\u_{t+1}, \s_{t+1}) &\leq  F(\u^\dagger_{t}, \s^\dagger_{t}) + \prt{\frac{1}{2\eta_{t} K_t}+\frac{1}{6\gamma}}\Norm{\u^{\dagger}_{t}-\u_t}^2 + \frac{2}{\gamma}\E\Norm{\u^{\dagger}_{t}-\u_{t+1}}^2 \nonumber\\
            &\quad +\prt{\frac{1}{2S\beta_{t} K_t}+\frac{1}{6n\gamma}}\Norm{\s^{\dagger}_{t}-\s_t}^2 + \frac{2}{n\gamma}\E\Norm{\s^{\dagger}_{t}-\s_{t+1}}^2 \nonumber\\
            &\quad +\frac{3D_{\Y}^2}{\alpha_t SK_t} + 64\Omega_t C_g^2 C_f^2 + \frac{S\alpha_t\sigma_0^2}{2Bn} + \frac{\alpha_t\sigma_0^2}{2B} + \frac{\eta_t C_f^2 \sigma_1^2}{B} + \frac{\beta_t C_f^2 \sigma_2^2}{B} + \frac{\eta_t \delta^2}{S}.
        \end{align}
        {Next, we apply the strong convexity of the auxiliary potential function $\Phi_\gamma$, as established in Lemma~\ref{lem:weak_compositional}. Since $\Phi_\gamma(\u,\s;\u',\s')$ is $(\frac{1}{\gamma} - C_f\rho)$-strongly convex with respect to $\u$ and $(\frac{1}{n\gamma} - \frac{C_f\rho}{n})$-strongly convex with respect to $\s$, it follows that:}
        \begin{align}
            &(\frac{1}{2\gamma}-\frac{C_f\rho}{2})\E\Norm{\u^{\dagger}_{t}-\u_{t+1}}^2 + (\frac{1}{2n\gamma}-\frac{C_f\rho}{2n})\E\Norm{\s^{\dagger}_{t}-\s_{t+1}}^2 \nonumber\\
            &\leq \E\Phi_\gamma(\u_{t+1},\s_{t+1};\u_t,\s_t) - \Phi_\gamma(\u^\dagger_{t},\s^\dagger_{t};\u_t,\s_t) \nonumber\\
            &\leq \frac{1}{2\eta_{t} K_t}\E\Norm{\u^{\dagger}_{t}-\u_t}^2 + \frac{1}{2S\beta_{t} K_t}\E\Norm{\s^{\dagger}_{t}-\s_t}^2 + \frac{3D_{\Y}^2}{\alpha_t SK_t}\nonumber\\
            &\quad  + 64\Omega_t C_g^2 C_f^2 + \frac{S\alpha_t\sigma_0^2}{2Bn} + \frac{\alpha_t\sigma_0^2}{2B} + \frac{\eta_t C_f^2 \sigma_1^2}{B} + \frac{\beta_t C_f^2 \sigma_2^2}{B} + \frac{\eta_t \delta^2}{S}.
        \end{align}
        {Furthermore, using the assumption $\gamma \leq \frac{1}{2C_f\rho}$, the final descent bound simplifies to:}        \begin{align}
            \E F(\u_{t+1},\s_{t+1}) &\leq F(\u^\dagger_{t},\s^\dagger_{t}) + \prt{\frac{9}{2\eta_{t} K_t}+\frac{1}{6\gamma}}\Norm{\u^{\dagger}_{t}-\u_t}^2 + \prt{\frac{9}{2S\beta_{t} K_t}+\frac{1}{6n\gamma}}\Norm{\s^{\dagger}_{t}-\s_t}^2 \nonumber\\
            &\quad +\frac{27D_{\Y}^2}{\alpha_t SK_t} + 576\Omega_t C_g^2 C_f^2 + \frac{9S\alpha_t\sigma_0^2}{2Bn} + \frac{9\alpha_t\sigma_0^2}{2B} + \frac{9\eta_t C_f^2 \sigma_1^2}{B} + \frac{9\beta_t C_f^2 \sigma_2^2}{B} + \frac{9\eta_t \delta^2}{S}.
        \end{align}
        Since $F(\u^\dagger_{t},\s^\dagger_{t}) \leq F(\u_{t},\s_{t}) - \frac{1}{2\gamma}\Norm{\u^{\dagger}_{t}-\u_t}^2 - \frac{1}{2n\gamma}\Norm{\s^{\dagger}_{t}-\s_t}^2$, it follows that
        \begin{align}
            \E F(\u_{t+1},\s_{t+1}) &\leq F(\u_{t},\s_{t}) - \frac{1}{3\gamma}\Norm{\u^{\dagger}_{t}-\u_t}^2 + \frac{9}{2\eta_{t} K_t}\Norm{\u^{\dagger}_{t}-\u_t}^2 - \frac{1}{3n\gamma}\Norm{\s^{\dagger}_{t}-\s_t}^2 + \frac{9}{2S\beta_{t} K_t}\Norm{\s^{\dagger}_{t}-\s_t}^2   \nonumber\\
            &\quad +\frac{27D_{\Y}^2}{\alpha_t SK_t} + 576\Omega_t C_g^2 C_f^2 + \frac{9S\alpha_t\sigma_0^2}{2Bn} + \frac{9\alpha_t\sigma_0^2}{2B} + \frac{9\eta_t C_f^2 \sigma_1^2}{B} + \frac{9\beta_t C_f^2 \sigma_2^2}{B} + \frac{9\eta_t \delta^2}{S}.
        \end{align}
        {Summing both sides over $t = 0, 1, \ldots, T-1$ and taking expectation with respect to $\mathcal{F}_{0,0}$, we conclude:}
        \begin{align}
            &\frac{1}{T}\sum_{t=0}^{T-1}\prt{(\frac{1}{3\gamma}-\frac{9}{2\eta_t K_t})\E\Norm{\u^{\dagger}_{t}-\u_t}^2 + (\frac{1}{3n\gamma}-\frac{9}{2S\beta_t K_t})\E\Norm{\s^{\dagger}_{t}-\s_t}^2} \nonumber\\
            &\leq \frac{F(\u_{0},\s_{0}) - \E F(\u_{T},\s_{T})}{T} + \frac{27 D_{\Y}^2}{\alpha_t S K_t} + 576\Omega_t C_g^2 C_f^2\nonumber\\
            &\quad + \frac{9S\alpha_t\sigma_0^2}{2Bn} + \frac{9\alpha_t\sigma_0^2}{2B} + \frac{9\eta_t C_f^2 \sigma_1^2}{B} + \frac{9\beta_t C_f^2 \sigma_2^2}{B} + \frac{9\eta_t \delta^2}{S}.
        \end{align}
        {Finally, choosing the step sizes and inner-loop iteration numbers appropriately as:
        \begin{align}
            \eta_t&\asymp \min\{\frac{\epsilon^2}{C_g^2C_f^2},\frac{B\epsilon^2}{C_f^2\sigma_1^2},\frac{S\epsilon^2}{\delta^2},\gamma\epsilon^2\} \nonumber\\
            \beta_t&\asymp \min\{\frac{\epsilon^2}{C_g^2C_f^2},\frac{B\epsilon^2}{C_f^2\sigma_2^2}, \gamma\epsilon^2\} \nonumber\\ 
            \alpha_t &\asymp \frac{B\epsilon^2}{\sigma_0^2} \nonumber\\
            K_t &\asymp \max\{\frac{D_{\Y}^2\sigma_0^2}{BS\epsilon^4}, \frac{\gamma}{\eta_t}, \frac{n\gamma}{S\beta_t}\},
        \end{align}
        we arrive at the desired convergence rate:}
        \begin{align}
            \E\brk{\mathrm{dist}(\mathbf{0},\partial F(\u^{\dagger}_{\bar{t}},\s^{\dagger}_{\bar{t}}))^2} 
            &\leq \frac{1}{\gamma^2T}\sum_{t=0}^{T-1}\E\prt{\Norm{\u^{\dagger}_{t}-\u_t}^2+\frac{1}{n}\Norm{\s^{\dagger}_{t}-\s_t}^2} \nonumber\\
            &\leq \frac{6\brk{F(\u_{0},\s_{0}) - \E F(\u_{T},\s_{T})}}{\gamma T} + \order{\gamma^{-1}\epsilon^2},
        \end{align}
        where $\bar{t}$ is uniformly sampled from $\{0,1,\cdots,T-1\}$. Then we can make $\E\brk{\mathrm{dist}(\mathbf{0},\partial F(\u^{\dagger}_{\bar{t}},\s^{\dagger}_{\bar{t}}))} \leq \epsilon$ by choosing $T=\order{\frac{F(\u_{0},\s_{0})-\inf_{\u,\s} F(\u,\s)}{\epsilon^2}}$. The total iteration complexity would be 
        \begin{align}
            \sum_{t=0}^{T-1}K_t = \order{\frac{ C_g^2C_f^2}{\epsilon^4} + \frac{C_f^2\sigma_1^2}{B\epsilon^4} + \frac{\delta^2}{S\epsilon^4} + \frac{nC_g^2C_f^2}{S\epsilon^4} + \frac{nC_f^2\sigma_2^2}{BS\epsilon^4} + \frac{D_{\Y}^2\sigma_0^2}{BS\epsilon^6}}
        \end{align}
    \end{proof}

\end{document}